\documentclass{article} % For LaTeX2e
\usepackage{iclr2025_conference,times}

\usepackage{amsmath,amsfonts,bm}

\def\eqref#1{equation~\ref{#1}}
\def\1{\bm{1}}

\DeclareMathAlphabet{\mathsfit}{\encodingdefault}{\sfdefault}{m}{sl}
\SetMathAlphabet{\mathsfit}{bold}{\encodingdefault}{\sfdefault}{bx}{n}

\newcommand{\R}{\mathbb{R}}

\DeclareMathOperator*{\argmin}{arg\,min}

\usepackage[hidelinks,colorlinks=true,citecolor=blue!50!black,linkcolor=black,urlcolor=green!50!black]{hyperref}
\usepackage{url}

\usepackage{booktabs}       % professional-quality tables
\usepackage{amsmath}
\usepackage{amssymb}
\usepackage{mathtools}
\usepackage{amsthm}

\newtheorem{theorem}{Theorem}[section]
\newtheorem{proposition}[theorem]{Proposition}
\newtheorem{lemma}[theorem]{Lemma}
\newtheorem{corollary}[theorem]{Corollary}
\newtheorem{definition}[theorem]{Definition}

\newtheorem{assumption}[theorem]{Assumption}

\title{Exact Computation of Any-Order Shapley Interactions for Graph Neural Networks}

\author{Maximilian Muschalik$^{1,*,\text{\Letter}}$, Fabian Fumagalli$^{2,*}$, Paolo Frazzetto$^{3}$, Janine Strotherm$^{2}$,
\\
\textbf{Luca Hermes$^{2}$,  Alessandro Sperduti$^{3,4,5}$, Eyke Hüllermeier$^{1}$, Barbara Hammer$^{2}$}
\medskip \\
$^1$ LMU Munich, MCML \quad $^2$ Bielefeld University, CITEC \quad $^3$ University of Padova
\\
$^4$ University of Trento \quad $^5$ Augmented Intelligence Center, FBK\\ 
$^*$ Equal Contribution \quad  $^\text{\Letter}$ Corresponding Author: \texttt{maximilian.muschalik@lmu.de} \\
}

\usepackage{color,colortbl} %new
\definecolor{tablegray}{gray}{0.94} %new
\definecolor{graphshapiq}{rgb}{0.4902,0.8078,0.5098} %new
\usepackage{multirow} %new
\usepackage{xcolor} %new
\usepackage{adjustbox}
\usepackage{algorithmic}
\usepackage{algorithm}
\usepackage{wrapfig} %new
\usepackage{enumitem} %new
\usepackage{titletoc} %new
\usepackage{wrapfig} %new
\usepackage[misc]{ifsym} % new

\usepackage[capitalize,noabbrev]{cleveref}
\usepackage[acronym]{glossaries}
\allowdisplaybreaks
\newglossaryentry{formula}{name=formula,
                           description={A mathematical expression}}

\newacronym{EU}{EU}{European Union}
\newacronym{iid}{iid}{independent and identically distributed}
\newacronym{iff}{iff}{if and only if}
\newacronym{wlog}{w.l.o.g.}{without loss of generality}
\newacronym{PDE}{PDE}{partial differential equation}
% ...
\newacronym{CDF}{CDF}{cumulative distribution function}
% ...
\newacronym{OP}{OP}{optimization problem}
\newacronym{AI}{AI}{artificial intelligence}
\newacronym{XAI}{XAI}{explainable artificial intelligence}
\newacronym{ML}{ML}{machine learning}
\newacronym{DL}{DL}{deep learning}
\newacronym{TP}{TP}{true positives}
\newacronym{FP}{FP}{false positives}
\newacronym{FN}{FN}{false negatives}
\newacronym{TN}{TN}{true negatives}
\newacronym{ACC}{ACC}{accuracy}                     % ACC = (TP+TN) / (TP+FN+FP+TN)
\newacronym{TPR}{TPR}{true positive rate}           % TPR = TP / (TP+FN) = 1 - FNR
\newacronym{FNR}{FNR}{false negative rate}          % FNR = FN / (TP+FN) = 1 - TPR
\newacronym{FPR}{FPR}{false positive rate}          % FPR = FP / (FP+TN) = 1 - TNR
\newacronym{TNR}{TNR}{true negative rate}           % TNR = TN / (FP+TN) = 1 - FPR
\newacronym{PPV}{PPV}{positive predictive value}    % PPV = TP / (TP+FP) = 1 - FDR
\newacronym{FDR}{FDR}{false discorvery rate}        % FDR = FP / (TP+FP) = 1 - PPV
\newacronym{FOR}{FOR}{false omission rate}          % FOR = FN / (FN*TN) = 1 - NPV
\newacronym{NPV}{NPV}{negative predictive value}    % NPV = TN / (FN+TN) = 1 - FOR
\newacronym{ROC}{ROC-curve}{receiver operating characteristic curve}
\newacronym{AUC}{AUC}{area under the (ROC) curve}
\newacronym{IG}{IG}{information gain}
\newacronym{MSE}{MSE}{mean squared error}
\newacronym{DI}{DI}{disparate impact}
\newacronym{DP}{DP}{demographic parity}
\newacronym{EOs}{EOs}{equalized odds}
\newacronym{EO}{EO}{equal opportunity}
\newacronym{SVM}{SVM}{support vector machine}
\newacronym{MLP}{MLP}{multi layer perceptron}
\newacronym{NN}{NN}{neural network}
\newacronym{GNN}{GNN}{Graph Neural Network}
\newacronym{GCN}{GCN}{Graph Convolutional Network}
\newacronym{GIN}{GIN}{Graph Isomorphism Network}
\newacronym{GAT}{GAT}{Graph Attention Network}
\newacronym{RF}{RF}{receptive field}
% ...
\newacronym{WDN}{WDN}{water distribution network}
\newacronym{WDS}{WDS}{water distribution system}
\newacronym{ERC}{ERC}{European Research Council}

\newacronym{SV}{SV}{Shapley Value}
\newacronym{SII}{SII}{Shapley Interaction Index}
\newacronym{k-SII}{$k$-SII}{$k$-Shapley Value}
\newacronym{SI}{SI}{Shapley Interaction}
\newacronym{MI}{MI}{Möbius Interaction}
\newacronym{STII}{STII}{Shapley Taylor Interaction Index}
\newacronym{FSII}{FSII}{Faithful Shapley Interaction Index}
\newacronym{GT}{GT}{Ground Truth}
\newacronym{MTG}{\emph{MTG}}{Mutagenicity}
\newacronym{BNZ}{\emph{BNZ}}{Benzene}
\newacronym{FLC}{\emph{FLC}}{FluorideCarbonyl}
\newacronym{ALC}{\emph{ALC}}{AlkaneCarbonyl}
\newacronym{CX2}{\emph{CX2}}{COX2}
\newacronym{PRT}{\emph{PRT}}{PROTEINS}
\newacronym{ENZ}{\emph{ENZ}}{ENZYMES}
\newacronym{BZR}{\emph{BZR}}{BZR}
\newacronym{WAQ}{\emph{WAQ}}{WaterQuality}
\newacronym{BSHAP}{BShap}{Baseline Shapley}

% ...
\newcommand{\Pow}{\mathcal{P}}  % powerset
% ...
\newcommand{\Xm}{\mathbf{X}}    % input matrix / feature matrix
\newcommand{\Hm}{\mathbf{H}}    % hidden feature matrix
\newcommand{\xm}{\mathbf{x}}    % input vector / feature vector
% ...
% ...
% ...
\newcommand{\Nbh}{\mathcal{N}}    % neighborhood
% ...
% ...
% ...
% ...
% ...
% ...
% ...
% ...

\iclrfinalcopy % Uncomment for camera-ready version, but NOT for submission.
\begin{document}

\maketitle

\begin{abstract}
Albeit the ubiquitous use of \glspl*{GNN} in \gls*{ML} prediction tasks involving graph-structured data, their interpretability remains challenging.
In \gls*{XAI}, the \gls*{SV} is the predominant method to quantify contributions of individual features to a \gls*{ML} model's output.
Addressing the limitations of \glspl*{SV} in complex prediction models, \glspl*{SI} extend the \gls*{SV} to groups of features.
In this work, we explain single graph predictions of \glspl*{GNN} with \glspl*{SI} that quantify node contributions and interactions among multiple nodes.
By exploiting the \gls*{GNN} architecture, we show that the structure of interactions in node embeddings are preserved for graph prediction.
As a result, the exponential complexity of \glspl*{SI} depends only on the receptive fields, i.e. the message-passing ranges determined by the connectivity of the graph and the number of convolutional layers. 
Based on our theoretical results, we introduce GraphSHAP-IQ, an efficient approach to compute any-order \glspl*{SI} exactly.
GraphSHAP-IQ is applicable to popular message-passing techniques in conjunction with a linear global pooling and output layer.
We showcase that GraphSHAP-IQ substantially reduces the exponential complexity of computing exact \glspl*{SI} on multiple benchmark datasets.
Beyond exact computation, we evaluate GraphSHAP-IQ's approximation of \glspl*{SI} on popular \gls*{GNN} architectures and compare with existing baselines.
Lastly, we visualize \glspl*{SI} of real-world water distribution networks and molecule structures using a \gls*{SI}-Graph.
\end{abstract}

\glsresetall

\section{Introduction}\label{sec_introduction}

Graph-structured data appears in many domains and real-world applications \citep{newman2018networks}, such as molecular chemistry \citep{Gilmer2017NeuralMP}, \glspl*{WDN} \citep{Ashraf2023PressurePrediction}, sociology \citep{borgatti2009network}, physics \citep{SanchezGonzalez2020LearningTS}, or human resources \citep{Frazzetto2023EnhancingHR}.
%epidemiology \citep{epidemicsNetworks}, or finance \citep{David2010NetworksCA}.
To leverage such structure in \gls*{ML} models, \glspl*{GNN} emerged as the leading family of architectures that specifically exploit the graph topology \citep{DBLP:journals/tnn/ScarselliGTHM09}.
A major drawback of \glspl*{GNN} is the opacity of their predictive mechanism, which they share with most deep-learning based architectures \citep{GraphFramEx}. 
Reliable explanations for their predictions are crucial when model decisions have significant consequences \citep{zhang2024trustworthy} or lead to new discoveries \citep{mccloskey2019using}.
\\
In \gls*{XAI}, the \gls*{SV} \citep{Shapley.1953} is a prominent concept to assign contributions to entities of black box \gls*{ML} models \citep{DBLP:conf/nips/LundbergL17,DBLP:journals/jmlr/CovertLL21,Chen2023Overview_ExplainabilityWithShapley}.
Entities typically represent features, data points \citep{Ghorbani.2019} or graph structures \citep{DBLP:conf/icml/YuanYWLJ21,DBLP:conf/nips/YeHWL23}.
Although \glspl*{SV} yield an axiomatic attribution scheme, they do not give any insights into joint contributions of entities, known as \emph{interactions}.
Yet, interactions are crucial to understanding decisions of complex black box \gls*{ML} models \citep{Wright.2016,Sundararajan.2020,Kumar.2021}.
\glspl*{SI} \citep{Grabisch.1999,Bord.2023} extend the \gls*{SV} to include joint contributions of multiple entities.
\glspl*{SI} satisfy similar axioms while providing interactions up to a maximum number of entities, referred to as the \emph{explanation order}. 
In this context, \glspl*{SV} are the least complex \glspl*{SI}, whereas \glspl*{MI} (or Möbius transform) \citep{harsanyi1963simplified,rota1964foundations} are the most complex \glspl*{SI} by assigning contributions to every group of entities.
Thus, \glspl*{SI} convey an adjustable explanation with an \emph{accuracy-complexity trade-off} for interpretability \citep{Bord.2023}.
\glspl*{SV}, \glspl*{SI} and \glspl*{MI} are limited by exponential complexity, e.g. with $20$ features already $2^{20} \approx 10^{6}$ model calls per explained instance are required.
Consequently, practitioners rely on model-agnostic approximation methods \citep{DBLP:conf/nips/LundbergL17,Fumagalli.2023} or model-specific methods \citep{Lundberg.2020,DBLP:conf/aaai/MuschalikFHH24} that exploit knowledge about the model's structure to reduce complexity.
As a remedy for \glspl*{GNN}, the \gls*{SV} was applied as a heuristic on subgraphs \citep{DBLP:conf/nips/YingBYZL19,DBLP:conf/nips/YeHWL23}, or approximated \citep{DBLP:conf/pkdd/DuvalM21,DBLP:conf/icml/BuiNNY24}.
\\
In this work, we address limitations of the \gls*{SV} for \gls*{GNN} explainability by computing the \glspl*{SI} visualized as the SI-Graph in \cref{fig_intro_illustration}.
Our method yields exact \glspl*{SI} by including \gls*{GNN}-specific knowledge and exploiting properties of the \glspl*{SI}.
In contrast to existing methods \citep{DBLP:conf/icml/YuanYWLJ21,DBLP:conf/nips/YeHWL23}, we evaluate the \gls*{GNN} on node level without the need to cluster nodes into subgraphs.
Instead of model-agnostic approximation \citep{DBLP:conf/pkdd/DuvalM21,DBLP:conf/icml/BuiNNY24}, we provide structure-aware computation for graph prediction tasks, and prove that \glspl*{MI} of node embeddings indeed transfer to graph prediction for linear readouts.
In summary, our approach is a model-specific computation of \glspl*{SI} for \glspl*{GNN}, akin to TreeSHAP \citep{Lundberg.2020} for tree-based models. 

\begin{figure}[t]
    \centering
    \begin{minipage}[c]{0.09\textwidth}
        \includegraphics[width=\textwidth]{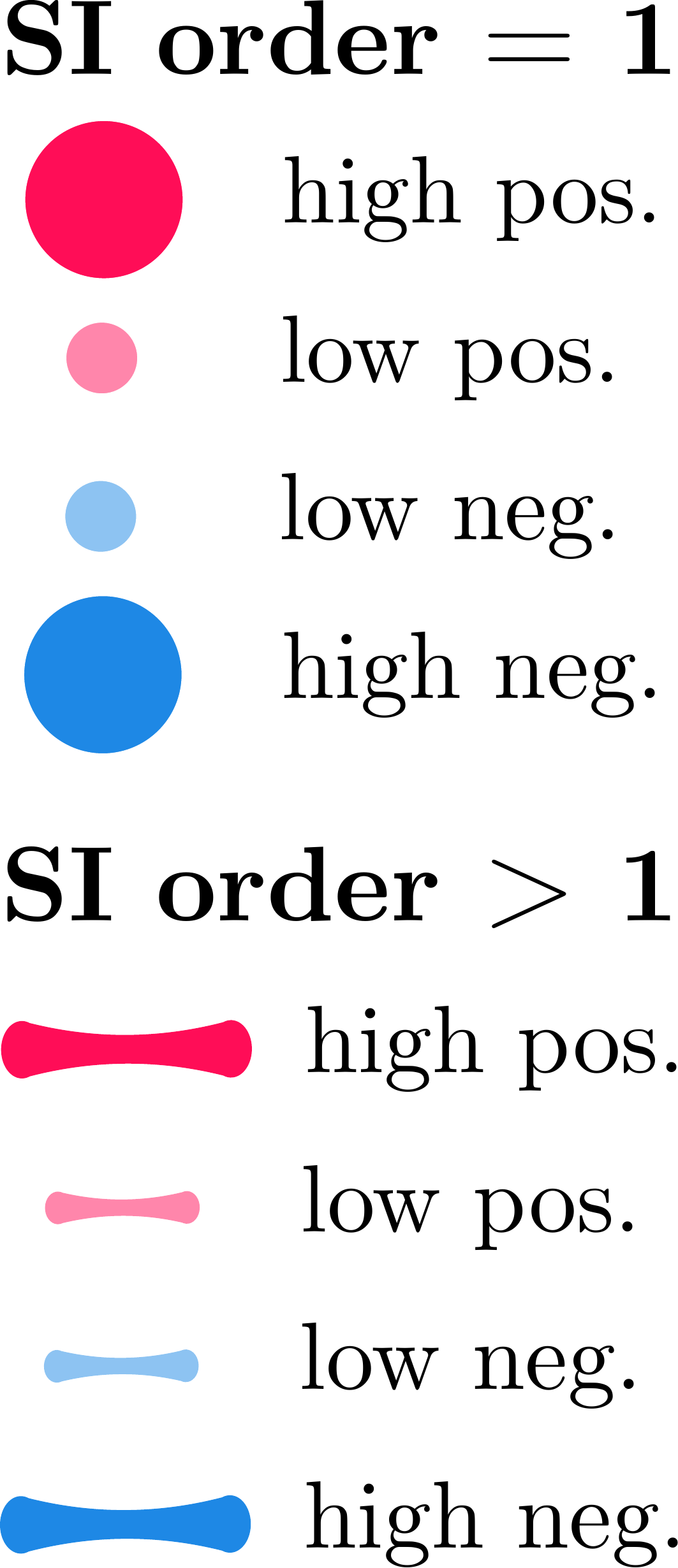}
    \end{minipage}
    \hfill
    \begin{minipage}[c]{0.29\textwidth}
    \centering
        \includegraphics[width=\textwidth]{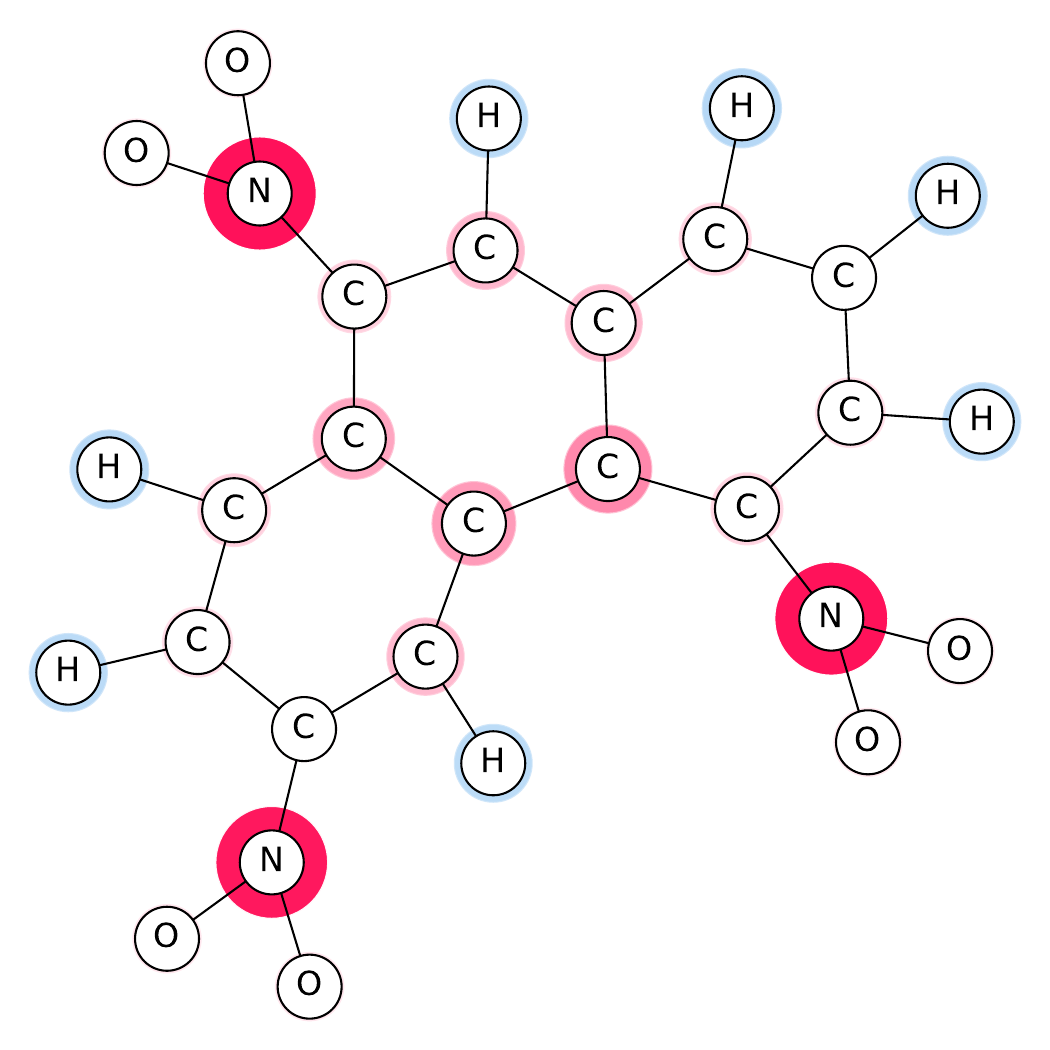}
        \textbf{Exact SVs}
    \end{minipage}
    \hfill
    \begin{minipage}[c]{0.29\textwidth}
    \centering
        \includegraphics[width=\textwidth]{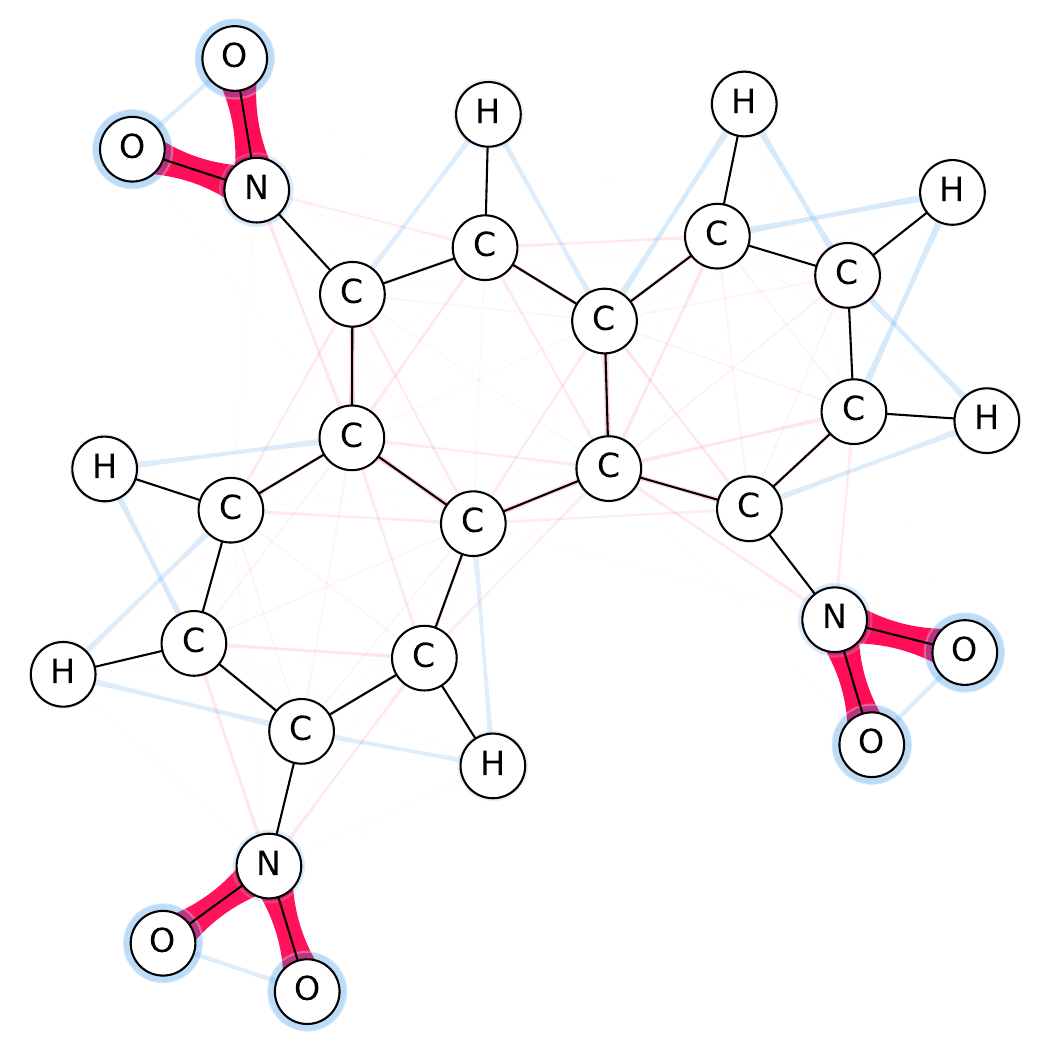}
        \textbf{Exact 2-SIIs}
    \end{minipage}
    \hfill
    \begin{minipage}[c]{0.29\textwidth}
    \centering
        \includegraphics[width=\textwidth]{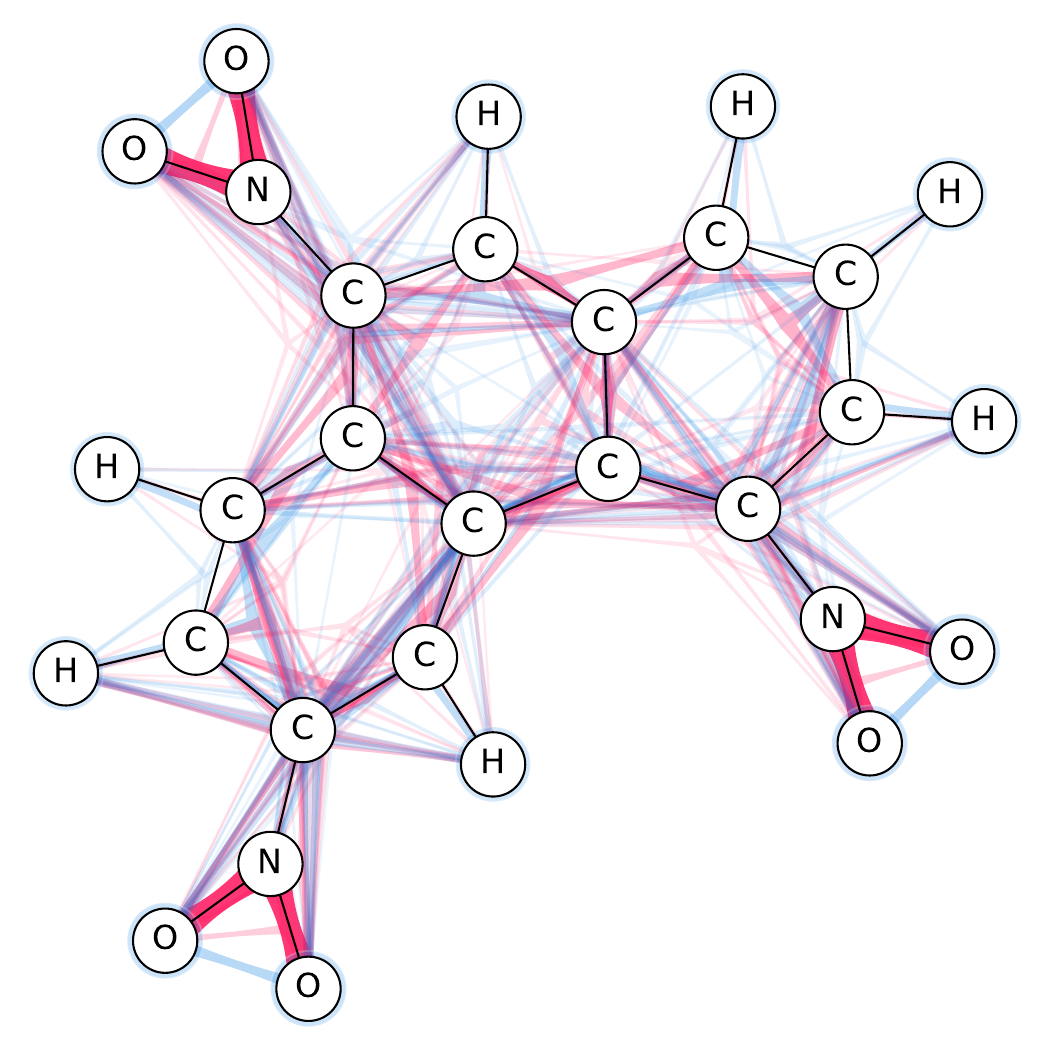}
        \textbf{Exact MIs}
    \end{minipage}
    \caption{\gls*{SI}-Graphs overlayed on a molecule graph showing exact \glspl*{SI} for a molecule with $30$ atoms from \emph{MTG}. A \gls*{GNN} correctly identifies it as mutagenic. The \glspl*{SI}, in line with ground-truth knowledge, highlight the $\text{NO}_{2}$ groups. Computing exact \glspl*{SI} requires $2^{30} \approx 10^9$ model calls. \mbox{GraphSHAP-IQ} needs $7\,693$.}
    \label{fig_intro_illustration}
\end{figure}

\paragraph{Contribution.} Our main contributions include:
\begin{enumerate}
    \item[\textbf{(1)}] We introduce \glspl*{SI} among nodes and the \gls*{SI}-Graph for graph predictions of \glspl*{GNN} that address limitations of the \gls*{SV} and exploit graph and \gls*{GNN} structure with our theoretical results. 
    \item[\textbf{(2)}] We present GraphSHAP Interaction Quantification (GraphSHAP-IQ), an efficient method to compute exact any-order \glspl*{SI} in \glspl*{GNN}.
    %or approximate them in restricted settings. 
    For restricted settings requiring approximation, we extend GraphSHAP-IQ and propose several interaction-informed baseline methods.
    \item[\textbf{(3)}] We show substantially reduced complexity when applying GraphSHAP-IQ on real-world benchmark datasets, and analyze \gls*{SI}-Graphs of a \gls*{WDN} and molecule structures. 
    \item[\textbf{(4)}] We find that interactions in deep readout \glspl*{GNN} are not restricted to the receptive fields.
\end{enumerate}

\paragraph{Related Work.}
\glspl*{SI}, enriching the \gls*{SV} \citep{Shapley.1953} with higher-order interactions, were introduced in game theory \citep{Grabisch.1999}, and modified for local interpretability in \gls*{ML} \citep{Lundberg.2020,Bord.2023}.
The exponential complexity of the \gls*{SV} and \glspl*{SI} requires approximation by model-agnostic Monte Carlo sampling \citep{Chen2023Overview_ExplainabilityWithShapley,Fumagalli.2023} or by exploiting the data structure \citep{DBLP:conf/iclr/ChenSWJ19}. Exact computation is only feasible with knowledge about the structure of the model and model-specific methods, such as TreeSHAP \citep{Lundberg.2020,DBLP:conf/aaai/MuschalikFHH24}, which is applicable to tree ensembles.
Here, we present a model-specific method applicable to \glspl*{GNN} and graph prediction tasks, which computes exact \glspl*{SI} by exploiting the graph and \gls*{GNN} structure.
\\
For local explanations of \glspl*{GNN} and graph prediction tasks, a variety of perturbation-based methods \citep{DBLP:conf/nips/YingBYZL19,PGExplainer,DBLP:conf/icml/YuanYWLJ21} have been proposed, which output isolated subgraphs, whereas \citet{DBLP:conf/cvpr/PopeKRMH19} introduce node attributions.
In this work, we use perturbations via node maskings and output attributions for all nodes, and all subset of nodes up to the given explanation order, which additively decomposes the graph prediction of the \gls*{GNN}.
\\
In context of \gls*{GNN} interpretability, the \gls*{SV} was applied in GraphSVX \citep{DBLP:conf/pkdd/DuvalM21} on single nodes with a structure-aware approximation for node predictions. For graph prediction tasks, however, GraphSVX proposes a model-agnostic approximation.
SubgraphX \citep{DBLP:conf/icml/YuanYWLJ21} and SAME \citep{DBLP:conf/nips/YeHWL23} use the \gls*{SV} to assess the quality of isolated subgraphs.
Recently, model-agnostic approximation of pairwise \glspl*{SI} have shown to improve isolated subgraph detections \citep{DBLP:conf/icml/BuiNNY24}.
In contrast to existing work, we propose a structure-aware method that efficiently computes exact \glspl*{SV} for single nodes and exact any-order \glspl*{SI} for subset of nodes for graph prediction tasks of \glspl*{GNN}. Our explanation is based on all possible interactions (\glspl*{MI}) and additively decomposes the prediction.
For a more detailed discussion of related work, we refer to \cref{appx_sec_rw_and_l_shapley}.

\section{Background}
In \cref{section_background_sv}, we introduce \glspl*{SI} that provide an adjustable \emph{accuracy-complexity} trade-off for explanations \citep{Bord.2023}. 
In this context, \glspl*{SV} are the simplest and \glspl*{MI} the most complex \glspl*{SI}.
In \cref{subsection_GraphNeuralNetworks}, we introduce \glspl*{GNN}, whose structure we exploit in \cref{section_graphshapiq} to efficiently compute any-order \glspl*{SI}. 
A summary of notations can be found in Appendix~\ref{appx_sec_notation}.

\subsection{Explanation Complexity: From Shapley Values to Möbius Interactions}\label{section_background_sv}

Concepts from cooperative game theory, such as the \gls*{SV} \citep{Shapley.1953}, are prominent in \gls*{XAI} to interpret predictions of a black box \gls*{ML} model via feature attributions \citep{Strumbelj.2014,DBLP:conf/nips/LundbergL17}.
Formally, a cooperative game $\nu:\Pow(N) \rightarrow \mathbb{R}$ is defined, where individual features $N=\{1,\dots,n\}$ act as players and achieve a payout for every group of players in the power set $\Pow(N)$.
To obtain feature attributions for the prediction of a single instance, $\nu$ typically refers to the model's prediction given only a subset of feature values.
Since classical \gls*{ML} models cannot handle missing feature values, different methods have been proposed, such as model retraining \citep{DBLP:journals/dke/StrumbeljKR09}, conditional expectations \citep{DBLP:conf/nips/LundbergL17,Aas.2021,Frye.2021}, marginal expectations \citep{Janzing.2020} and baseline imputations \citep{DBLP:conf/nips/LundbergL17,Sundararajan.2020}.
In high-dimensional feature spaces, retraining models or approximating feature distributions is infeasible, imputing absent features with a baseline, known as \gls*{BSHAP} \citep{DBLP:conf/icml/SundararajanN20}, is the prevalent method \citep{DBLP:conf/nips/LundbergL17,Sundararajan_Taly_Yan_2017,Jethani.2022}.
We now first introduce the \glspl*{MI} as a backbone of additive contribution measures.
Later in \cref{section_graphshapiq}, we exploit sparsity of \glspl*{MI} for \glspl*{GNN} to compute the \gls*{SV} and any-order \glspl*{SI}.

%\textbf{Möbius Interactions (MIs)} $m: \Pow(N) \rightarrow \mathbb{R}$ , alternatively Möbius transform \citep{rota1964foundations}, Harsanyi dividend \citep{harsanyi1963simplified}, or internal interaction index \citep{Fujimoto.2006}, are a fundamental concept of cooperative game theory.

\textbf{Möbius Interactions (MIs)} $m: \Pow(N) \rightarrow \mathbb{R}$ are a fundamental concept of cooperative game theory and provide the basis for different summary measures. The \gls*{MI} is
\begin{align}\label{align_möbius_recovery}
    m(S) := \sum_{T \subseteq S} (-1)^{\vert S \vert -\vert T \vert} \nu(T) \text{ and they recover } \ \nu(T) = \sum_{S \subseteq T} m(S) \text{ for all } S,T\subseteq N.
\end{align}
From the \glspl*{MI}, every game value can be additively recovered, and \glspl*{MI} are the unique measure with this property \citep{harsanyi1963simplified, rota1964foundations}.
The \gls*{MI} of a subset $S \subseteq N$ can thus be interpreted as the \emph{pure additive contribution} that is exclusively achieved by a coalition of all players in $S$, and cannot be attributed to any subgroup of $S$.
The \glspl*{MI} are further a basis of the vector space of games \citep{Grabisch.2016}, and therefore every measure of contribution, such as the \gls*{SV} or the \glspl*{SI}, can be directly recovered from the \glspl*{MI}, cf. \cref{appx_sec_conversion}.

\textbf{Shapley Values (SVs)} for players $i \in N$ of a cooperative game $\nu$ are the weighted average 
\begin{align*}
    \phi^{\text{SV}}(i) := \sum_{T \subseteq N \setminus i} \frac{1}{n \cdot \binom{n-1}{\vert T \vert }} \Delta_i(T)  \text{ with } \ \Delta_i(T) := \nu(T \cup i) - \nu(T)
\end{align*}
over marginal contributions $\Delta_i(T)$.
It was shown \citep{Shapley.1953} that the \gls*{SV} is the unique attribution method that satisfies desirable axioms: \textit{linearity} (the \gls*{SV} of linear combinations of games, e.g., model ensembles, coincides with the linear combinations of the individual \glspl*{SV}), \textit{dummy} (features that do not change the model's prediction receive zero \gls*{SV}), \textit{symmetry} (if a model does not change its prediction when switching two features, then both receive the same \gls*{SV}), and lastly \textit{efficiency} (the sum of all \glspl*{SV} equals the difference between the model's prediction $\nu(N)$ and the featureless prediction $\nu(\emptyset)$).
We may normalize $\nu$, such that $\nu(\emptyset) = 0$, which does not affect the \glspl*{SV}.
The \gls*{SV} assigns attributions to individual features, which distribute the \glspl*{MI} that contain feature $i$, cf. \cref{appx_sec_conversion}.
However, the \gls*{SV} does not provide any insights about \emph{feature interactions}, i.e. the joint contribution of multiple features to the prediction.
Yet, in practice, understanding complex models requires investigating interactions \citep{Slack.2020,Sundararajan.2020,Kumar.2021}.
While the \glspl*{SV} are limited in their expressivity, the \glspl*{MI} are difficult to interpret due to the exponential number of components.
The \glspl*{SI} provide a framework to bridge both concepts.

\textbf{Shapley Interactions (SIs)} explore model predictions beyond individual feature attributions, and provide additive contribution for all subsets up to \emph{explanation order} $k=1,\dots,n$.
More formally, the \glspl*{SI} $\Phi_k: \Pow_k(N) \rightarrow \mathbb{R}$ assign interactions to subsets of $N$ up to size $k$, summarized in $\Pow_k(N)$.
The \glspl*{SI} decompose the model's prediction with $\nu(N) = \sum_{S \subseteq N, \vert S \vert \leq k} \Phi_k(S)$.
The least complex \glspl*{SI} are the \glspl*{SV}, which are obtained with $k=1$.
For $k=n$, the \glspl*{SI} are the \glspl*{MI} with $2^n$ components, which provide the most faithful explanation of the game but entail the highest complexity. 
\glspl*{SI} are constructed based on extensions of the marginal contributions $\Delta_i(T)$, known as discrete derivatives \citep{Grabisch.1999}.
For two players $i,j \in N$, the discrete derivative $\Delta_{ij}(T)$ for a subset $T \subseteq N \setminus ij$ is defined as
$\Delta_{ij}(T) := \nu(T \cup ij) -\nu(T) - \Delta_i(T) - \Delta_j(T)$, i.e., the joint contribution of adding both players together minus their individual contributions in the presence of $T$.
This recursion is extended to any subset $S \subseteq N$ and $T \subseteq N \setminus S$.
A positive value of the discrete derivative $\Delta_S(T)$ indicates synergistic effects, a negative value indicates redundancy, and a value close to zero indicates no joint information of all players in~$S$ given $T$.
The \gls*{SII} \citep{Grabisch.1999} provides an axiomatic extension of the \gls*{SV} and summarizes the discrete derivatives in the presence of all possible subsets $T$ as
    \begin{align*}
        &\phi^{\text{SII}}(S) = \sum_{T \subseteq N \setminus S } \frac{1}{(n - \vert S \vert +1) \cdot \binom{n - \vert S \vert}{\vert T \vert }} \Delta_S(T) &&\text{ with }   &&\Delta_S(T) := \sum_{L \subseteq S}(-1)^{\vert S \vert - \vert L \vert} \nu(T \cup L).
    \end{align*}
Given an explanation order $k$, the \glspl*{k-SII} \citep{Lundberg.2020, Bord.2023} construct \glspl*{SI} recursively based on the \gls*{SII}, such that the interactions of \gls*{SII} and \gls*{k-SII} for the highest order coincide.
Alternatively, the \gls*{STII} \citep{DBLP:conf/icml/SundararajanN20} and the \gls*{FSII} \citep{Tsai.2022} have been proposed, cf. \cref{appx_sec_otherinteractions}.
In summary, \glspl*{SI} provide a flexible framework of increasingly complex and faithful contributions ranging from the \gls*{SV} ($k=1$) to the \glspl*{MI} ($k=n$).
Given the \glspl*{MI}, it is possible to reconstruct \glspl*{SI} of arbitrary order, cf. \cref{appx_sec_conversion}.
In \cref{section_graphshapiq}, we will exploit the sparse structure of \glspl*{MI} of \glspl*{GNN} to efficiently compute any-order \glspl*{SI}.

\subsection{Graph Neural Networks}
\label{subsection_GraphNeuralNetworks}

\glspl*{GNN} are neural networks specifically designed to process graph input \citep{DBLP:journals/tnn/ScarselliGTHM09}. 
A graph $g = (V,E,\Xm)$ consists of sets of nodes $V = \{v_1,...,v_n\}$, edges $E \subset V \times V$ and $d_0$-dimensional node features $\Xm = [\xm_1, ..., \xm_n]^t \in \R^{n \times d_0}$, where $\xm_i$ are the node features of node $v_i \in V$. 
A \textit{message passing} \gls*{GNN} leverages the structural information of the graph $g$ to iteratively aggregate node feature information of a given node $v \in V$ within its \textit{neighborhood} $\Nbh(v) := \{ u \in V \mid e_{uv} \in E \}$. 
More precisely, in each iteration $k \in \{1,...,\ell\}$, the $d_k$-dimensional $k$-th hidden node features $\Hm^{(k)} = [\mathbf{h}_1^{(k)}, \ldots , \mathbf{h}_n^{(k)}]^t \in \R^{n \times d_k}$ are computed node-wise by
\begin{align}
\label{align_messagePassing}
    \mathbf{h}_{i}^{(0)} := \mathbf{x}_i, \quad 
    \mathbf{h}_{i}^{(k)} 
    := 
    \rho^{(k)} (
    \mathbf{h}_{i}^{(k-1)}, 
    \psi (
    \{ \{ \varphi^{(k)}(\mathbf{h}_j^{(k-1)}) \mid u_j \in\Nbh(v_i) \} \}
    )
    ),
\end{align}
where $\{\{\cdot\}\}$ indicates a multiset and $\rho^{(k)}$ and $\varphi^{(k)}$ are arbitrary (aggregation) functions acting on the corresponding spaces. Moreover, $\psi$ is implemented as a permutation-invariant function, ensuring independence of both the order and number of neighboring nodes, and allows for an embedding of multisets as vectors. The node embedding function is thus $f_i(\Xm) := \mathbf{h}_i^{(\ell)}$. 
For \textit{graph prediction} tasks, the representations of the last nodes $\Hm^{(\ell)}$ must be aggregated in a fixed-size graph embedding for the downstream task. More formally, this is achieved with an additional permutation-invariant pooling function $\Psi$ and a parametrized output layer $\sigma: \R^{d_{\ell}} \rightarrow \R^{d_{\text{out}}}$, where $d_{\text{out}}$ corresponds to the output dimension.
The output of a \gls*{GNN} for graph-level inference is defined as  
\begin{align}
\label{align_gnn}
    f_g(\Xm) := \sigma \big( \Psi( \{ \{ f_i(\Xm) \mid v_i \in V \} \}) \big).
\end{align}
For \emph{graph classification}, class probabilities are obtained from $f_g(\Xm)$ through \emph{softmax} activation.

\section{Any-Order Shapley Interactions for Graph Neural Networks}\label{section_graphshapiq}
In the following, we are interested in explaining the prediction of a \gls*{GNN} $f_g$ for a graph $g$ with respect to nodes. 
We aim to decompose a model's prediction into \glspl*{SI} $\Phi_k$ visualized by a \gls*{SI}-Graph.

\begin{definition}[\gls*{SI}-Graph]\label{definition_si_graph}
The \gls*{SI}-Graph is an undirected hypergraph $g^{\text{SI}}_k := (N,\Pow_k(N),\Phi_k)$ with node attributes $\Phi_k(i)$ for $i \in N$ and hyperedge attributes $\Phi_k(S)$ for $2 \leq \vert S \vert \leq k$.
\end{definition}

The simplest \gls*{SI}-Graph displays the \glspl*{SV} ($k=1$) as node attributes, whereas the most complex \gls*{SI}-Graph displays the \glspl*{MI} ($k=n$) as node and hyperedge attributes, illustrated in \cref{fig_intro_illustration}.
The complexity of the \gls*{SI}-Graph is adjustable by the explanation order $k$, which determines the maximum hyperedge order.
The sum of all contributions in the \gls*{SI}-Graph yields the model's prediction (for regression) or the model's logits for the predicted class (for classification).
This choice is natural for an additive contribution measure due to additivity in the logit-space.
To compute \glspl*{SI}, we introduce the \gls*{GNN}-induced graph game $\nu_g$ with a node masking strategy in \Cref{subsection_GraphGames}.
The graph game is defined on all nodes and describes the output given a subset of nodes, where the remaining are masked.
Computing \glspl*{SI} on the graph game defines a perturbation-based and a decomposition-based \gls*{GNN} explanation \citep{DBLP:journals/pami/YuanYGJ23}, which is an extension of node attributions \citep{agarwal2023evaluating}.
In \Cref{subsection_graphshapiq_theory}, we show that \glspl*{GNN} with linear global pooling and output layer satisfy an invariance property for the node game associated with the node embeddings (\cref{theorem_nodegame_invariance}).
This invariance implies sparse \glspl*{MI} for the graph game (Proposition~\ref{proposition_MöbiusTransform_GraphLevel}), which determines the complexity of \glspl*{MI} by the corresponding receptive fields (\cref{theorem_complexity}), which substantially reduces the complexity of \glspl*{SI} in our experiments.
In \cref{subsection_graphshapiq_algorithm}, we introduce GraphSHAP-IQ, an efficient algorithm to exactly compute and 
estimate \glspl*{SI} on GNNs. All proofs are deferred to \cref{appx_sec_proofs}.

\subsection{A Cooperative Game for Shapley Interactions on Graph Neural Networks}
\label{subsection_GraphGames}
Given a \gls*{GNN} $f_g$, we propose the graph game for which we compute \emph{axiomatic and fair} \glspl*{SI}.

\begin{definition}[GNN-induced Graph and Node Game]\label{definition_graphgame}
    For a graph $g = (V,E,\Xm)$ and a \gls*{GNN} $f_g$, we let $N:=\{i: v_i \in V\}$ be the node indices and define the graph game $\nu_g: \Pow(N) \rightarrow \mathbb{R}$ as
    \begin{align*}
        \nu_g(T) := f_{g,\hat y}(\mathbf{X}^{(T)}) \text{ with }
        \mathbf{X}^{(T)} := (\mathbf{x}_1^{(T)}, ..., \mathbf{x}_n^{(T)})^{t} \in \mathbb{R}^{n \times d_0} 
        \text{ and }
        \mathbf{x}_i^{(T)} :=
        \begin{cases}
            \mathbf{x}_i &\text{ if } i \in T, \\
            \mathbf{b} &\text{ if } i \not \in T,
        \end{cases}
    \end{align*}
    with $i\in N$ and baseline $\mathbf{b} \in \mathbb{R}^{d_0}$.
    In graph regression $f_{g,\hat y}\equiv f_g$ and for graph classification $f_{g,\hat y}$ is the component of the predicted class $\hat y$ of $f_g$.
    We further introduce the (multi-dimensional) node game $\nu_i: \Pow(N) \rightarrow \mathbb{R}^{d_\ell}$ as $\nu_i(T) := f_i(\Xm^{(T)})$ for $i \in N$ and each node $v_i \in V$.
\end{definition}

The graph game outputs the prediction of the GNN for a subset of nodes $T \subseteq N$ by masking all node features of nodes $v_i$ with $i \in N \setminus T$ using a suitable baseline $\mathbf{b}$, illustrated in \cref{fig_illustration_graphgame}, left.
Computing such \glspl*{SV} is known as \gls*{BSHAP} \citep{DBLP:conf/icml/SundararajanN20} and a prominent approach for feature attributions \citep{DBLP:conf/nips/LundbergL17,DBLP:journals/jmlr/CovertLL21,Chen2023Overview_ExplainabilityWithShapley}.
As a baseline $\mathbf{b}$, we propose the average of each node feature over the whole graph.
By definition, the prediction of the GNN is given by $\nu_g(N)= f_g(\Xm)$, and due to the efficiency axiom, the sum of contributions in the \gls*{SI}-Graph yields the model's prediction, and thus a decomposition-based \gls*{GNN} explanation \citep{DBLP:journals/pami/YuanYGJ23}.
The graph and the node game are directly linked by \Cref{align_gnn} as
\begin{align}
\label{align_graphnodegame_link}
    \nu_g(T) 
    = f_{g,\hat y}(\Xm^{(T)})
    = \sigma_{\hat y}(\Psi(\{ \{f_i(\Xm^{(T)})\} \mid v \in V \}))
    = \sigma_{\hat y}(\Psi(\{\{\nu_i (T) \} \mid i \in N \})),
\end{align}
where $\sigma_{\hat y}$ outputs the component of $\sigma$ for the predicted class $\hat y$.
The number of convolutional layers $\ell$ determines the \emph{receptive field}, i.e. the message-passing range defined by its $\ell$-hop neighborhood
\begin{align*}
    \Nbh_i^{(\ell)} := \{j \in N \mid d_g(i,j) \leq \ell\} \text{ with } d_g(i,j) := \text{length of shortest path from $v_j$ to $v_i$ in $g$}.
\end{align*}

Consequently, the node game $\nu_i$ is unaffected by maskings outside its $\ell$-hop neighborhood.

\begin{theorem}[Node Game Invariance]
\label{theorem_nodegame_invariance}
    For a graph $g$ and an $\ell$-Layer GNN $f_g$, let $\nu_i$ be the GNN-induced node game with $i \in N$.
    Then, $\nu_i$ satisfies the invariance $\nu_i(T) = \nu_i(T \cap \Nbh_i^{(\ell)})$ for $T \subseteq N$.
\end{theorem}

%with 45% it fits next to the node masking
\begin{wrapfigure}{r}{0.465\textwidth}
\begin{center}
\includegraphics[width=\linewidth]{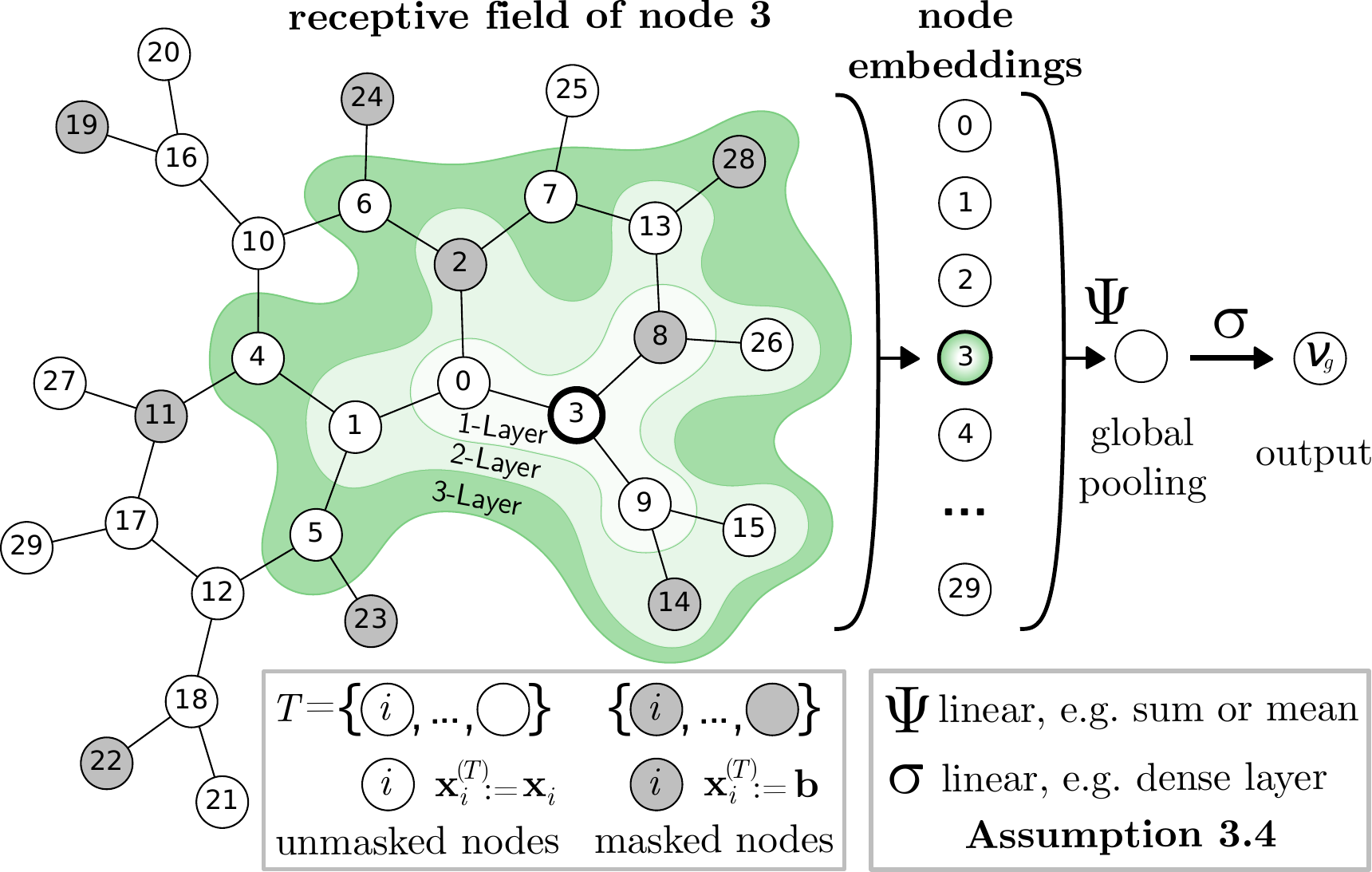}
\end{center}
\caption{Illustration of the graph game $\nu_g$. Masked nodes (grey) are imputed by baseline $\mathbf{b}$, and embeddings are determined by the receptive field (left). Subsequently, a linear pooling ($\Psi$) and output layer ($\sigma$) yield the \gls*{GNN}-induced graph game output.}
\label{fig_illustration_graphgame}
\end{wrapfigure}

\textbf{Node Masking:} Computing \glspl*{SI} on the graph game is a perturbation-based explanation \citep{DBLP:journals/pami/YuanYGJ23}, where also other masking strategies were proposed \citep{agarwal2023evaluating}; for example, node masks \citep{DBLP:conf/nips/YingBYZL19,DBLP:conf/icml/YuanYWLJ21}, edge masks \citep{PGExplainer,schlichtkrull2021interpreting} or node feature masks \citep{agarwal2023evaluating}.
Our method is not limited to a specific masking strategy as long as it defines an invariant game (\cref{theorem_nodegame_invariance}).
We implement our method with the well-established and theoretically understood \gls*{BSHAP} \citep{DBLP:conf/icml/SundararajanN20}.
Alternatively, the $T$-induced subgraph could be used, but \glspl*{GNN} are fit to specific graph topologies, such as molecules, and perform poorly on isolated subgraphs \citep{DBLP:conf/nips/AlsentzerFLZ20}.
Note that, different masking strategies may emphasize different aspects of \glspl*{GNN}, which is important future work.
%Maskings may emphasize different aspects of \glspl*{GNN}, which we leave to future work.
Due to the invariance, we show that \glspl*{MI} and \glspl*{SI} of the graph game are sparse.
To obtain our theoretical results, we require a structural assumption.

\begin{assumption}[GNN Architecture]
\label{assumption_GNN}
    We require the global pooling $\Psi$ and the output layer $\sigma$ to be linear functions, e.g. $\Psi$ is a mean or sum pooling operation and $\sigma$ is a dense layer.
\end{assumption}
\textbf{Linearity Assumption:} In our experiments, we show that popular \gls*{GNN} architectures yield competitive performances under Assumption~\ref{assumption_GNN} on multiple benchmark datasets. 
In fact, such an assumption should not be seen as a hindrance, as it is the norm in GNN benchmark evaluations \citep{Errica2019AFC}. Furthermore, simple global pooling functions, such as \textit{sum} or \textit{mean}, are adopted in many GNN architectures \citep{Xu2017GNNs_GIN, GNNBook2022}, while more sophisticated pooling layers do not always translate into empirical benefits \citep{Mesquita2020RethinkingPI, Grattarola2021UnderstandingPI}.
Likewise, a linear output layer is a common design choice, and the advantage of deeper output layers must be validated for each task \citep{You2020DesignSF}.

\subsection{Computing Exact Shapley and Möbius Interactions for the Graph Game}
\label{subsection_graphshapiq_theory}
Given a GNN-induced graph game $\nu_g$ from Definition~\ref{definition_graphgame} with Assumption~\ref{assumption_GNN}, i.e. $\Psi$ and $\sigma$ are linear, then the \glspl*{MI} of each node game are restricted to the $\ell$-hop neighborhood.
Intuitively, maskings outside the receptive field do not affect the node embedding.
Consequently, we show that the \glspl*{MI} of the graph game are restricted by all existing $\ell$-hop neighborhoods.
More formally, due to the invariance of the node games (\Cref{theorem_nodegame_invariance}), the \glspl*{MI} for subsets that are not fully contained in the $\ell$-hop neighborhood $\Nbh_i^{(\ell)}$ are zero.

\begin{lemma}[Trivial Node Game Interactions]
\label{lemma_MöbiusTransform_NodeLevel}
    Let $m_i: \Pow(N) \rightarrow \mathbb{R}^{d_\ell}$ be the \glspl*{MI} of the GNN-induced node game $\nu_i$ for $i \in N$ under Assumption~\ref{assumption_GNN}. Then, $m_i(S) = \mathbf{0}$ for all $S \not \subseteq \Nbh_i^{(\ell)}$.
\end{lemma}

Lemma~\ref{lemma_MöbiusTransform_NodeLevel} yields that the node game interactions outside of the $\ell$-hop neighborhood do not have to be computed.
Due to Assumption~\ref{assumption_GNN}, the interactions of the GNN-induced graph game are equally zero for subsets that are not fully contained in any $\ell$-hop neighborhood.

\begin{proposition}[Trivial Graph Game Interactions]
\label{proposition_MöbiusTransform_GraphLevel}
    Let $m_g: \Pow(N) \rightarrow \R$ be the \glspl*{MI} of the GNN-induced graph game $\nu_g$ under Assumption~\ref{assumption_GNN} and let ${\mathcal I} := \bigcup_{i \in N} \Pow(\Nbh_i^{(\ell)}) 
    $ be the set of non-trivial interactions.
    Then, $m_g(S) = 0$ for all $S\subseteq N$ with $S \notin {\mathcal I}$.
\end{proposition}

${\mathcal I}$ is the set of non-trivial \glspl*{MI}, whose size depends on the receptive field of the \gls*{GNN}.
The size of $\mathcal I$ also directly determines the required model calls to compute \glspl*{SI}.

\begin{theorem}[Complexity]
\label{theorem_complexity}
 For a graph $g$ and an $\ell$-Layer GNN $f_g$, computing \glspl*{MI} and \glspl*{SI} on the GNN-induced graph game $\nu_g$ requires $\vert {\mathcal I} \vert $ model calls.
 The complexity is thus bounded by
  \begin{align*}
     \vert \mathcal I \vert \leq \sum_{i \in N} 2^{\vert \Nbh_i^{(\ell)}\vert} \leq n \cdot 2^{n^{(\ell)}_{\max}} \leq n \cdot  2^{\frac{d_{\max}^{\ell+1}-1}{d_{\max}-1}},
  \end{align*}
  where $n_{\max}^{(\ell)} := \max_{i \in N} \vert \Nbh_i^{(\ell)} \vert$ is the size of the largest $\ell$-hop neighborhood and $d_{\max}$ is the maximum degree of the graph instance.
\end{theorem}

In other words, \Cref{theorem_complexity} shows that the complexity of \glspl*{MI} (originally $2^n$) for \glspl*{GNN} depends at most  \emph{linearly on the size} of the graph $n$.
Moreover, the complexity depends exponentially on the connectivity $d_{\max}$ of the graph instance and the number of convolutional layers $\ell$ of the \gls*{GNN}.
Note that this is a very rough theoretical bound.
In our experiments, we empirically demonstrate that in practice for many instances  exact \glspl*{SI} can be computed, even for large graphs ($n>100)$.
Besides this upper bound, we empirically show that the \emph{graph density}, which is the ratio of edges compared to the number of edges in a fully connected graph, is an efficient proxy for the complexity.

\subsection{GraphSHAP-IQ: An Efficient Algorithm for Shapley Interactions}
\label{subsection_graphshapiq_algorithm}

\begin{wrapfigure}{r}{0.33\textwidth}
\begin{minipage}[b]{\linewidth}
\vspace{-2.2em}
\begin{algorithm}[H]
    \caption{GraphSHAP-IQ}
    \label{alg_exact_graphshapiq}
    \begin{algorithmic}[1]
    \REQUIRE Graph $g=(V,E,\Xm)$, $\ell$-Layer GNN $f_g$, \gls*{SI} order $k$.
    \STATE $\mathcal I \gets \bigcup_{i \in N} \Pow(\Nbh_i^{(\ell)})$ \label{alg_exact_nonzero_mis}
    \STATE $\nu_g \gets [f_g(\Xm^{(T)})]_{T \in {\mathcal I}}$\label{alg_exact_line_model_calls}
    \STATE $m_g \gets [{\textsc{\texttt{MI}}}(\nu_g, S)]_{S \in {\mathcal I}}$ \label{alg_exact_line_moebius}
    \STATE $\Phi_k \gets {\textsc{\texttt{MItoSI}}}(m_g,k)$ \label{alg_exact_line_conversion}
    \STATE \textbf{return} \glspl*{MI} $m_g$, \glspl*{SI} $\Phi_k$ \label{alg_exact_line_output}
    \end{algorithmic}
\end{algorithm}  
\end{minipage}
\end{wrapfigure}

Building on \cref{theorem_complexity}, we propose GraphSHAP-IQ, an efficient algorithm to compute \glspl*{SI} for \glspl*{GNN}.
At the core is the exact computation of \glspl*{SI}, outlined in \Cref{alg_exact_graphshapiq}, which we then extend for approximation in restricted settings.
Moreover, we propose interaction-informed baseline methods that directly exclude zero-valued \glspl*{SI}, which, however, still require all model calls for exact computation.
To compute exact \glspl*{SI}, GraphSHAP-IQ identifies the set of non-zero \glspl*{MI} $\mathcal I$ based on the given graph instance (line~\ref{alg_exact_nonzero_mis}).
The \gls*{GNN} is then evaluated for all maskings contained in $\mathcal{I}$ (line~\ref{alg_exact_line_model_calls}).
Given these \gls*{GNN} predictions, the \glspl*{MI} for all interactions in $\mathcal I$ are computed (line~\ref{alg_exact_line_moebius}).
Based on the computed \glspl*{MI}, the \glspl*{SI} are computed using the conversion formulas (line~\ref{alg_exact_line_conversion}).
Lastly, GraphSHAP-IQ outputs the exact \glspl*{MI} and \glspl*{SI}.
In restricted settings, computing exact \glspl*{SI} could still remain infeasible.
We thus propose an approximation variant of GraphSHAP-IQ by introducing a hyperparameter $\lambda$, which limits the highest order of computed \glspl*{MI} in line~\ref{alg_exact_nonzero_mis}.
Hence, GraphSHAP-IQ outputs exact \glspl*{SI}, if $\lambda = n_{\max}^{(\ell)}$, thereby requiring the optimal budget.
For a detailed description of GraphSHAP-IQ and the interaction-informed variants, we refer to \cref{appx_sec_graphshapiq_alg}.

\begin{table}[b]
\centering
\caption{Summary of datasets and model accuracy.}
\label{tab_setup}
\resizebox{0.99\columnwidth}{!}{\begin{tabular}{lcccc|ccc|ccc|ccc|ccc}
\toprule
\multicolumn{5}{c|}{\textit{Dataset Description}} & \multicolumn{12}{c}{\textit{Model Accuracy by Layer (\%)}} \\
\multirow{2}{*}{\textbf{Dataset}} & \multirow{2}{*}{\textbf{Graphs}} & \multirow{2}{*}{\textbf{$d_{\text{out}}$}} & \multirow{2}{*}{\textbf{\begin{tabular}[c]{@{}c@{}}Nodes\\ (avg)\end{tabular}}} & \multirow{2}{*}{\textbf{\begin{tabular}[c]{@{}c@{}}Density\\ (avg)\end{tabular}}} & \multicolumn{3}{c}{\textbf{GCN}} & \multicolumn{3}{c}{\textbf{GAT}} & \multicolumn{3}{c}{\textbf{GIN}} & \multicolumn{3}{c}{\textbf{Speed-Up}} \\
 &  &  &  &  & \textbf{1} & \textbf{2} & \textbf{3} & \textbf{1} & \textbf{2} & \textbf{3} & \textbf{1} & \textbf{2} & \textbf{3} & \textbf{1} & \textbf{2} & \textbf{3} \\ \midrule
\begin{tabular}[c]{@{}l@{}}\gls{BNZ} \\ \citep{DBLP:conf/nips/Sanchez-Lengeling20}\end{tabular} & $12000$ & $2$ & $20.6$ & $22.8$ & $84.2$ & $88.6$ & $90.4$ & $83.1$ & $85.1$ & $85.7$ & $84.9$ & $90.5$ & $90.8$ & $10^4$ & $10^3$ & $10^2$ \\
\rowcolor{tablegray}\begin{tabular}[c]{@{}l@{}}\gls{FLC} \\ \citep{DBLP:conf/nips/Sanchez-Lengeling20}\end{tabular} & $8671$ & $2$ & $21.4$ & $21.6$ & $82.4$ & $83.9$ & $83.9$ & $82.4$ & $82.2$ & $82.4$ & $84.6$ & $87.2$ & $87.1$ & $10^4$ & $10^3$ & $10^2$ \\
\begin{tabular}[c]{@{}l@{}}\gls{MTG} \\\citep{Kazius_McGuire_Bursi_2005}\end{tabular} & $4337$ & $2$ & $30.3$ & $18.3$ & $77.8$ & $80.7$ & $80.3$ & $72.6$ & $73.6$ & $74.8$ & $77.8$ & $77.4$ & $77.5$ & $10^5$ & $10^4$ & $10^2$ \\
\rowcolor{tablegray}\begin{tabular}[c]{@{}l@{}}\gls{ALC} \\ \citep{DBLP:conf/nips/Sanchez-Lengeling20}\end{tabular} & $1125$ & $2$ & $21.4$ & $21.5$ & $98.7$ & $97.8$ & $99.1$ & $98.2$ & $96.3$ & $97.3$ & $96.9$ & $97.3$ & $97.8$ & $10^4$ & $10^3$ & $10^2$ \\
\begin{tabular}[c]{@{}l@{}}\gls{PRT} \\ \citep{DBLP:conf/ismb/BorgwardtOSVSK05}\end{tabular} & $1113$ & $2$ & $39.1$ & $42.4$ & $75.2$ & $71.1$ & $74.0$ & $75.3$ & $60.5$ & $79.8$ & $79.3$ & $74.9$ & $67.7$ & $10^5$ & $10^3$ & $10^2$ \\
\rowcolor{tablegray}\begin{tabular}[c]{@{}l@{}}\gls{ENZ} \\ \citep{DBLP:conf/ismb/BorgwardtOSVSK05}\end{tabular} & $600$ & $6$ & $32.6$ & $32.0$ & $34.2$ & $37.5$ & $35.8$ & $32.5$ & $35.0$ & $35.8$ & $36.7$ & $35.0$ & $39.2$ & $10^6$ & $10^4$ & $10^3$  \\
\begin{tabular}[c]{@{}l@{}}\gls{CX2} \\ \citep{DBLP:journals/jcisd/SutherlandOW03}\end{tabular} & $467$ & $2$ & $41.2$ & $10.6$ & $87.2$ & $86.1$ & $87.2$ & $81.9$ & $87.2$ & $85.1$ & $84.0$ & $85.1$ & $85.1$ & $10^9$ & $10^8$ & $10^6$ \\
\rowcolor{tablegray}\begin{tabular}[c]{@{}l@{}}\gls{BZR} \\ \citep{DBLP:journals/jcisd/SutherlandOW03}\end{tabular} & $405$ & $2$ & $35.8$ & $13.0$ & $90.1$ & $87.7$ & $90.2$ & $88.9$ & $86.4$ & $87.7$ & $88.9$ & $88.9$ & $88.9$ & $10^8$ & $10^6$ & $10^4$ \\ \bottomrule
\end{tabular}}
\end{table}

\begin{figure}[t]
    \centering
    \hfill
    \begin{minipage}[c]{0.26\textwidth}
    \centering
        \includegraphics[width=\textwidth]{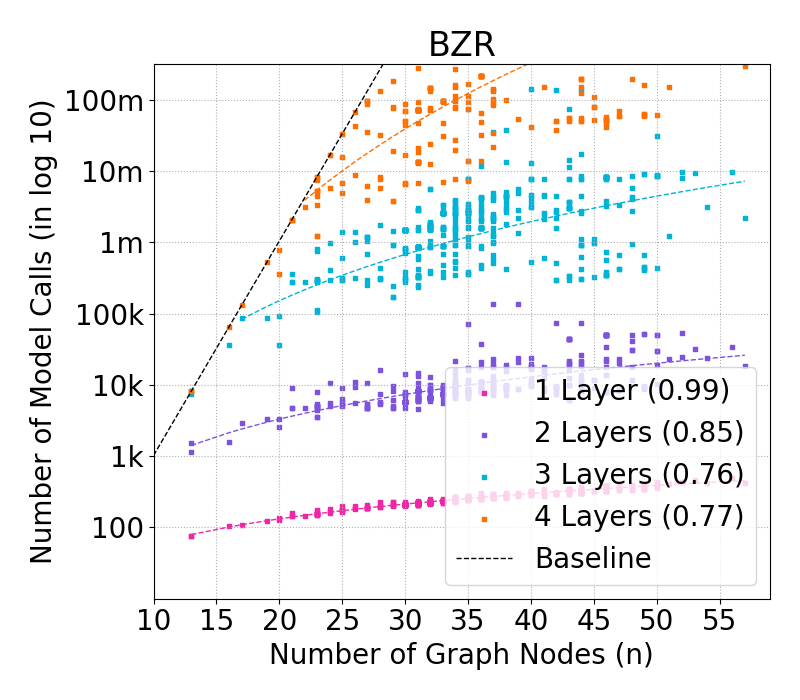}
    \end{minipage}
    \hspace{1em}
    \begin{minipage}[c]{0.26\textwidth}
    \centering
        \includegraphics[width=\textwidth]{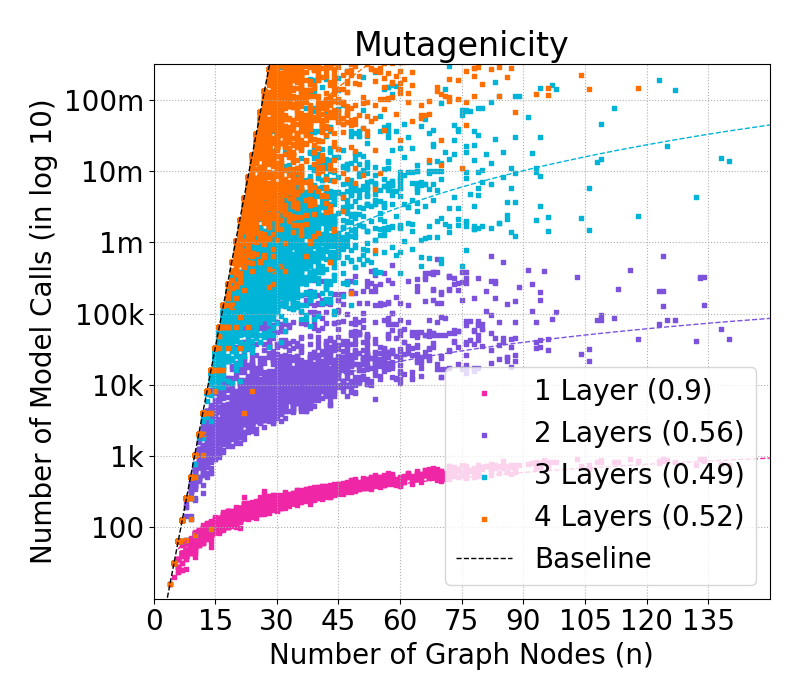}
    \end{minipage}
    \hspace{1em}
    \begin{minipage}[c]{0.26\textwidth}
    \centering
        \includegraphics[width=\textwidth]{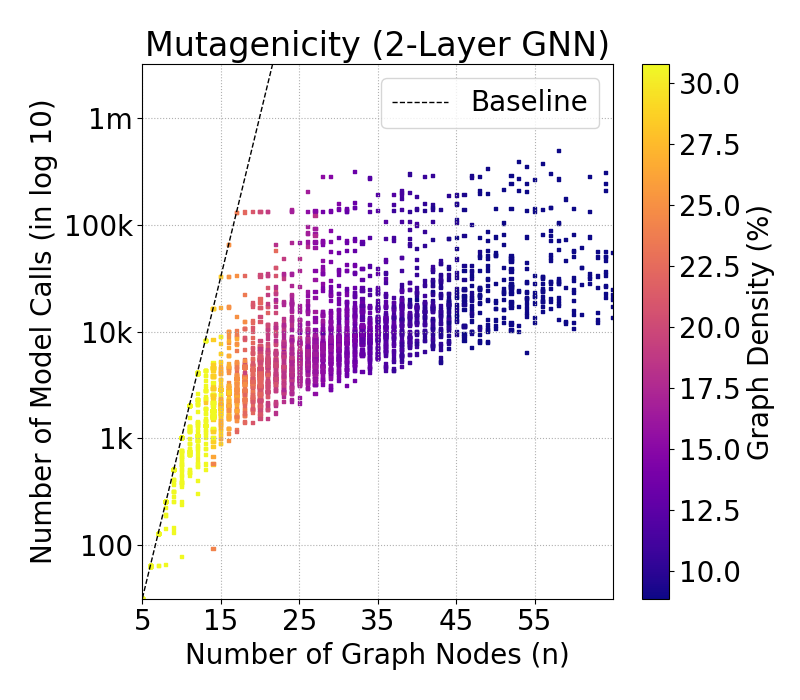}
    \end{minipage}
    \hfill
    \caption{
    Complexity of GraphSHAP-IQ in model calls (in $\log10$) by number of nodes for all graphs of \gls*{BZR} (left) and \gls*{MTG} (middle, right) visualized by number of convolutional layers (left, middle) and graph density \gls*{GNN} (right). While model-agnostic baselines scale exponentially (dashed lines), GrahpSHAP-IQ scales approximately linearly with graph sizes ($R^2$ of log-curves in braces).
    }
    \label{fig_complexity}
\end{figure}

\section{Experiments}
In this section, we empirically evaluate GraphSHAP-IQ for \gls*{GNN} explainability, and showcase a substantial reduction in complexity for exact \glspl*{SI} (\cref{sec_exp_complexity}), benefits of approximation (\cref{sec_approximation}), and explore the \gls*{SI}-Graph for \glspl*{WDN} and molecule structures (\cref{sec_exp_application}). 
Following~\citet{GraphFramEx}, we trained
a \gls*{GCN} \citep{Kipf2017GNNs_GCN}, \gls*{GIN} \citep{Xu2017GNNs_GIN}, and \gls*{GAT} \citep{Velickovic2017GNNs_GAT} on eight real-world chemical datasets for graph classification and a \gls*{WDN} for graph regression, cf. \cref{tab_setup}.
%We trained all models under Assumption~\ref{assumption_GNN} and reported test accuracies, that are comparable to existing benchmarks~\citep{Errica2019AFC,You2020DesignSF}.
All models adhere to Assumption~\ref{assumption_GNN} and report comparable test accuracies~\citep{Errica2019AFC,You2020DesignSF}.
All experiments are based on \texttt{shapiq} \citep{muschalik2024shapiq} and details can be found in \cref{appx_exp} or at {\footnotesize\url{https://github.com/FFmgll/GraphSHAP-IQ}}.

\subsection{Complexity Analysis of GraphSHAP-IQ for Exact Shapley Interactions}\label{sec_exp_complexity}
In this experiment, we empirically validate the benefit of exploiting graph and \gls*{GNN} structures to compute \emph{exact \glspl*{SI}} with GraphSHAP-IQ.
The complexity is measured by the number of evaluations of the GNN-induced graph game, i.e. the number of model calls of the \gls*{GNN}, which is the limiting factor of \glspl*{SI}, cf. runtime analysis in Appendix~\ref{appx_sec_runtime_analysis}.
For every graph in the benchmark datasets, described in \cref{tab_setup}, we compute the complexity of GraphSHAP-IQ, where the first upper bound from \cref{theorem_complexity} is used if $\max_{i\in N} \vert \Nbh_i^{(\ell)} \vert > 23$, i.e. the complexity exceeds $2^{23} \approx 8.3\times 10^6$.
\cref{fig_complexity} displays the log-scale complexity (y-axis) by the number of nodes $n$ (x-axis) for \gls*{BZR} (left) and \gls*{MTG} (middle, right) for varying number of convolutional layers $\ell$ (left, middle) and by graph density for a 2-Layer \gls*{GNN} (right).
The model-agnostic baseline is represented by a dashed line.
For results on all datasets, see \cref{appx_sec_complexity}.
\cref{fig_complexity} shows that the computation of \glspl*{SI} is substantially reduced by GraphSHAP-IQ.
Even for large graphs with more than $100$ nodes, where the baseline requires over $10^{30}$ model calls, many instances can be exactly computed for $1$-Layer and $2$-Layer \glspl*{GNN} with fewer than $10^5$ evaluations.
In fact, the complexity grows \emph{linearly} with graph size across the dataset, as shown by high $R^2$ scores of fitted logarithmic curves.
\Cref{fig_complexity} (right) shows that the graph density
is an efficient proxy of complexity, with higher values for instances near the baseline.

\subsection{Interaction-Informed Approximation of Shapley Interactions}\label{sec_approximation}

For densely connected graphs and \glspl*{GNN} with many layers, exact computation of \glspl*{SI} might still be infeasible.
We, thus, evaluate the \emph{approximation of \glspl*{SI}} with GraphSHAP-IQ, current state-of-the-art model-agnostic baselines (implemented in \texttt{shapiq}), and our proposed interaction-informed variants.
For the \gls*{SV} (order 1), we apply \textit{KernelSHAP} \citep{DBLP:conf/nips/LundbergL17}, \textit{Unbiased KernelSHAP} \citep{Covert.2021}, \textit{k-additive SHAP} \citep{Pelegrina.2023}, \textit{Permutation Sampling} \citep{Castro.2009}, \textit{SVARM} \citep{DBLP:conf/aaai/KolpaczkiBMH24}, and \textit{L-Shapley} \citep{DBLP:conf/iclr/ChenSWJ19}. 
We estimate \gls*{k-SII} (order 2 and 3) with \textit{KernelSHAP-IQ} \citep{fumagalli2024kernelshapiq}, \textit{Inconsistent KernelSHAP-IQ} \citep{fumagalli2024kernelshapiq}, \textit{Permutation Sampling} \citep{Tsai.2022}, \textit{SHAP-IQ} \citep{Fumagalli.2023}, and \textit{SVARM-IQ} \citep{Kolpaczki.2024}.
For each baseline, we use the interaction-informed variant, cf. \cref{appx_sec_interaction_informed}.
We select graphs containing $20 \leq n \leq 40$ nodes for the \textit{MTG}, \textit{PRT}, and \gls*{BZR} benchmark datasets. 
For each graph instance, we compute ground-truth \glspl*{SI} via GraphSHAP-IQ and evaluate all methods using the same number of model calls, which is the main driver of runtime, cf. Appendix~\ref{appx_sec_runtime_analysis}.
\cref{fig_approximation} (left) displays the average MSE (lower is better) for varying  the model calls.
GraphSHAP-IQ outperforms the baselines in settings with a majority of lower-order \glspl*{MI}.
\cref{fig_approximation} (middle) compares average runtime and MSE for varying explanation orders at GraphSHAP-IQ's ground-truth budget.
Notably, GraphSHAP-IQ is among the fastest methods, and remains unaffected by increasing explanation order.
Moreover, for all baselines (except permutation sampling), the interaction-informed variants substantially improve the approximation quality and runtime.
Consequently, noisy estimates of \glspl*{SI} are substantially improved (\cref{fig_approximation}, right).

\begin{figure}[t]
    \centering
    \begin{minipage}[c]{0.69\textwidth}
    \includegraphics[width=\textwidth]{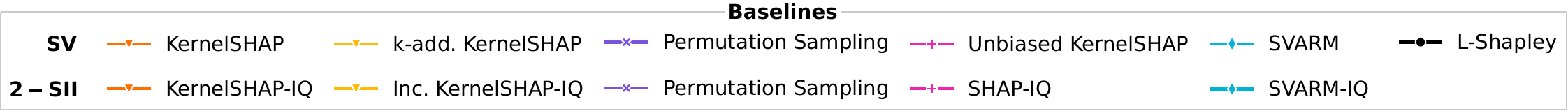}
    \begin{minipage}[c]{0.49\textwidth}
    \includegraphics[width=\textwidth]{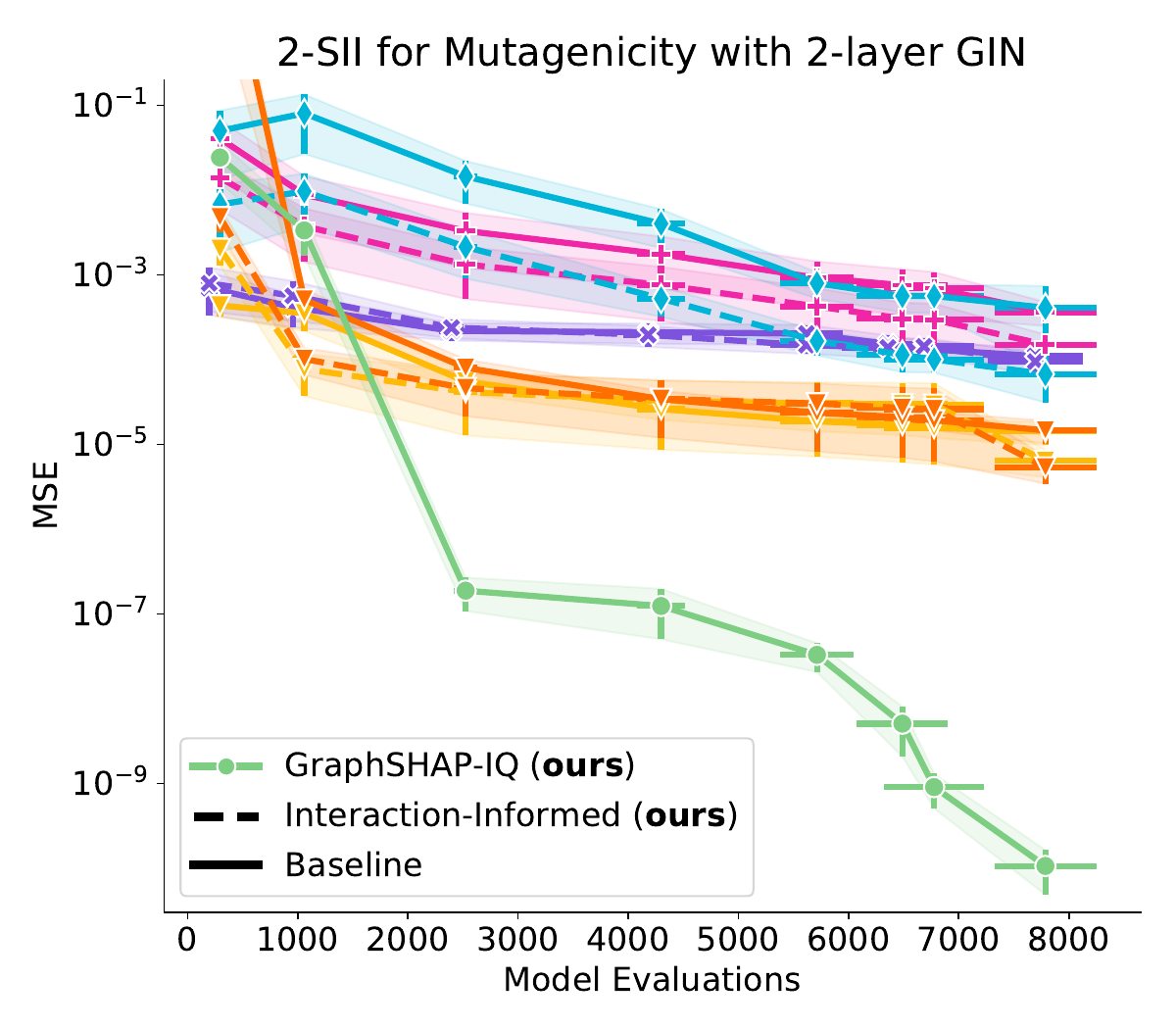}
    \end{minipage}
    \hfill
    \begin{minipage}[c]{0.49\textwidth} 
    \includegraphics[width=\textwidth]{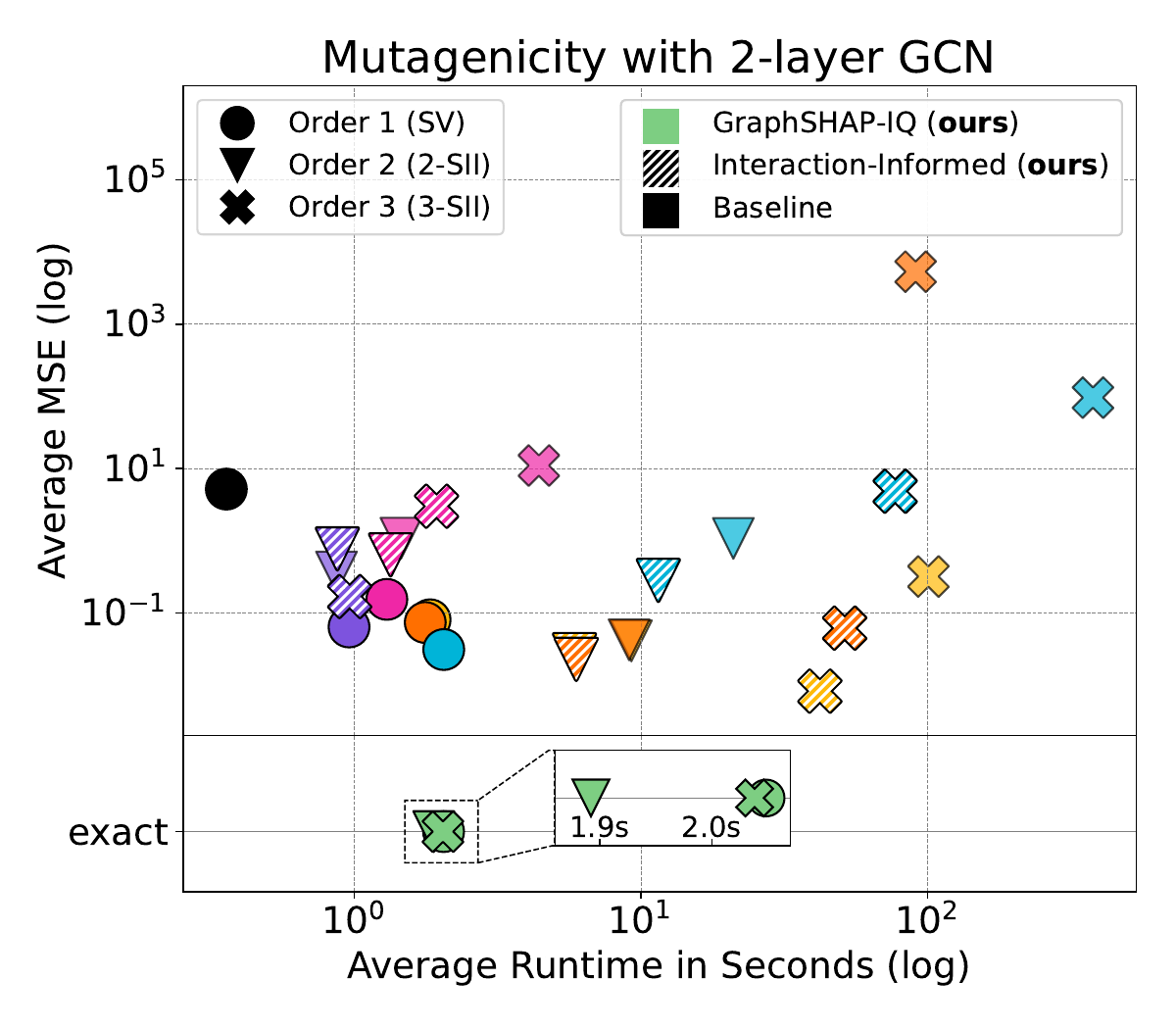}
    \end{minipage}
    \end{minipage}
    \hfill
    \begin{minipage}[c]{0.3\textwidth}
    \includegraphics[width=\textwidth]{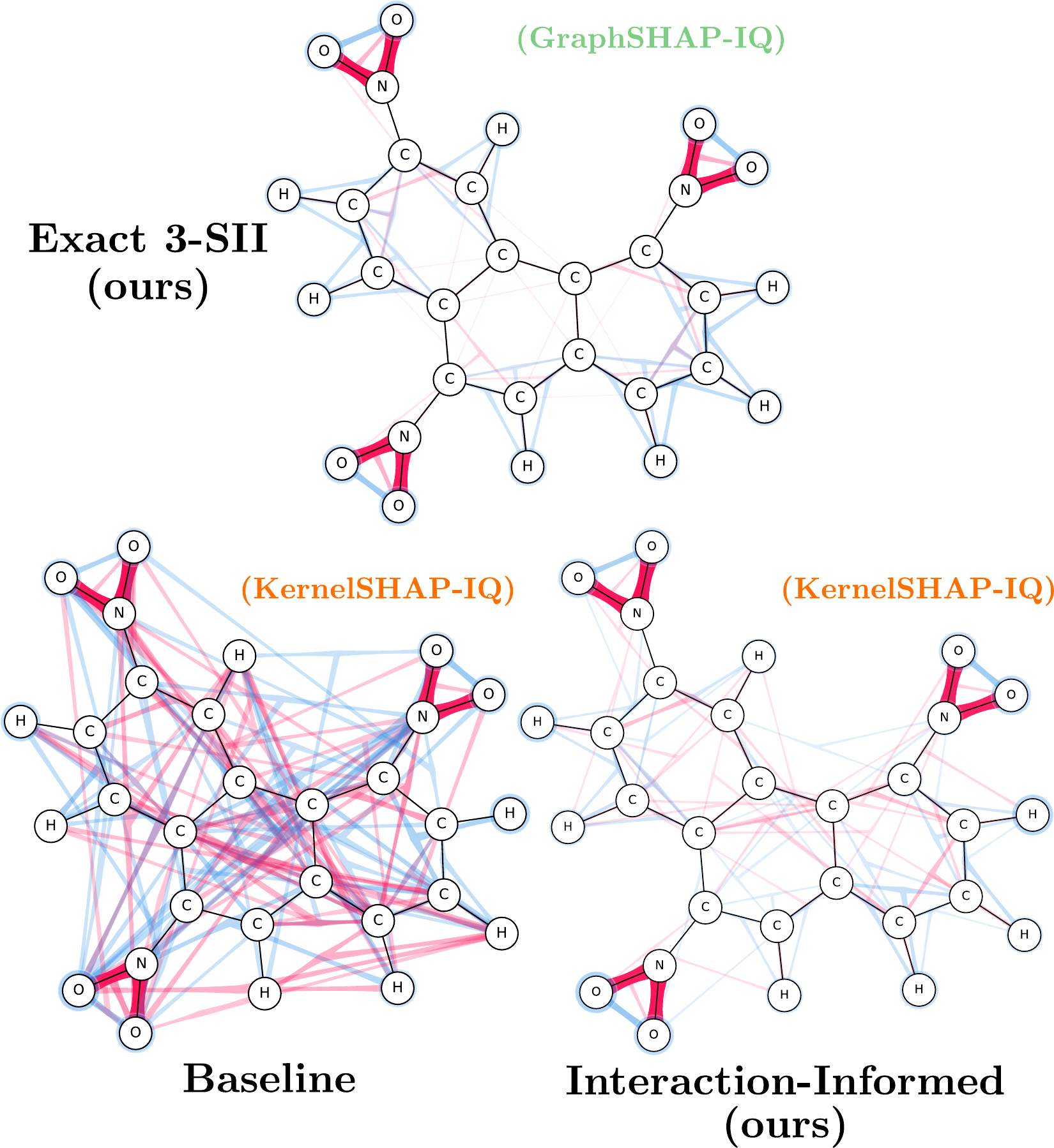}
    \end{minipage}
    \caption{Approximation of \glspl*{SI} with GraphSHAP-IQ (\textcolor{graphshapiq}{green}) and model-agnostic baselines for \emph{MTG} (left). 
    At budgets, where GraphSHAP-IQ reaches exact \glspl*{SI}, the baselines achieve varying estimation qualities and computational costs (middle) which leads to different estimated explanations, especially without interaction-informed baselines (right).
    }
    \label{fig_approximation}
\end{figure}

\subsection{Real-World Applications of Shapley Interactions and the SI-Graph}\label{sec_exp_application}

\begin{figure}[t]
    \centering
    \begin{minipage}[b]{0.27\textwidth}
    \centering
        \includegraphics[width=\textwidth]{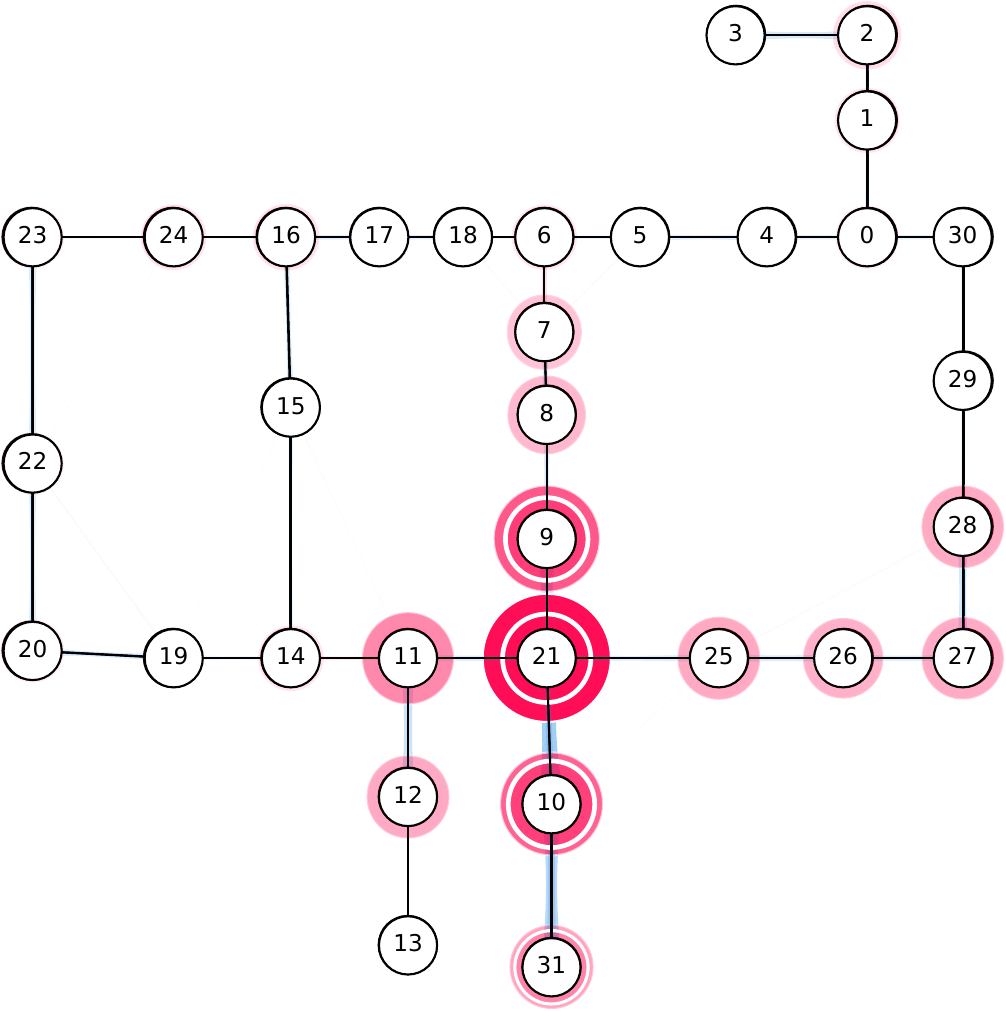}
        \\
         \textbf{(a) WDN Explanations}
    \end{minipage}
    \hfill
    \begin{minipage}[b]{0.2\textwidth}
        \centering
        \includegraphics[width=\textwidth]{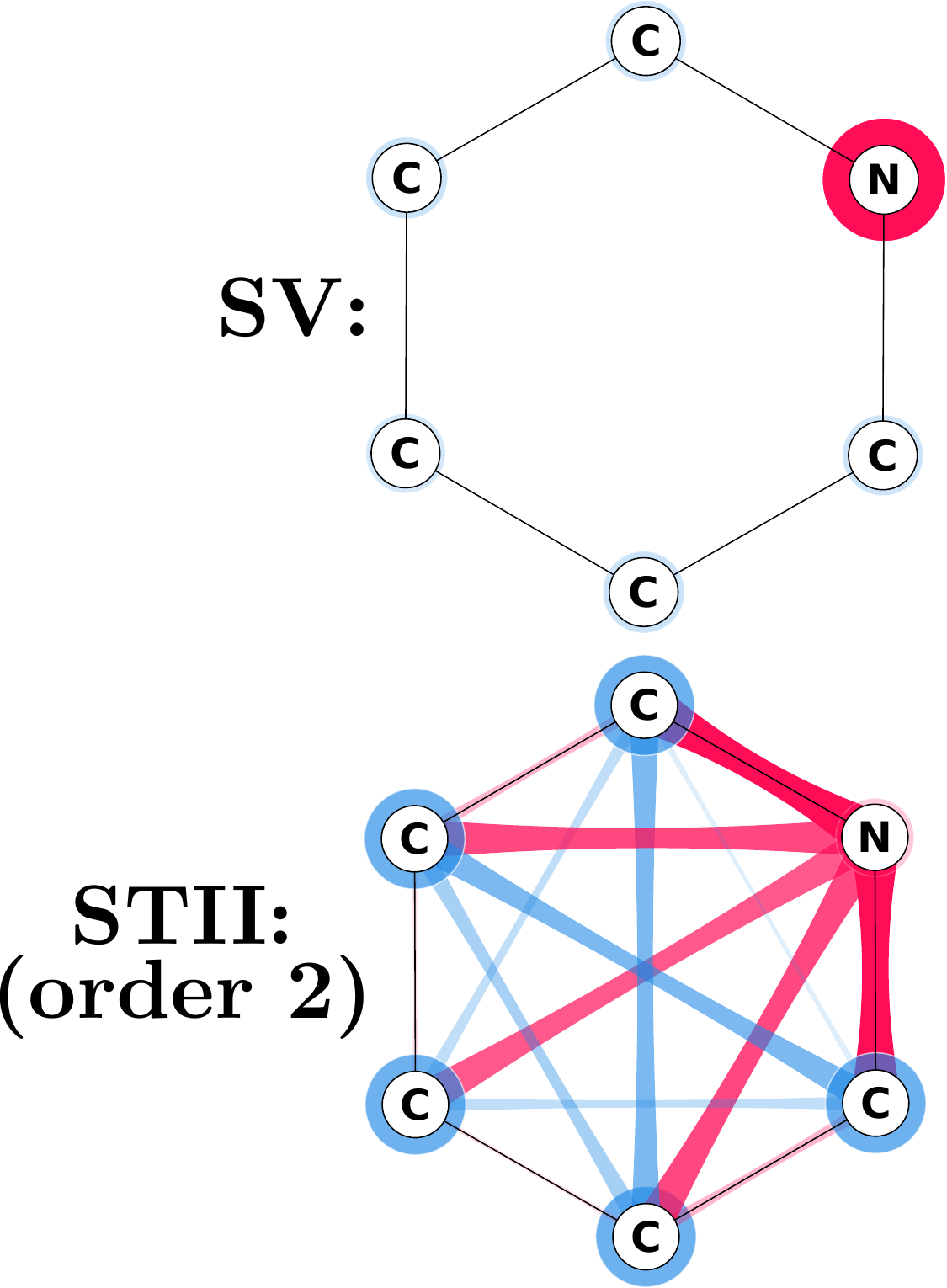}
    \textbf{(b) Pyridine Molecule}
    \end{minipage}
    \hfill
    \begin{minipage}[b]{0.37\textwidth}
        \centering
        \includegraphics[width=\textwidth]{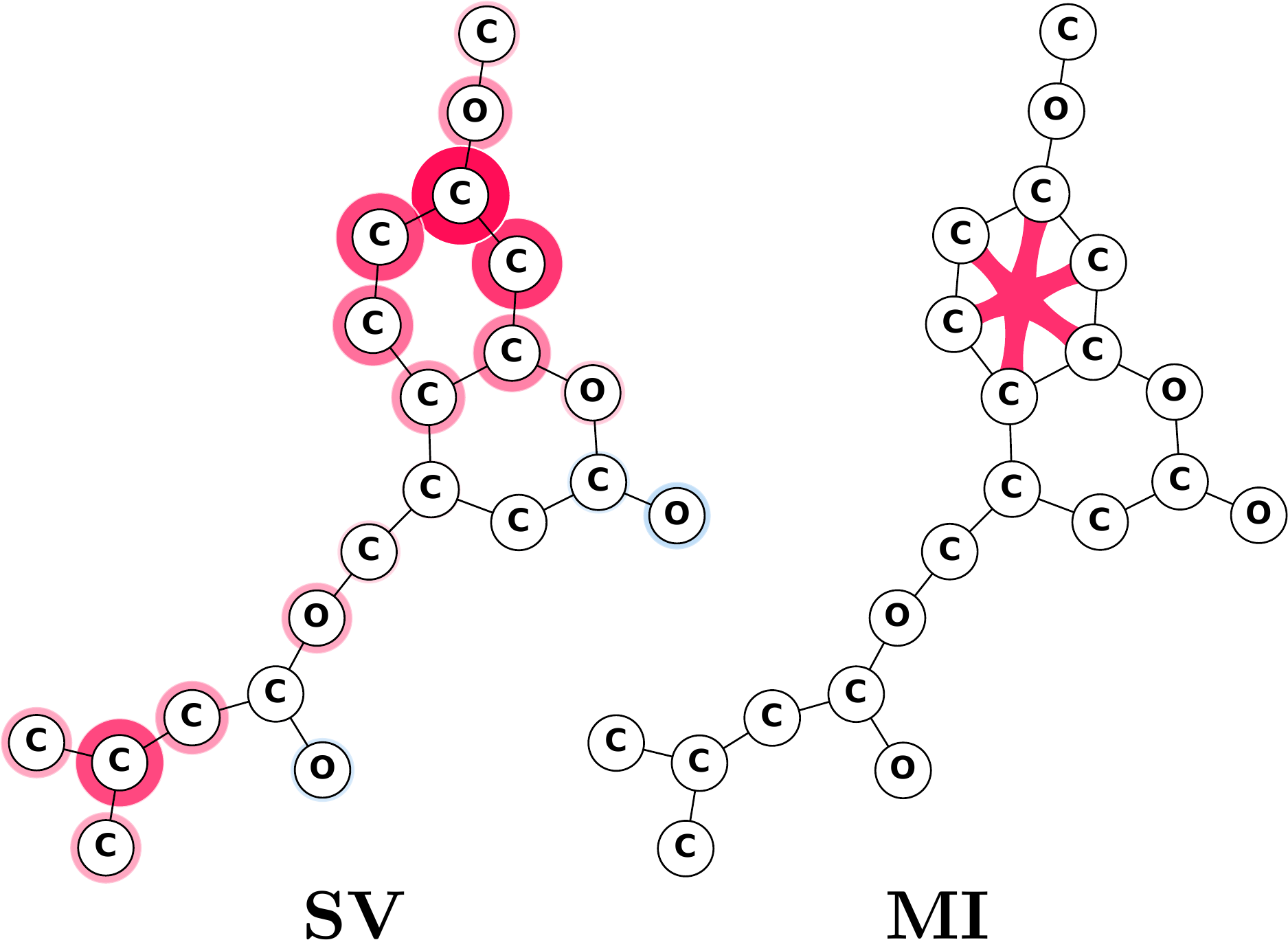}
    \textbf{(c) Benzene Molecule}
    \end{minipage}
    \caption{Exact \glspl*{SI} values for three example graph structures. \glspl*{SV} illustrate the trajectory of chlorination levels in a WDN (a). STII (order 2) values showcase that a Pyridine molecule is not classified as benzene (b), and the largest positive \gls*{MI} for a benzene molecule is the benzene ring (c).}
    \label{fig_xai_example}
\end{figure}

We now apply GraphSHAP-IQ in real-world applications and include further results in \cref{appx_sec_additional_experimental_results}.
\textbf{Monitoring water quality} in \glspl*{WDN} requires insights into a dynamic system governed by local partial differential equations.
Here, we investigate the spread of chlorine as a graph-level regression of a \gls*{WDN}, where a \gls*{GNN} predicts the fraction of nodes chlorinated after some time.
Based on the Hanoi WDS \citep{Vrachamis2018WaterDistributionNetworks_Hanoi}, we create a temporal \gls*{WAQ} dataset containing $1\,000$ graphs consisting of 30 time steps.
We train and explain a simple \gls*{GNN}, which processes node and edge features like chlorination level at each node and water flow between nodes.
\cref{fig_xai_example,appx_fig_wdn} show that $2$-\glspl*{SII} spread over the WDS aligned with the water flow. 
Therein, mostly first-order interactions influence the time-varying chlorination levels.
\\
\textbf{Benzene rings in molecules} are structures consisting of six carbon (C) atoms connected in a ring with alternating single and double bonds. We expect a well-trained \gls*{GNN} to identify benzene rings to incorporate higher-order \glspl*{MI} (order $\geq6$).
\cref{fig_xai_example} shows two molecules and their \gls*{SI}-Graphs computed by GraphSHAP-IQ.
The Pyridine molecule in \cref{fig_xai_example} (b) is correctly predicted to be non-benzene as the hexagonal configuration features a nitrogen (N) instead of a carbon, which is confirmed by the \glspl*{SV} highlighting the nitrogen.
\glspl*{STII} of order 2 reveal that the \gls*{MI} of nitrogen is zero and interactions with neighboring carbons are non-zero, presumably due to higher-order \glspl*{MI}, since \gls*{STII} distributes all higher-order \glspl*{MI} to the pairwise \glspl*{STII}.
In addition, \glspl*{STII} among the five carbon atoms impede the prediction towards the benzene class.
Interestingly, opposite carbons coincide with the highest negative interaction. 
The \glspl*{MI} for a benzene molecule with 21 atoms in \cref{fig_xai_example} (c) reveal that the largest positive \gls*{MI} coincides with the 6-way \gls*{MI} of the benzene ring.
\\
\textbf{Mutagenicity of molecules} is influenced by compounds like nitrogen dioxide (NO$_2$) \citep{Kazius_McGuire_Bursi_2005}.
\cref{fig_intro_illustration} shows \glspl*{SI} for a \gls*{MTG} molecule, which \gls*{GNN} identifies as mutagenic.
$2$-\glspl*{SII} and \glspl*{MI} both show that not the nitrogen atom but the interactions of the NO$_2$ bonds contributed the most.

\section{Comparison of Linear and Deep Readouts}

\begin{figure}
    \centering
    \includegraphics[width=\linewidth]{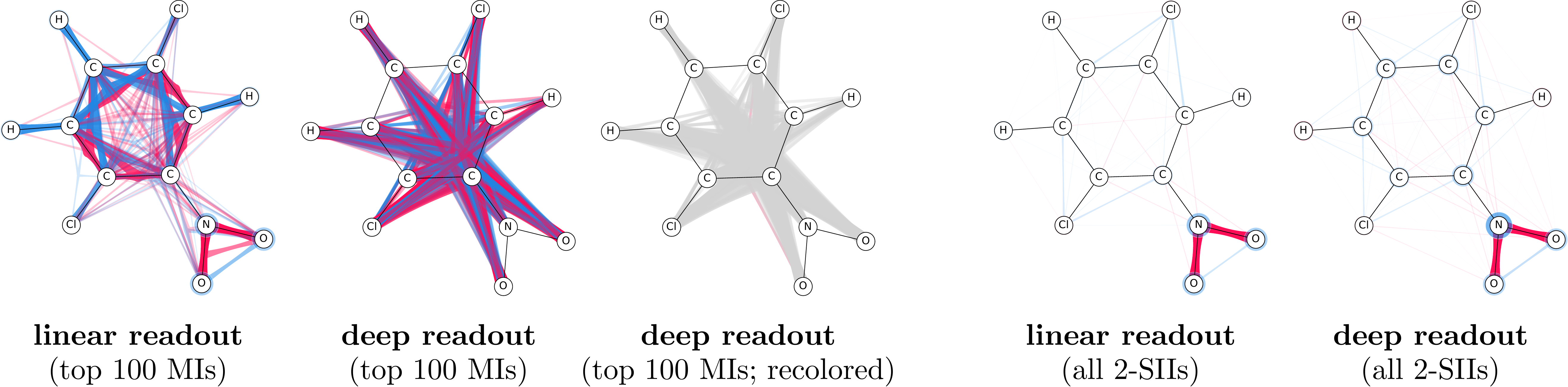}
    \caption{\glspl*{MI} and $2$-\gls*{SII} scores for a 2-layer \gls*{GCN} trained on \gls*{MTG} with a linear readout and a deep readout. Deep readout produces interactions outside the receptive field which are recolored in grey.}
    \label{fig_dr_results}
\end{figure}

\cref{assumption_GNN} imposes the use of a linear readout, which is a limitation of our method, though it remains commonly used in practice \citep{You2020DesignSF}.
In this section, we compare \glspl*{SI} for \glspl*{GNN} with both linear and non-linear readouts.
We train a 2-layer \gls*{GCN} architecture with both linear and non-linear (2-layer perceptron) readouts on \gls*{MTG}, where both models achieve comparable performance.
Figure~\ref{fig_dr_results} shows that non-linear readouts produce substantially different \glspl*{MI} that are \emph{not restricted} to the receptive fields.
Lower-order explanations are similar, indicating the correct reasoning of both models (NO$_2$ group signalling mutagenicity).
We conclude that linear readout restricts the interactions of the \gls*{GNN} to the graph structure and its receptive fields.
In contrast, non-linear readouts enable interactions that extend beyond the receptive fields of the \gls*{GNN}.
Many \gls*{SV}-based \gls*{XAI} methods for \glspl*{GNN} \citep{DBLP:conf/icml/YuanYWLJ21,DBLP:conf/nips/ZhangLSS22,DBLP:conf/nips/YeHWL23,DBLP:conf/icml/BuiNNY24} implicitly rely on the assumption that interactions outside the receptive fields are negligible, which should be evaluated carefully for non-linear readouts.

\section{Limitations and Future Work}
We presented GraphSHAP-IQ, an efficient method to compute \glspl*{SI} that applies to all popular message passing techniques in conjunction with a linear global pooling and output layer.
\textbf{Assumption~\ref{assumption_GNN}} is a common choice for \glspl*{GNN} \citep{Errica2019AFC, Xu2017GNNs_GIN, GNNBook2022} and does not necessarily yield lower performance \citep{Mesquita2020RethinkingPI,You2020DesignSF, Grattarola2021UnderstandingPI}, which is confirmed by our experiments. 
However, exploring non-linear choices that preserve trivial \glspl*{MI} is important for future research.
\textbf{Masking node features with a fixed baseline}, known as \gls*{BSHAP}, preserves the topology of the graph structure and is a well-established approach \citep{Sundararajan.2020}.
Nevertheless, alternatives such as induced subgraphs, edge-removal, or learnable masks, could emphasize other properties of the \gls*{GNN}.
Lastly, \textbf{approximation of \glspl*{SI}} with GraphSHAP-IQ and interaction-informed baselines substantially improved the estimation, where novel methods tailored to Proposition~\ref{proposition_MöbiusTransform_GraphLevel} are promising future work.
Our results may further be applied to other models with spatially restricted features, such as convolutional neural networks.

\section{Conclusion}
We introduced the \gls*{GNN}-induced graph game, a cooperative game for \glspl*{GNN} on graph prediction tasks that outputs the model's prediction given a set of nodes.
The remaining nodes are masked using a baseline for node features, corresponding to the well-established \gls*{BSHAP} \citep{Sundararajan.2020}.
We showed that under linearity assumptions on global pooling and output layers, the complexity of computing exact \glspl*{SV}, \glspl*{SI}, and \glspl*{MI} on any \gls*{GNN} is determined solely by the receptive fields.
Based on our theoretical results, we presented GraphSHAP-IQ and interaction-informed variants of existing baselines, to efficiently compute any-order \glspl*{SI} for \glspl*{GNN}.
We show that GraphSHAP-IQ and interaction-informed baselines substantially reduces the complexity of \glspl*{SI} on multiple real-world benchmark datasets and propose to visualize \glspl*{SI} as the \gls*{SI}-Graph.
By computing the \gls*{SI}-Graph, we discover trajectories of chlorine in \glspl*{WDN} and important molecule substructures, such as benzene rings or NO$_2$ groups.
\clearpage 

\section*{Ethics Statement}

This paper presents work aiming to advance the field of \gls*{ML} and specifically the field of \gls*{XAI}.
There are many potential societal consequences of our work.
Our research holds significant potential for positive societal impact, particularly in areas like the natural sciences (e.g., chemistry and biology) and network analytics.
Our work can positively impact \gls*{ML} adoption and potentially reveal biases or unwanted behavior in \gls*{ML} systems.

However, we recognize that the increased explainability provided by \gls*{XAI} also carries ethical risks. 
There is the potential for ``explainability-based white-washing'', where organizations, firms, or institutions might misuse \gls*{XAI} to justify questionable actions or outcomes.
With responsible use, XAI can amplify the positive impacts of \gls*{ML}, ensuring its benefits are realized while minimizing harm.
 
\section*{Reproducibility Statement}
The python code for GraphSHAP-IQ is available at \url{https://github.com/FFmgll/GraphSHAP-IQ}, and can be used on any \emph{graph game}, a class specifically tailored to the \texttt{shapiq} package \citep{muschalik2024shapiq}.
We include formal proofs of all claims made in the paper in \cref{appx_sec_proofs}.
We further describe the experimental setup and details regarding reproducibility in \cref{appx_exp}.
Our experimental results, setups and plots can be reproduced by running the corresponding scripts.
The datasets and their sources are described in \cref{tab_dataset_licenses}.

\subsubsection*{Acknowledgments}
We gratefully thank the anonymous reviewers for their valuable feedback for improving this work!
We thank André Artelt for the EPyT-Flow toolbox support.
Fabian Fumagalli and Maximilian Muschalik gratefully acknowledge funding by the Deutsche Forschungsgemeinschaft
(DFG, German Research Foundation): TRR 318/1 2021 – 438445824. Janine Strotherm gratefully acknowledges funding from the \gls*{ERC} under the \gls*{ERC} Synergy Grant Water-Futures (Grant agreement No. 951424).
Luca Hermes and Paolo Frazzetto gratefully acknowledge funding by ``SAIL: SustAInable Lifecycle of Intelligent Socio-Technical Systems,'' funded by the Ministry of Culture and Science of the State of North Rhine-Westphalia under grant NW21-059A. 
Additionally, Paolo Frazzetto gratefully acknowledges funding by Amajor SB S.p.A. Alessandro Sperduti gratefully acknowledges the support of the PNRR project FAIR - Future AI Research (PE00000013), Concession Decree No. 1555 of October 11, 2022, CUP C63C22000770006.

\bibliography{references.bib}
\bibliographystyle{iclr2025_conference}

\clearpage
\appendix
\onecolumn

\section*{Organisation of the Supplement Material}

\startcontents[sections]
\printcontents[sections]{l}{1}{\setcounter{tocdepth}{3}}
\clearpage

\clearpage

\section{Notation}\label{appx_sec_notation}
We use lower-case letters to represent the cardinalities of the subsets, e.g. $s := \vert S \vert$.
A summary of notations is given in \cref{appx_tab_notations}.

% \usepackage{booktabs}
\begin{table}[h!]
\caption{Notation Table}\label{appx_tab_notations}
\vspace{0.5em}
\resizebox{\textwidth}{!}{
\begin{tabular}{@{}ll@{}}
\toprule
\textbf{Notation}  & \textbf{Description}  
\\ \midrule
\multicolumn{2}{c}{\textbf{Data Notations}}
\\ \midrule
$g = (V,E,\mathbf{X})$ & Graph instance to explain, with nodes $V$, edges $E$, and node features $\mathbf{X}$ 
\\
$V := \{v_1,\dots,v_n\}$ & Set of nodes in $g$
\\
$E \subset V \times V$ & Set of edges in $g$
\\
$d_0$ & Number of node features
\\
$\mathbf{x_i} \in \mathbb{R}^{d_0}$ & Node feature vector of node $v_i$
\\
$\mathbf{X} \in \mathbb{R}^{n \times d_0}$ & Node feature matrix for all nodes in $g$
\\
$\mathbf{b} \in \mathbb{R}^{d_0}$ & Baseline node features used for masking 
\\
$N := \{i: v_i \in V\}$ & Set of graph node indices 
\\
$n := \vert N \vert$ & Number of nodes in the graph instance 
\\
$i \in N$ & Index of node $v_i \in V$ in $g$
\\
$\mathcal N(v) \subseteq N$ & Node indices of the neighborhood of node $v$ 
\\ \midrule
\multicolumn{2}{c}{\textbf{\gls*{GNN} Notations}}
\\ \midrule
$f_g(\mathbf{X})$ & \gls*{GNN} prediction for graph $g$ and node features $\mathbf{X}$
\\
$f_{g,\hat y}(\mathbf{X})$ & \gls*{GNN} prediction of class $\hat y$ for graph $g$ and node features $\mathbf{X}$
\\
$\ell$ & Number of graph convolutional layers
\\
$k=1,\dots,\ell$ & Intermediate graph convolutional layers
\\
$d_k$ & Number of features in node embedding at layer $k$
\\
$\mathbf{h}_i^{(k)} \in \mathbb{R}^{d_{k}}$ & Node embedding vector for node $v_i$ at layer $k$ 
\\
$\mathbf{H}^{(k)} := [\mathbf{h}_1^{(k)},\dots,\mathbf{h}_n^{(k)}]$ & Node embedding matrix of the graph instance at layer $k$ 
\\
$\Psi$ & (Linear) permutation-invariant global pooling function 
\\
$\sigma$ & (Linear) readout layer of the GNN 
\\
$\mathcal N_i^{(\ell)} \subseteq N$ & Node indices of the $\ell$-hop neighborhood of node $v_i$ 
\\
$n_{\max}^{(\ell)} := \max_{i\in N} \vert \mathcal N_i^{(\ell)} \vert$ & Size of the largest $\ell$-hop neighborhood of the graph instance 
\\
$d_{\max}$ & Maximum node degree of the graph instance 
\\     
\midrule 
\multicolumn{2}{c}{\textbf{Masking Notations}}
\\ \midrule
$T \subseteq N$ & Set of node indices, typically used for masked predictions, where $N \setminus T$ are masked
\\
$\mathbf{x}_i^{(T)} :=
\begin{cases}\mathbf{x}_i &\text{ if } i \in T, \\ \mathbf{b} &\text{ if } i \not \in T,\end{cases}$ 
& Masked node features of nodes $v_i$, where nods in $N \setminus T$ are masked with baseline $\mathbf{b}$
\\
$\mathbf{X}^{(T)}=[\mathbf{x}_1^{(T)},\dots,\mathbf{x}_n^{(T)}]$ with $T \subseteq N$ & Masked node feature matrix of $g$, where nodes in $N \setminus T$ are masked with $\mathbf{b}$ 
\\
$f_{g}(\mathbf{X}^{(T)})$ & Masked prediction of the \gls*{GNN} using masked node features $\mathbf{X}^{(T)}$
\\
$f_i(\mathbf{X}^{(T)})$ with $i \in N$ & Masked node embedding $f_i$ for node $v_i$ evaluated at masked node features $\mathbf{X}^{(T)}$
\\
\midrule 
\multicolumn{2}{c}{\textbf{Game Theory Notations}}
\\ \midrule
$\mathcal P(N)$ & Power set of $N$, i.e. all possible sets of node indices used to define the game
\\
$\mathcal P_k(N)$ & Power set of $N$ up to sets of size $k$, i.e. the sets for which the final \gls*{SI} explanation is computed
\\
$\nu_g(T) := f_{g,\hat y}(\mathbf{X}^{(T)})$ & Graph game evaluated at $T\subseteq N$, i.e. GNN prediction $f_g$ of class $\hat y$ with masked node features $\mathbf{X}^{(T)}$ 
\\
$\nu_i(T) := f_i(\mathbf{X}^{(T)})$ for $i \in N$ & Graph node game evaluated at $T\subseteq N$, i.e. GNN embedding $f_i$ of node $i$ with masked node features $\mathbf{X}^{(T)}$
\\
$S \subseteq N$ & Set of node indices, used to refer to an interaction, i.e. the sets for which attributions are computed
\\
$m_g(S)$ & MI for a set $S \subseteq N$ of the graph game $\nu_g$ 
\\
$m_i(S)$ & MI for a set $S \subseteq N$ of the graph node game $\nu_i$ 
\\
$\mathcal I := \bigcup_{i\in N} \mathcal P(\mathcal N_i^{(\ell)})$ & Set of non-trivial MIs, i.e. $m_g(S) = 0$, if $S \notin \mathcal I$ 
\\
$\Phi_k(S)$ & Final \gls*{SI} explanation of order $k$ evaluated at a set $S$ of size at most $k$
\end{tabular}
}
\end{table}

\clearpage
\section{Proofs}\label{appx_sec_proofs}
\subsection{Proof of Theorem~3.3}
\begin{proof}
    By definition of $\nu_i$, we need to show that $f_i(\Xm^{(T)}) =f_{i}(\Xm^{(T \cap \Nbh_i^{(\ell)})})$ holds. 
    We prove by induction over $\ell$.
    Denote $f_i^{(\ell)}$ the corresponding node embedding function for the $\ell$-layer \gls*{GNN}
    For a \gls*{GNN} with $\ell=1$, it is immediately clear from \cref{align_messagePassing} that
    \begin{align*}
    f^{(1)}_i(\Xm^{(T)})
    = 
    \rho^{(1)} (
    \xm_i^{(T)}, 
    \psi (
    \{\{ \varphi^{(1)}(\xm^{(T)}_j) \mid u_j \in \Nbh_i^{(1)} \}\}
    )
    ) 
    = f_i^{(1)}(\Xm^{(T \cap \Nbh_i^{(1)})}).
\end{align*}

Next, assume $f_{i}^{(\ell-1)}(\Xm^{(T)}) = f_{i}^{(\ell-1)}\left(\Xm^{(T \cap \Nbh_i^{(\ell-1)})}\right)$ holds for any \gls*{GNN}, $i \in N$ and $T \subseteq N$.
Then for a \gls*{GNN} with $\ell$ layers, we have
    \begin{align*}
    f_i^{(\ell)}(\Xm^{(T)})
    &= 
    \rho^{(\ell)} (
    f_{i}^{(\ell-1)}(\Xm^{(T)}), 
    \psi (
    \{\{ \varphi^{(\ell)}(f_{j}^{(\ell-1)}(\Xm^{(T)})) \mid v_j \in \Nbh_i^{(1)} \}\}
    )
    ) 
    \\
    &= \rho^{(\ell)} (
    f_{i}^{(\ell-1)}(\Xm^{(T \cap \Nbh_i^{(\ell-1)})}), 
    \psi (
    \{\{ \varphi^{(\ell)}(f_{j}^{(\ell-1)}(\Xm^{(T \cap \Nbh_j^{(\ell-1)})})) \mid v_j \in \Nbh_i^{(1)} \}\}
    )
    ) 
    \\
    &= \rho^{(\ell)} (
    f_{i}^{(\ell-1)}(\Xm^{(T \cap \Nbh_i^{(\ell)})}), 
    \psi (
    \{\{ \varphi^{(\ell)}(f_{j}^{(\ell)}(\Xm^{(T \cap \Nbh_i^{(\ell)})})) \mid v_j \in \Nbh_i^{(1)} \}\}
    )
    ) 
    \\
    &= f_{i}^{(\ell)}(\Xm^{(T \cap \Nbh_i^{(\ell)})})
\end{align*}
where we have used that $\Nbh_i^{(\ell-1)} \subseteq \Nbh_i^{(\ell)}$ and $\Nbh_j^{(\ell-1)} \subseteq \Nbh_i^{(\ell)}$ for $j \in \Nbh_i^{(1)}$ together with the invariance (induction hypothesis) of $f_{i}^{(\ell-1)}$ and $f_{j}^{(\ell-1)}$.
\end{proof}

\subsection{Proof of Lemma~3.5}
\begin{proof}
    Let $i\in N$ and $S \not \subseteq \Nbh_i^{(\ell)}$.
    Our goal is to show that $m_i(S) = 0$.
    For every subset $T \subseteq N$, we can define disjoint sets
    \begin{align*}
        T^+_i := T \cap \Nbh_i^{(\ell)} \text{ and } T_i^- := T \cap (N\setminus \Nbh_i^{(\ell)}) , \text{ such that } T= T_i^+ \sqcup T_i^-,
    \end{align*}
    with $T_i^+ \subseteq \Nbh_i^{(\ell)}$ and $T_i^- \subseteq N \setminus \Nbh_i^{(\ell)}$.
    For $S = S_i^+ \sqcup S_i^-$ the assumption $S \not \subseteq \Nbh_i^{(\ell)}$ implies that $S_i^-$ is not empty, i.e. $ s_i^- := \vert S_i^- \vert > 0$.
    Furthermore, due to the node game invariance, cf. \cref{theorem_nodegame_invariance}, we have $\nu_i(T) = \nu_i(T^+_i)$ for every $T \subseteq N$.
    In the following, lower-case letters represent the corresponding cardinalities of the subsets, e.g. $t_i^+ := \vert T_i^+ \vert$.
    For the \gls*{MI} of interest $S$, the sum, by definition of \glspl*{MI}, ranges over $\Pow(S)$. Thus, since $S_i^+, S_i^- \in \Pow(S)$, we have for every $T_i^+ \subseteq S_i^+$ and every $T_i^- \subseteq S_i^-$ that the union $T_i^+ \sqcup T_i^- \in \Pow(S)$.
    On the other hand, every $T \in \Pow(S)$ can be uniquely decomposed into $T = T_i^+ \sqcup T_i^-$.
    Hence, instead of summing over all $T \in \Pow(S)$, we may also sum over all subsets $T = T_i^+ \cup T_i^-$ with $T_i^+ \subseteq S_i^+$ and $T_i^- \subseteq S_i^-$.
    The \gls*{MI} for $S$  is then
    \begin{align*}
            m_i(S) 
            &= 
            \sum_{T \subseteq S}(-1)^{s - t} ~ \nu_i(T)
            \\
            &=
            \sum_{T_i^+ \subseteq S_i^+}\sum_{T_i^- \subseteq S_i^-}(-1)^{s - (t^+_i + t^-_i)} ~ \nu_i\left(T_i^+ \sqcup T_i^-\right)
            \\
            &=
            \sum_{T_i^+ \subseteq S_i^+}\sum_{T_i^- \subseteq S_i^-}(-1)^{s - (t^+_i + t^-_i)} ~ \nu_i\left(T_i^+ \right)
            \\
            &=
            \sum_{T_i^+ \subseteq S_i^+}(-1)^{s - t^+_i}  \nu_i\left(T_i^+ \right) \sum_{T_i^- \subseteq S_i^-} (-1)^{t^-_i}
            \\
            &=
            \sum_{T_i^+ \subseteq S_i^+}(-1)^{s - t^+_i}  \nu_i\left(T_i^+ \right) \sum_{t_i^-=0}^{s_i^-} \binom{s_i^-}{t_i^-} (-1)^{t^-_i}
            \\
            &= \sum_{T_i^+ \subseteq S_i^+}(-1)^{s - t^+_i}  \nu_i\left(T_i^+ \right) (1-1)^{s_i^-} = 0,
    \end{align*}
    where we used the node game invariance in the third equation, the binomial theorem $(a+b)^n = \sum_{k=0}^n \binom{n}{k} a^{n-k} b^{k}$ in the second last equation, and $s_i^- > 0$ in the last equation.
\end{proof}

\subsection{Proof of Proposition~3.6}

\begin{proof}
    The proof will be based on two important properties of the \glspl*{MI}.
    \begin{lemma}\label{lemma_constant_mi}
        The \gls*{MI} of a constant game $\nu \equiv c \in \mathbb{R}$ is
        \begin{align*}
            m_c(S) = \begin{cases}
                c, \text{ if } S=\emptyset,
                \\
                0, \text{ otherwise.}
            \end{cases}
        \end{align*}
    \end{lemma}
    \begin{proof}
        The \gls*{MI} for a constant game is computed for $S=\emptyset$ as $m(\emptyset) = \nu(\emptyset) = c$ and for $s := \vert S \vert > 0$ as
        \begin{align*}
            m(S) &= \sum_{T\subseteq S} (-1)^{s-t} \nu(T) = c \sum_{T\subseteq S}(-1)^{s-t} = c \cdot (-1)^s \sum_{t=0}^s \binom{s}{t} (-1)^{t} = c \cdot (-1)^s (1-1)^s = 0,
        \end{align*}
    where we have used the binomial theorem $(a+b)^n = \sum_{k=0}^n \binom{n}{k} a^{n-k} b^k$.
    \end{proof}
    The second property is the linearity of the \gls*{MI} in terms of a linear combination of games.

    \begin{lemma}[Linearity \citep{Fujimoto.2006}] \label{lemma_linearity}
        For a linear combination of two games $\nu := c \cdot \nu_1 + \nu_2$ for a constant $c \in \mathbb{R}$, the \gls*{MI} of $\nu$ is given as
        \begin{align*}
            m_\nu(S) = c \cdot m_{\nu_1}(S) + m_{\nu_2}(S).
        \end{align*}
        for all $S \subseteq N$.
    \end{lemma}

    \begin{proof}
        This result follows from Theorem 4.9 and Lemma 4.1 in \citet{Fujimoto.2006}.
        However, it may also be verified directly
        \begin{align*}
            m_\nu(S) = \sum_{T\subseteq S} (-1)^{s-t} \nu(T) = \sum_{T \subseteq S}(-1)^{s-t} (c \cdot \nu_1(T) + \nu_2(T)) = c \cdot m_{\nu_1}(T) + m_{\nu_2}(T).
        \end{align*}
    \end{proof}

    Next, let $S \not \in \bigcup_{i \in N} \Pow(\Nbh_i^{(\ell)})$.
    Our goal is to show that $m_g(S) = 0$.
    By Lemma~\ref{lemma_MöbiusTransform_NodeLevel}, we have that $m_i(S) = \mathbf{0} \in \mathbb{R}^{d_{\ell}}$ for all $i \in N$ and node games $\nu_i: \Pow(N) \rightarrow \mathbb{R}^{d_\ell}$.
    We can thus define a new game $\nu_\Psi: \Pow(N) \rightarrow \mathbb{R}^{d_\ell}$ as $\nu_\Psi(T) := \Psi(\{\{\nu_i(T) \mid i \in N\}\})$ for $T \subseteq N$.
    Due to Assumption~\ref{assumption_GNN}, $\Psi$ is linear (a linear combination of inputs), and by Lemma~\ref{lemma_constant_mi}, constant shifts do not affect the \glspl*{MI}.
    We therefore assume that $\Psi(\{\{\mathbf{0}\},\dots \{\mathbf{0}\}\}) = \mathbf{0}$ for any number of zero vectors $\mathbf{0}$.
    By the linearity of the \gls*{MI} (Lemma~\ref{lemma_linearity}), we have then
    \begin{align*}
        m_\Psi(S) = \Psi(\{\{m_i(S) \mid i \in N\}\}) = \mathbf{0}.
    \end{align*}
    Lastly, $\nu_g$ represents the logits of the predicted class in the GNN, i.e. $\nu_g(T) = \sigma(\nu_\Psi(T))$.
    The output layer $\sigma$ transforms the games defined as the components of the (multi-dimensional) output of $\nu_{\Psi}$ to an vector of size $d_{\text{out}}$, where $\sigma_{\hat y}$ returns the component used in the graph game $\nu_g$.
    Again, by Lemma~\ref{lemma_constant_mi}, constant shifts do not affect the \glspl*{MI}, and we let $\sigma_{\hat y}(\mathbf{0}) = 0$, to obtain
    \begin{align*}
        m_g(S) = \sigma_{\hat y}(m_\Psi(S)) = \sigma_{\hat y}(\mathbf{0}) = 0,
    \end{align*}
    and likewise for graph regression, which concludes the proof.
\end{proof}

\subsection{Proof of Theorem~3.7}
By Proposition~\ref{proposition_MöbiusTransform_GraphLevel}, the set of non-trivial \glspl*{MI} is given by $\mathcal I$.
To compute exact \glspl*{MI} it is therefore necessary to compute all \glspl*{MI} contained in $\mathcal I$.
Clearly, to compute the \gls*{MI} $m_g(S)$ for any $S \subseteq N$, it is by definition of the \gls*{MI} necessary to evaluate $\nu_g(S)$.
Hence, the complexity of computing exact \glspl*{MI} cannot be lower than $\vert \mathcal I \vert$.
Given now all graph game evaluations $\nu_g(T)$ with $T \in \mathcal I$, we proceed to show that no additional evaluation is required.
In fact, to compute the \gls*{MI} $m_g(S)$ for $S \in \mathcal I$, we require all game evaluations $\nu(T)$ with $T\subseteq S$.
By definition of $\mathcal I$ (Proposition~\ref{proposition_MöbiusTransform_GraphLevel}), there exists a node index $i \in N$, such that $S \in \Pow(\Nbh_i^{(\ell)})$.
Since $S \subseteq \Nbh_i^{(\ell)}$ it follows immediately that all $T\subseteq S$ satisfy $T \in \Pow(\Nbh_i^{(\ell)})$, and hence $T \in \mathcal I$, which we have already computed.
This finishes the proof that exact computation requires $\vert \mathcal I \vert$ graph game evaluations and hence \gls*{GNN} model calls.
Additionally, the number of elements in $\mathcal I$ is trivially bounded by the number of subsets in $N$, i.e. $2^n$ and further 
\begin{align*}
    \vert \mathcal I \vert = \left \vert \bigcup_{i \in N} \Pow(\Nbh_i^{(\ell)}) \right \vert \leq \sum_{i \in N} \vert \Pow(\Nbh_i^{(\ell)}) \vert = \sum_{i\in N} 2^{\vert \Nbh_i^{(\ell)}\vert} \leq n \cdot 2^{n_{\max}^{(\ell)}},
\end{align*}
where each $\ell$-hop neighborhood was bounded by $n_{\max}^{(\ell)} := \max_{i \in N} \vert \Nbh^{(\ell)}_i \vert$.
Lastly, we bound the number of nodes in the $\ell$-hop neighborhood by bounding the number of nodes in each hop with the maximum degree $d_{\max}$ as
\begin{align*}
    n_{\max}^{(\ell)} \leq 1 + d_{\max} + d_{\max}^2 + \dots + d_{\max}^{\ell} = \frac{d_{\max}^{\ell+1}-1}{d_{\max} - 1},
\end{align*}
where we have used the formula for geometric progression.
With this bound we obtain the final bound
\begin{align*}
    \vert \mathcal I \vert \leq n \cdot 2^{n_{\max}^{(\ell)}} \leq n \cdot 2^{\frac{d_{\max}^{\ell+1}-1}{d_{\max}-1}},
\end{align*}
which concludes the proof.

\clearpage
\section{Detailed Comparison of GraphSHAP-IQ and Related Work}\label{appx_sec_rw_and_l_shapley}

The following contains a detailed summary of the related and relevant work.
We further discuss distinctions of GraphSHAP-IQ with specific methods from related work in \cref{appx_sec_ext_rw}.
Moreover, we establish a connection between GraphSHAP-IQ and L-Shapley in \cref{appx_sec_lshapley}

The \gls*{SV} \citep{Shapley.1953} was applied in \gls*{XAI} for local \citep{DBLP:conf/nips/LundbergL17,DBLP:journals/jmlr/CovertLL21} and global \citep{Casalicchio.2019,Covert_Lundberg_Lee_2020} model interpretability, or data valuation \citep{Ghorbani.2019}.
The Myerson value \citep{DBLP:journals/mor/Myerson77} is a variant of the \gls*{SV} for games restricted to graph components.
Extensions of the \gls*{SV} to interactions for \gls*{ML} were proposed  by \cite{Lundberg.2020,Sundararajan.2020,Tsai.2022,Bord.2023}
Due to the exponential complexity, model-specific methods for tree-based models for \glspl*{SV} \citep{Lundberg.2020,Yu.2022}, any-order \glspl*{SI} \citep{DBLP:conf/aaai/ZernBK23,DBLP:conf/aaai/MuschalikFHH24} and \glspl*{MI} \citep{Hiabu.2023} were proposed.
Model-agnostic approximation methods cover the \gls*{SV} \citep{DBLP:conf/nips/LundbergL17,DBLP:conf/iclr/ChenSWJ19,DBLP:conf/aaai/KolpaczkiBMH24}, \glspl*{SI} \citep{Fumagalli.2023,Kolpaczki.2024,fumagalli2024kernelshapiq}, and \glspl*{MI} \citep{DBLP:journals/corr/abs-2402-02631}.
\\
Instance-wise explanations on \glspl*{GNN} were proposed via perturbations \citep{DBLP:conf/nips/YingBYZL19,PGExplainer, DBLP:conf/icml/YuanYWLJ21,schlichtkrull2021interpreting}, gradients \citep{DBLP:conf/cvpr/PopeKRMH19,DBLP:conf/nips/Sanchez-Lengeling20,DBLP:journals/pami/SchnakeELNSMM22,DBLP:conf/icml/XiongSGMMN23} or surrogate models \citep{DBLP:conf/nips/VuT20,DBLP:journals/tkde/HuangYTSC23}.
\gls*{GNN} explanations are given in terms of nodes \citep{DBLP:conf/cvpr/PopeKRMH19,DBLP:journals/tkde/HuangYTSC23} and subgraphs \citep{DBLP:conf/nips/YingBYZL19,PGExplainer, DBLP:conf/icml/YuanYWLJ21}.
Other use paths \citep{DBLP:journals/pami/SchnakeELNSMM22, DBLP:journals/tkde/HuangYTSC23} and edges \citep{schlichtkrull2021interpreting}.
For perturbation-based attributions, maskings for nodes \citep{DBLP:conf/nips/YingBYZL19,DBLP:conf/icml/YuanYWLJ21} or edges \citep{PGExplainer,schlichtkrull2021interpreting} were introduced.
The \gls*{SV} was applied in perturbation-based methods to assess the quality of subgraphs \citep{DBLP:conf/icml/YuanYWLJ21,DBLP:conf/nips/ZhangLSS22,DBLP:conf/nips/YeHWL23}, approximate \glspl*{SV} \citep{DBLP:conf/pkdd/DuvalM21,DBLP:conf/www/AkkasA24} or for pre-defined motifs \citep{DBLP:journals/corr/abs-2202-08815}.
Recently, pairwise \glspl*{SI} were approximated to discover subgraph explanations \citep{DBLP:conf/icml/BuiNNY24}.
\\
In contrast to existing work, we compute exact \glspl*{SV} on node level.
Moreover, we exploit graph and \gls*{GNN} structure to  compute \glspl*{SI} that uncover complex interactions,
formally prove that for linear readouts interaction structures of node embeddings are preserved for graph predictions.

\subsection{Detailed Differences between GraphSHAP-IQ and Related Methods}\label{appx_sec_ext_rw}

In this section, we compare GraphSHAP-IQ with TreeSHAP \citep{Lundberg.2020}, TreeSHAP-IQ \citep{DBLP:conf/aaai/MuschalikFHH24}, GraphSVX \citep{DBLP:conf/pkdd/DuvalM21}, SubgraphX \citep{DBLP:conf/icml/YuanYWLJ21}, GStarX \citep{DBLP:conf/nips/ZhangLSS22} SAME \citep{DBLP:conf/nips/YeHWL23}, and recent work by \citet{DBLP:conf/icml/BuiNNY24}.

\paragraph{TreeSHAP \citep{Lundberg.2020}}
TreeSHAP is a model-specific computation method to explain predictions of decision trees and ensembles of decision trees. 
TreeSHAP defines a cooperative game based on perturbations via the conditional distribution learned by the tree (path-dependent TreeSHAP) or via interventions using a background dataset (interventional TreeSHAP). 
Given these game definitions, TreeSHAP efficiently computes exact \glspl*{SV} in polynomial time by exploiting the tree structure.
Recently, extensions of TreeSHAP for interventional perturbations \citep{DBLP:conf/aaai/ZernBK23}, and with TreeSHAP-IQ for path-dependent perturbations \citep{DBLP:conf/aaai/MuschalikFHH24} were proposed to efficiently compute exact any-order \glspl*{SI} in polynomial time.
Similar to GraphSHAP-IQ, TreeSHAP and TreeSHAP-IQ are model-specific variants for efficient and exact computation of \glspl*{SV} and \glspl*{SI}.
However, TreeSHAP and TreeSHAP-IQ are only applicable to tree-based models, whereas GraphSHAP-IQ is only applicable to \glspl*{GNN} and other model classes with spatially-restricted features.
While both methods are model-specific variants to efficiently compute \glspl*{SI}, these methods are applied on fundamentally different model classes and exhibit completely different computation schemes.

\paragraph{Myerson-Taylor Interactions \citep{DBLP:conf/icml/BuiNNY24}}
This work highlights the opportunities and significance of \glspl*{SI} for \gls*{GNN} interpretability. \citet{DBLP:conf/icml/BuiNNY24} propose an optimal partition of the graph instance as an explanation. 
This partition is found by approximating second-order \gls*{STII} via Monte Carlo on connected components of the graph, without being structure-aware of \glspl*{GNN}.
Note that in our benchmark datasets, there are no isolated components in the graphs, and therefore this does not yield any reduction in complexity.
In contrast, GraphSHAP-IQ is able to compute exact \gls*{STII} (and Myerson-\gls*{STII}) for any order.

\paragraph{GraphSVX \citep{DBLP:conf/pkdd/DuvalM21}.}
GraphSVX proposes to compute Shapley values without interactions on \glspl*{GNN}. This method does not use any \gls*{GNN}-specific structural assumption for graph classification, since GraphSVX considers SVs (not interactions). For graph classification, GraphSVX is an application of KernelSHAP to \glspl*{GNN}. However, for node classification, a result for SVs based on the dummy axiom was established, which also follows from Lemma~\ref{lemma_MöbiusTransform_NodeLevel}. Note that this result for \glspl*{SV} is not applicable for graph prediction, since the dummy axiom does not hold on the graph level, cf. Section 5.5., global/graph classification in \citep{DBLP:conf/pkdd/DuvalM21}. For our main result (\cref{proposition_MöbiusTransform_GraphLevel}), considering the purified interactions (\glspl*{MI}) instead of \glspl*{SV} is crucial.

\paragraph{SubgraphX \citep{DBLP:conf/icml/YuanYWLJ21}.}
SubgraphX identifies isolated subgraphs as explanation. This is done by proposing the \gls*{SV} as a heuristical scoring function for subgraph exploration.
Given a subgraph candidate, a reduced game is defined, where the whole subgraph represents a single player. Based on the computed \gls*{SV}, the quality of the subgraph is determined.
In contrast, GraphSHAP-IQ does not require to group nodes and does not identify isolated components using a scoring function. Instead, GraphSHAP-IQ computes exact \glspl*{SV} on node level, and \glspl*{SI} on all possible subgraphs up to order $k$, which yields an additive decomposition using all these components, which is faithful to the graph game.

\paragraph{GStarX \citep{DBLP:conf/nips/ZhangLSS22}.}
GStarX identifies isolated subgraphs as an explanation. It is related to SubgraphX \citep{DBLP:conf/icml/YuanYWLJ21} and proposes an alternative scoring function to the \glspl*{SV}. In contrast, GraphSHAP-IQ does not identify isolated components using a scoring function.
Instead, GraphSHAP-IQ decomposes the \gls*{GNN} into all possible components (subgraphs) up to order $k$. 
Hence, GraphSHAP-IQ explains the \gls*{GNN} prediction across all maskings faithfully according to the \glspl*{SI}.

\paragraph{SAME \citep{DBLP:conf/nips/YeHWL23}.}
SAME proposes a hierarchical MCTS algorithm as a heuristic way to find explanations as subgraphs. SAME considers $k$-hop neighborhoods for computation of the SV, which however is not accurate for graph prediction, similar to GraphSVX \citep{DBLP:conf/pkdd/DuvalM21}.

\subsection{Linking GraphSHAP-IQ and L-Shapley}\label{appx_sec_lshapley}
In this section, we present a link of GraphSHAP-IQ to L-Shapley \citep{DBLP:conf/iclr/ChenSWJ19} and prove a novel result, which states that L-Shapley with sufficiently large parameter computes exact \glspl*{SV} on games that admit the invariance property (\cref{theorem_nodegame_invariance}), e.g. the GNN-induced node game.
L-Shapley is a model-agnostic method to approximate \glspl*{SV} that utilizes the underlying graph structure of features.
L-Shapley proposes to compute the marginal contributions $\Delta_i(T\setminus i)$ of subsets $T$ that contain $i \in N$ in its $\lambda$-hop neighborhood
\begin{align}\label{align_lshapley}
    \hat\phi_\lambda^{\text{L-SV}}(i) := \frac{1}{\vert\Nbh_i^{(\lambda)}\vert}\sum_{T \subseteq \Nbh_i^{(\lambda)}: i \in T} \binom{\vert \Nbh_i^{(\lambda)} \vert - 1}{t - 1}^{-1} \Delta_i( T \setminus i).
\end{align}

The weights thereby match the weights of the \gls*{SV} restricted to its $k$-hop neighborhood.
While L-Shapley is restricted to the \gls*{SV}, it strongly differs conceptionally from GraphSHAP-IQ in multiple ways.
First, GraphSHAP-IQ allows to compute interactions up to order $\lambda$ in the $\ell$-hop neighborhood, whereas L-Shapley (implicitly) computes interactions up to order $\lambda$ in the $\lambda$-hop neighborhood.
Second, L-Shapley considers only a single neighborhood of the player $i\in N$, whereas GraphSHAP-IQ includes all interactions that contain player $i$ in any $\ell$-hop neighborhoods.
Consequently, GraphSHAP-IQ covers all interactions up to order $\lambda$ of the GNN-induced graph game, whereas L-Shapley only covers interactions up to order $\lambda$ that are fully contained in the $\lambda$-hop neighborhood of the node game of player $i$.
Third, L-Shapley is a model-agnostic method, which is based on a heuristical view of the game structure, whereas GraphSHAP-IQ exploits the \gls*{GNN} structure and additivity of all node games summarized in the graph game.
Consequently, GraphSHAP-IQ computes exact \gls*{SV} with $\lambda=\ell$, whereas L-Shapley computes exact \glspl*{SV} with $\lambda=n$, i.e. only with all $2^n$ model calls.
Lastly, GraphSHAP-IQ maintains efficiency and computes the \glspl*{MI}, which allow to construct any-order \glspl*{SI} yielding an accuracy-interpretability trade-off based on practitioner's needs.

\begin{theorem}\label{theorem_lshapley}
    For a graph $g$ and a single GNN-induced node game $\nu_i$ for $i\in N$ associated with the node embedding $\mathbf{h}^{(\ell)}_{v_i}$, GraphSHAP-IQ applied on the node game with $\lambda=\vert \Nbh_i^{(\ell)}\vert$ and explanation order $k = 1$ returns the L-Shapley computation with $\lambda=\ell$.
    That is, in this specific case the L-Shapley computation is equal to the GraphSHAP-IQ computation and yields the exact \glspl*{SV}.
\end{theorem}

\begin{proof}
    Let $\nu_i$ be a single node game of node $v_i \in V$.
    If $\lambda=\vert \Nbh_i^{(\ell)}\vert$, then by Corollary~\ref{appx_cor_graphshapiq} GraphSHAP-IQ computes exact \glspl*{MI} contained in $\Nbh_i^{(\ell)}$, which by Lemma~\ref{lemma_MöbiusTransform_NodeLevel} are all non-trivial interactions of the node game up to order $\lambda=\ell$.
    Consequently, GraphSHAP-IQ computes exact \glspl*{SV} of the node game.
    On the other hand, we will show that the L-Shapley computation in \cref{align_lshapley} also yields these exact \glspl*{SV} in this case.
    We first observe that due to the node game invariance in \cref{theorem_nodegame_invariance}, we have for 
        \begin{align*}
T^+_i := T \cap \Nbh_i^{(\ell)} \subseteq \Nbh_i^{(\ell)} \text{ and } T_i^- := T \cap (N\setminus \Nbh_i^{(\ell)}) \subseteq N \setminus \Nbh_i^{(\ell)},
    \end{align*}
    with $T= T_i^+ \sqcup T_i^-$ that 
    \begin{align*}
        \Delta_i(T\setminus i) = \nu(T \setminus i) - \nu(T) = \nu(T_i^+ \setminus i) - \nu(T_i^+) = \Delta_i(T_i^+ \setminus i).
    \end{align*}
    Before we compute the \gls*{SV} we need to state the following identity.

    \begin{lemma}[Lemma 1 in \citet{DBLP:conf/iclr/ChenSWJ19}]\label{lemma_comb_identity}
    For any integer $n$ and pair of non-negative integers $s \geq t$, we have
    \begin{align*}
        \sum_{j=0}^n \frac{1}{\binom{n+s}{j+t}}\binom{n}{j} = \frac{s+1+n}{(s+1)\binom{s}{t}}.
    \end{align*}
    \end{lemma}
    \begin{proof}
        The proof is given in the proof in Lemma 1 in Appendix B by \citet{DBLP:conf/iclr/ChenSWJ19}.
    \end{proof}
    
    Hence, we can compute the \gls*{SV} as
    \begin{align*}
        \phi^{\text{SV}}(i) &= \sum_{T\subseteq N \setminus i} \frac{1}{n \cdot \binom{n-1}{t}} \Delta_i(T) 
        \\
        &= \sum_{T\subseteq N: i \in T} \frac{1}{n \cdot \binom{n-1}{t-1}} \Delta_i(T\setminus i) 
        \\
        &= \sum_{T\subseteq N: i \in T} \frac{1}{n \cdot \binom{n-1}{t-1}} \Delta_i(T_i^+\setminus i) 
        \\
        &= \sum_{T_i^+\subseteq \Nbh_i^{(\ell)}: i \in T_i^+}  \Delta_i(T_i^+\setminus i) \sum_{T_i^- \subseteq N \setminus  \Nbh_i^{(\ell)}} \frac{1}{n \cdot \binom{n-1}{t_i^+ + t_i^- -1}}
        \\
        &= \sum_{T_i^+\subseteq \Nbh_i^{(\ell)}: i \in T_i^+}  \Delta_i(T_i^+\setminus i) \sum_{t_i^-=0}^{n-\vert \Nbh_i^{(\ell)} \vert} \binom{n - \vert \Nbh_i^{(\ell)}\vert}{t_i^-} \frac{1}{n \cdot \binom{n-\vert \Nbh_i^{(\ell)}\vert + (\vert \Nbh_i^{(\ell)} \vert -1)}{t_i^- + (t_i^+-1)}}
        \\
        &=  \sum_{T_i^+\subseteq \Nbh_i^{(\ell)}: i \in T_i^+}  \Delta_i(T_i^+\setminus i) \frac{ \vert \Nbh_i^{(\ell)}\vert - 1 + 1 + n- \vert \Nbh_i^{(\ell)}\vert }{n \cdot (\vert \Nbh_i^{(\ell)}\vert -1 +1) \binom{\vert \Nbh_i^{(\ell)}\vert -1}{t_i^+-1)}}
        \\
        &=  \sum_{T_i^+\subseteq \Nbh_i^{(\ell)}: i \in T_i^+}  \Delta_i(T_i^+\setminus i) \frac{1}{\vert \Nbh_i^{(\ell)}\vert \cdot \binom{\vert \Nbh_i^{(\ell)}\vert-1}{t_i^+-1)}},
    \end{align*}
    where we have used Lemma~\ref{lemma_comb_identity} in the second last equation.
    This result of the \glspl*{SV} clearly coincides with \cref{align_lshapley}, which concludes the proof.
\end{proof}

\Cref{theorem_lshapley} shows that L-Shapley computes exact \glspl*{SV}, if the game corresponds to a node game in a GNN.
However, L-Shapley does not compute exact \glspl*{SV} for the graph game, as interactions involved in the computation of the \gls*{SV} may also appear in other neighborhoods.
In fact, L-Shapley applied on the graph game only converges for $k=n$, which corresponds to $2^n$ model evaluations.
In our experiments, we also show empirically that L-Shapley performs poorly on GNN-induced graph games.

\clearpage
\section{Additional Algorithmic Details for GraphSHAP-IQ}\label{appx_sec_graphshapiq_alg}

In this section, we provide further details of the GraphSHAP-IQ algorithm, including the full algorithm that is capapable of approximation.

\subsection{Exact Computation}

\cref{alg_exact_graphshapiq} outlines the exact computation for GraphSHAP-IQ and we provide further details and pseudocode in this section.
\cref{appx_alg_mi} describes the computation of the \glspl*{MI} from line \ref{alg_exact_line_moebius} in \cref{alg_exact_graphshapiq}.
Given the \glspl*{MI}, GraphSHAP-IQ outputs the converted \glspl*{SI}, according to the conversion formulas discussed in \cref{appx_sec_conversion}.
\cref{appx_alg_mi2si} describes the pseudocode of the conversion method called in line \ref{alg_exact_line_conversion} in \cref{alg_exact_graphshapiq}.
Here, the method $\textsc{\texttt{getConversionWeight}}$ outputs the distribution weight for each specific index given an \gls*{MI} $\tilde S$, a \gls*{SI} $S$ and the index, according to the conversion formulas discussed in \cref{appx_sec_conversion}.

\begin{algorithm}
    \caption{Möbius Interaction (MI)}
    \label{appx_alg_mi}
    \begin{algorithmic}[1]
    \REQUIRE Game values $\nu_g$ and \gls*{MI} of interest $S \subseteq N$.
    \STATE $m_g[S] \gets 0$
    \FOR{$T \in \Pow(S)$}
        \STATE $m_g[S] \gets m_g[S] + (-1)^{\vert S \vert - \vert T \vert} \nu_g[T]$ 
    \ENDFOR
    \STATE \textbf{return} \gls*{MI} $m_g$.
    \end{algorithmic}
\end{algorithm}

\begin{algorithm}
    \caption{\gls*{MI}to\gls*{SI}}
    \label{appx_alg_mi2si}
    \begin{algorithmic}[1]
    \REQUIRE (Approximated) \glspl*{MI} $\hat m_g$, \gls*{SI} order $k$ and \gls*{SI} index
    \FOR{$\tilde S \in \textsc{\texttt{Index}}(\hat m_g)$}
        \FOR{$S \in \Pow(\tilde S)$ with $\vert S \vert \leq k$}
        \STATE $\Phi_k[S] \gets 0$, if not initialized yet.
        \STATE $w_S \gets \textsc{\texttt{getConversionWeight}}(\tilde S, S, \text{index})$
        \STATE $\Phi_k[S] \gets \Phi_k[S] + w_S \cdot \hat m_g[\tilde S]$
        \ENDFOR
    \ENDFOR
    \STATE \textbf{return} \glspl*{SI} $\Phi_k$.
    \end{algorithmic}
\end{algorithm}

\subsection{Approximation with GraphSHAP-IQ}

For graphs with high connectivity and GNNs with many convolutional layers, computing exact \glspl*{SI} via \cref{theorem_complexity} may still be infeasible.
We thus propose GraphSHAP-IQ, a flexible approach to compute either exact \glspl*{SI} or an approximation.
Outlined in \Cref{alg_graphshapiq}, GraphSHAP-IQ depends on a single parameter $\lambda=1,\dots,n$ that controls the size of computed \glspl*{MI}. 
Given $\lambda$, GraphSHAP-IQ deploys a deterministic approximation method that computes exhaustively all non-trivial \glspl*{MI} of the GNN-induced graph game $\nu_g$ up to order $\lambda$.
In \cref{appx_sec_lshapley}, we discuss a link between \glspl*{SV} and L-Shapley \citep{DBLP:conf/iclr/ChenSWJ19}.
\\
\textbf{Exact \glspl*{SI}:} If $\lambda \geq n_{\max}^{(\ell)}$, i.e. $\lambda$ is larger or equal to the largest neighborhood size, then the set of \emph{remaining neighborhoods} ${\mathcal R}_\lambda$ (line \ref{alg_line_remaining_set}) is empty. 
GraphSHAP-IQ then computes all non-trivial \glspl*{MI} and returns exact \glspl*{SI}.
In fact, according to Proposition~\ref{proposition_MöbiusTransform_GraphLevel}, ${\mathcal I}_\lambda$ (line \ref{alg_line_interaction_set}) contains all non-trivial \glspl*{MI} and GraphSHAP-IQ evaluates the graph game $\nu_g$ on all of these subsets (line \ref{alg_line_model_calls}).
The \glspl*{MI} $\hat m_g(S)$ are then computed for all $S \in {\mathcal I}_\lambda = {\mathcal I}$ (line \ref{alg_line_moebius}), and converted to \glspl*{SI} (line \ref{alg_line_conversion}), cf. \cref{appx_sec_conversion}.
\\
\textbf{Approximation:} If $\lambda<n_{\max}^{(\ell)}$, then ${\mathcal I}_\lambda$ (line \ref{alg_line_interaction_set}) does not contain all non-trivial \glspl*{MI} and the set of remaining neighborhoods ${\mathcal R}_\lambda$ (line \ref{alg_line_remaining_set}) is not empty. 
Consequently, the recovery property of the \glspl*{MI}, cf. \cref{align_möbius_recovery}, does not hold for $T \in {\mathcal R}_\lambda$, and in particular not for the model prediction $\nu_g(N) = f(\Xm)$.
Thus, for neighborhoods $T \in {\mathcal R}_\lambda$ GraphSHAP-IQ computes the current recovery value using the \glspl*{MI} (line \ref{alg_line_nuhat}) and assigns the gap to $\hat m_g(T)$ (line \ref{alg_line_recovery}).
Note that a previous sorting (line \ref{alg_line_sort}) is required to ensure that previously assigned neighborhood \glspl*{MI} are included.
Ultimately, the largest neighborhood $S^{*} \in \mathcal R_\lambda$ (line \ref{alg_line_maxMI}) receives the gap $\tau$ (line \ref{alg_line_tau}) of the prediction $\nu_g(N)=f(\Xm)$ and the sum of all \glspl*{MI} to maintain efficiency (line \ref{alg_line_efficiency}).
The rationale of maintaining the recovery property is that while GraphSHAP-IQ introduces a bias on the approximated \glspl*{MI}, it equally distributes missing interaction mass onto lower-order \glspl*{SI}.
GraphSHAP-IQ further satisfies the following corollary.

\begin{corollary}\label{appx_cor_graphshapiq}
    GraphSHAP-IQ computes exact \glspl*{SI}, if $\lambda \geq n_{\max}^{(\ell)}-1$.
\end{corollary}

\begin{proof}
    By Proposition~\ref{proposition_MöbiusTransform_GraphLevel} it is clear that GraphSHAP-IQ computes exact \glspl*{MI} and consequently exact \glspl*{SI} for $\lambda \geq \max_{i \in N}(\vert \Nbh_i^{(\ell)} \vert)$, since GraphSHAP-IQ computes all \glspl*{MI} in $\mathcal I$ up to order $\lambda$.
For $\lambda \geq \max_{i \in N}(\vert \Nbh_i^{(\ell)} \vert)-1$, GraphSHAP-IQ computes all \glspl*{MI} up to order $\max_{i \in N}(\vert \Nbh_i^{(\ell)} \vert)-1$
However, for all neighborhoods of size $\max_{i \in N}(\vert \Nbh_i^{(\ell)} \vert)$, there is only a single \gls*{MI} remaining, which is then collected in $\mathcal R_\lambda$.
W.l.o.g let $R$ be the set to the corresponding missing \gls*{MI}.
Due to the recovery property of the \glspl*{MI}, this last missing interaction is exactly the gap computed in line 7, since
\begin{align*}
    m_g(R) = \sum_{T \subseteq R}(-1)^{r-t} \nu(T) = \nu(R) + \sum_{T \subset R}(-1)^{r-t} \nu(T).
\end{align*}
Since this holds for all missing interactions in $\mathcal R_\lambda$ and further $\mathcal I= \mathcal I_\lambda \cup \mathcal R_\lambda$, it concludes the proof as also the efficiency gap (line~\ref{alg_line_efficiency}) is zero due to
\begin{align*}
    \sum_{T \in \mathcal I_\lambda \cup \mathcal R_\lambda} m_g(T) = \sum_{T \in \mathcal I} m_g(T) = \sum_{T \subseteq N} m_g(T) = \nu_g(N) = f_g(\Xm).
\end{align*}

\end{proof}

\begin{algorithm}[H]
    \caption{GraphSHAP-IQ}
    \label{alg_graphshapiq}
    \begin{algorithmic}[1]
    \REQUIRE Graph $g=(V,E,\Xm)$, $\ell$-Layer GNN $f_g$, \gls*{MI} order $\lambda$, \gls*{SI} order $k$.
    \STATE ${\mathcal I}_\lambda \gets [T]_{T \in {\mathcal I},\vert T \vert \leq \lambda}$ \label{alg_line_interaction_set}
    \STATE ${\mathcal R}_\lambda \gets [\Nbh_i^{(\ell)}]_{\Nbh_i^{(\ell)} \notin \mathcal I_\lambda}$  \label{alg_line_remaining_set}
    \STATE $\nu_g \gets [f_g(\Xm^{(T)})]_{T \in {\mathcal I}_\lambda \cup {\mathcal R}_\lambda}$\label{alg_line_model_calls}
    \STATE $\hat m_g \gets [{\textsc{\texttt{MI}}}(\nu_g, S)]_{S \in {\mathcal I}_\lambda}$ \label{alg_line_moebius}
    \STATE ${\mathcal R}_\lambda \gets {\textsc{\texttt{sort}}}({\mathcal R}_\lambda)$ \label{alg_line_sort}
    \FOR{$S \in {\mathcal R}_\lambda$}
        \STATE $\hat \nu_g[S] \gets \sum_{T\subset S}\hat m_g[T]$ \label{alg_line_nuhat}
        \STATE $\hat m_g[S] \gets \nu_g[S] -\hat\nu_g[S]$\label{alg_line_recovery}  
    \ENDFOR
    \STATE $S^* \gets {\textsc{\texttt{selectMax}}}(R_\lambda)$ \label{alg_line_maxMI}
    \STATE $\tau \gets f_g(\Xm) -\sum_T\hat m_g[T]$  \label{alg_line_tau}
    \STATE $\hat m_g[S^{*}] \gets 
    \hat m_g[S^{*}] - \tau$  \label{alg_line_efficiency}
    \STATE $\hat\Phi_k \gets {\textsc{\texttt{MItoSI}}}(\hat m_g,k)$ \label{alg_line_conversion}
    \STATE \textbf{return} \glspl*{MI} $\hat m_g$, \glspl*{SI} $\hat\Phi_k$
    \end{algorithmic}
\end{algorithm}  

\subsection{Interaction-Informed Baseline Methods}\label{appx_sec_interaction_informed}
For a given graph instance and a \gls*{GNN}, \cref{theorem_complexity} states that \glspl*{MI} are sparse, i.e. all \glspl*{MI} of subsets that are not in $\mathcal I$ are zero.
Unfortunately, model-agnostic baseline methods, such as Permutation Sampling \citep{Tsai.2022}, SHAP-IQ \citep{Fumagalli.2023}, SVARM-IQ \citep{Kolpaczki.2024}, KernelSHAP-IQ \citep{fumagalli2024kernelshapiq} or Inconsistent KernelSHAP-IQ \citep{Pelegrina.2023,fumagalli2024kernelshapiq}, use Monte Carlo sampling on the game evaluations, i.e. the masked predictions from the graph game $\nu_g$.
Therefore, results on \glspl*{MI} cannot be directly applied to improve the baseline methods.
However, all variants of \glspl*{SI} can be represented by \glspl*{MI}, and it has been shown that the \gls*{SI} of a subset $S \subseteq N$ is given as the weighted sum of \glspl*{MI} exclusively of subsets $T\subseteq N$ that contain $S$, i.e. $S \subseteq T$ \citep{Bord.2023}.
Due to the structure of non-trivial \glspl*{MI} $\mathcal I$ from \cref{proposition_MöbiusTransform_GraphLevel}, the following holds: For two sets $S,T \subseteq N$ with $S \subseteq T$, if $T \in \mathcal I$, then also $S \in \mathcal I$.
Conversely, if $S \notin \mathcal I$, then also $T \notin S$.
In other words, for the \gls*{SI} of a subset $S \subseteq N$ with $S\notin \mathcal I$, its \gls*{MI} is zero, as well as all \glspl*{MI} of subsets $T$ that contain $S$.
Hence, the \gls*{SI} for a subset $S \notin \mathcal I$ is zero as well, since it can be computed as a weighted sum exclusively of \glspl*{MI} of subsets that contain $S$.
In conclusion, \glspl*{SI} of subsets that are not contained in $\mathcal I$ can be discarded from approximation and set to zero.
For baseline methods that estimate each \gls*{SI} separately using Monte Carlo, such as Permutation Sampling, SHAP-IQ and SVARM-IQ, we exclude these \glspl*{SI} from approximation and set their values to zero.
For regression-based methods, such as KernelSHAP-IQ and Inconsistent KernelSHAP-IQ, these interactions are excluded from the regression problem, and manually set to zero.

Notably, since all individual nodes are contained in $\mathcal I$, i.e. all nodes affect the graph prediction, the approximation of the \gls*{SV} is unaffected by this adjustment.
Consequently, we do not adjust L-Shapley \citep{DBLP:conf/iclr/ChenSWJ19}, as it is only applicable for \glspl*{SV}.

\clearpage
\section{Further Background on Shapley Interactions}

In this section we provide further background on \glspl*{SI}.
We first provide more related work in \cref{appx_sec_si_rw}.
In \cref{appx_sec_otherinteractions}, we introduce \glspl*{SI} that satisfy the efficiency axiom.
In \cref{appx_sec_conversion}, we discuss conversion formulas of \glspl*{MI} to various \glspl*{SI}.

\subsection{Related Work on Shapley Interactions}\label{appx_sec_si_rw}

The \gls*{SV} \citep{Shapley.1953} is a concept from game theory \citep{Grabisch.2016} and applied in \gls*{XAI} for local \citep{DBLP:conf/nips/LundbergL17,DBLP:journals/jmlr/CovertLL21,Chen2023Overview_ExplainabilityWithShapley} and global \citep{Casalicchio.2019,Covert_Lundberg_Lee_2020} model interpretability, or data valuation \citep{Ghorbani.2019}.
The Myerson value \citep{DBLP:journals/mor/Myerson77} is a variant of the \gls*{SV} for games restricted on components of graphs.
Extensions of the \gls*{SV} to interactions were proposed in game theory \citep{Grabisch.1999,marichal1999chaining} and \gls*{ML} \citep{Lundberg.2020,Sundararajan.2020,Tsai.2022,Bord.2023}, where \gls*{ML}-based extensions preserve efficiency, e.g. the sum of interactions equals the model's prediction.
The \glspl*{MI} \citep{rota1964foundations,harsanyi1963simplified,Fujimoto.2006} define a fundamental concept in game theory, and the edge case of \gls*{ML}-based \glspl*{SI}.
Due to the exponential complexity of \glspl*{SI}, efficient methods for tree-based models for \glspl*{SV} \citep{Lundberg.2020,Yu.2022}, any-order \glspl*{SI} \citep{DBLP:conf/aaai/ZernBK23,DBLP:conf/aaai/MuschalikFHH24} and \glspl*{MI} \citep{Hiabu.2023} were proposed.
In model-agnostic settings, approximation methods have been proposed \citep{Tsai.2022,Fumagalli.2023,Kolpaczki.2024,fumagalli2024kernelshapiq} as extensions of \gls*{SV}-based methods \citep{Castro.2009,Covert.2021,Kolpaczki.2024,DBLP:conf/nips/LundbergL17,Pelegrina.2023}, and by exploiting graph-structured inputs \citep{DBLP:conf/iclr/ChenSWJ19}.
Recently, \glspl*{MI} were computed in synthetic settings \citep{DBLP:journals/corr/abs-2402-02631}.

\subsection{Efficient Shapley Interaction Indices}\label{appx_sec_otherinteractions}
In the following, we describe extensions of the \gls*{SV} to \glspl*{SI}, which yield efficient interaction indices.
That is, given an explanation order $k$ the \glspl*{SI} $\Phi_k$ are defined on all interactions up to order $k$ and yield an additive decomposition of the model's prediction, i.e
\begin{equation*}
    \nu(N) = \sum_{S \subseteq N, \vert S \vert \leq k} \Phi_k(S).
\end{equation*}

\paragraph{$k$-Shapley Values ($k$-SII)}\label{appx_sec_ksii}

The \glspl*{k-SII} values of order $1$ are defined as the \glspl*{SV} $\phi^{\text{SV}}$.
Furthermore, the pairwise \glspl*{k-SII} are given as
    \begin{align*}     
        &\Phi_2(ij) :=  \phi^{\text{SII}}(ij) 
        &&\text{ and } 
        &&\Phi_2(i) := 
            \phi^{\text{SV}}(i) - \frac 1 2 \sum_{j \in N \setminus i} \phi^{\text{SII}}(ij). 
            %\text{ and satisfy } \nu(\emptyset) + \sum_{S\subseteq N}^{1 \leq \vert S \vert \leq 2} \Phi_2(S) = \nu(N).
    \end{align*}
The general case is described by \citet{Bord.2023} and involves the Bernoulli numbers.
For any explanation order $k=1,\dots,n$, the \glspl*{k-SII} are recursively defined as
 \begin{align*} 
 \Phi_{k}(S) :=
        \begin{cases}
            \phi^{\text{SII}}(S) &\text{if } \vert S \vert = k
            \\
            \Phi_{k-1}(S) + B_{k-\vert S \vert} \sum_{\tilde S \subseteq N \setminus S}^{\vert S \vert + \tilde S = k} \phi^{\text{SII}}(S \cup \tilde S) &\text{if } \vert S \vert < k
        \end{cases}
    \end{align*}
    with $1\leq \vert S \vert \leq k \leq n$, $\Phi_0 \equiv 0$ and Bernoulli numbers $B_n$ \citep{Bord.2023}.
    For $k=1$, \gls*{k-SII} is the \gls*{SV}, and for $k=n$ the \gls*{MI} \citep{Bord.2023}.
    An explicit definition of \gls*{k-SII} is given in Appendix A.1 in \citet{Bord.2023} as
    \begin{align*}
        \Phi_k(S) = \sum_{r=0}^{k-\vert S \vert} \sum_{R \subseteq N, \vert R \vert = r} B_k \phi^{\text{SII}}(S \cup R).
    \end{align*}

Besides \gls*{SII} and \gls*{k-SII}, there exist other interaction indices that extend the \gls*{SV} to higher orders and yield the \glspl*{MI} for $k=n$.
These essentially differ by their distribution of \gls*{MI} to lower-order interactions.
Similar to \gls*{k-SII}, these interaction indices can be described by the \glspl*{MI}.
Hence, they can be computed once the \glspl*{MI} are known.

\paragraph{Shapley Taylor Interaction Index (STII)}
The \gls*{STII} \citep{Sundararajan.2020} is an interaction index that gives stronger emphasis to the highest order of computed interactions, which capture all higher-order \glspl*{MI}.
It is defined as 
\begin{align*}
&\Phi^{\text{STII}}_k(S) := a(S) \text{ for } \vert S \vert < k &&\text{ and }
    &&\Phi^{\text{STII}}_k(S) :=
        \frac{k}{n} \sum_{T \subseteq N \setminus S} \frac{1}{\binom{n-1}{\vert T \vert}} \Delta_S(T) \text{ for } \vert S \vert = k.
\end{align*}

The \gls*{STII} for lower order, i.e. $\vert S \vert < k$, is simply the \gls*{MI}.
For the maximum order, i.e. $\vert S \vert = k$, the \gls*{STII} is computed as an average over discrete derivatives.
Therefore, the \gls*{STII} distributes the higher-order \gls*{MI} with $\vert S \vert > k$ solely on the top-order interactions of the \gls*{STII}.
It was shown \citep{Sundararajan.2020} that the \gls*{STII} is the only interaction index preserving the Shapley axioms that additionally satisfies the \emph{interaction distribution axiom} besides the classical axioms.

\paragraph{Faithful Shapley Interaction Index (FSII)}
The \gls*{FSII} \citep{Tsai.2022} is an interaction index that yields stronger emphasis on the faithfulness of the interaction index.
It is defined as the best approximation of interactions up to order $k$ in terms of a weighted least square objective
\begin{align*}
    \Phi_k^{\text{FSII}} := \argmin_{\Phi_k \subseteq \mathbb{R}^{n_k}} \sum_{T \subseteq N,~0<\vert T \vert < n} \mu(t) \left( \nu(T) - \sum_{S \subseteq T, \vert T \vert \leq k} \Phi_k(S) \right)^2,
    \\
    \text{such that } \nu(N) = \sum_{S \subseteq N,~\vert S \vert \leq k} \Phi_k(S) \text{ and } \nu(\emptyset) = \Phi_k(\emptyset),
\end{align*}
where $\mu(t) \propto \frac{n-1}{\binom{n}{t}\cdot t \cdot (n-t)}$ and $n_k := \vert \{ S\subseteq N \mid \vert S \vert \leq k\}\vert$ is the number of interactions up to size $k$.
It was shown that the \gls*{FSII} is the only interaction index preserving the Shapley axioms that additionally satisfies the \emph{faithfulness property}, i.e. is represented as the solution to a single weighted least square problem \citep{Tsai.2022}.

\subsection{Conversion Formulas for Möbius Interactions to Shapley Interactions}\label{appx_sec_conversion}
The \glspl*{MI} are a basis of the vector space of cooperative games and thus all \glspl*{SI} can be expressed in terms of the \glspl*{MI}.
The \gls*{SV} as well as the \gls*{SII} can be directly computed by the conversion formulas
\begin{align*}
     &\phi^{\text{SV}}(i) =\sum_{\tilde S \subseteq N: i \in \tilde S} \frac{m(\tilde S)}{\vert \tilde S \vert } &&\text{ and } &&\phi^{\text{SII}}(S) =\sum_{ \tilde S \subseteq N: \tilde S \supseteq S} \frac{m(\tilde S)}{\vert \tilde S \vert - \vert  S \vert +1}.
\end{align*}

The conversion formula for the \gls*{MI} to \gls*{k-SII} is given in Theorem 6 in \citet{Bord.2023}.
The conversion for \gls*{STII} is given in the proof of Theorem 8 in Appendix H \citep{Bord.2023}.
The conversion formula from to the \gls*{FSII} is given in Theorem 19 in \citet{Tsai.2022}.

\clearpage
\section{Experimental Setup and Reproducibility}\label{appx_exp}
All the experiments have been conducted using the PyTorch Geometric library \citep{PyG}, running on a computing machine equipped with Intel(R) Xeon(R) CPU, one Nvidia RTX A5000 and 60GB of RAM. Overall, the total compute time of this project, including preliminary experiments, training of \glspl*{GNN}, computation of baselines and \glspl*{SI}, required no more than 100 hours, which could be reduced by simple parallelization. 

\subsection{Datasets}
All considered methods were empirically validated on eight common real-world chemical datasets for graph classification and one real-world water distribution network for graph regression. We avoid synthetic datasets commonly used in the graph \gls*{XAI} community due to their limitations \citep{GraphFramEx}. 
The licenses for the datasets are summarized in \cref{tab_dataset_licenses}.

\begin{table}[htbp]
\centering
\caption{Dataset License Overview}
\label{tab_dataset_licenses}
\resizebox{\columnwidth}{!}{%
\begin{tabular}{@{}cccc@{}}
\toprule
\textbf{Dataset} & \textbf{Source} & \textbf{License} \\ \midrule
\gls*{BNZ} & \citep{DBLP:conf/nips/Sanchez-Lengeling20} & \href{https://creativecommons.org/publicdomain/zero/1.0/}{CC0 1.0 Universal} \\
\gls*{FLC} & \citep{DBLP:conf/nips/Sanchez-Lengeling20} & \href{https://creativecommons.org/publicdomain/zero/1.0/}{CC0 1.0 Universal} \\
\gls*{MTG} & \citep{Kazius_McGuire_Bursi_2005} & \href{https://creativecommons.org/publicdomain/zero/1.0/}{CC0 1.0 Universal} \\
\gls*{ALC} & \citep{DBLP:conf/nips/Sanchez-Lengeling20} & \href{https://creativecommons.org/publicdomain/zero/1.0/}{CC0 1.0 Universal} \\
\gls*{PRT} & \citep{DBLP:conf/ismb/BorgwardtOSVSK05,Morris+2020} & ``Free to use'', cf. Section 2 of \citet{Morris+2020} \\
\gls*{ENZ} & \citep{DBLP:conf/ismb/BorgwardtOSVSK05,Morris+2020}  & ``Free to use'', cf. Section 2 of \citet{Morris+2020} \\
\gls*{CX2} & \citep{DBLP:journals/jcisd/SutherlandOW03,Morris+2020} & ``Free to use'', cf. Section 2 of \citet{Morris+2020} \\
\gls*{BZR} & \citep{DBLP:journals/jcisd/SutherlandOW03,Morris+2020} & ``Free to use'', cf. Section 2 of \citet{Morris+2020} \\
\gls*{WAQ} & \citep{Vrachamis2018WaterDistributionNetworks_Hanoi} & \href{https://eupl.eu/1.2/en}{EUROPEAN UNION PUBLIC LICENCE v. 1.2 EUPL} \\ \bottomrule
\end{tabular}
}
\end{table}

\subsubsection{Chemical Datasets}
This data reflects biological and chemical problems,
where particular substructures can predict the properties of molecules. The \textbf{Benzene (BNZ)} \citep{DBLP:conf/nips/Sanchez-Lengeling20} dataset consists of 12 000 moleculars, with each graph labeled as containing or not a benzene ring. In this context, the underlying explanations for the predictions are the specific atoms (nodes) that make up the benzene rings. The \textbf{Fluoride Carbonyl (FLC)} \citep{DBLP:conf/nips/Sanchez-Lengeling20} dataset comprises 8 671 molecular graphs, each labeled positive or negative. A positive sample denotes a molecule with a fluoride ion (F-) and a carbonyl functional group (C = O). The underlying explanations for the labels are the specific combinations of fluoride atoms and carbonyl functional groups present within a given molecule. Similarly, the \textbf{Alkane Carbonyl (ALC)} \citep{DBLP:conf/nips/Sanchez-Lengeling20} dataset consists of 1 125 molecular graphs, labeled positive if it features an unbranched alkane chain and a functional carbonyl group (C = O), which serves as an underlying explanation. The \textbf{Mutagenicity (MTG)} \citep{Kazius_McGuire_Bursi_2005} dataset comprises 1 768 graph molecules, categorized into two classes based on their mutagenic properties, specifically their effect on the Gram-negative bacterium S. Typhimurium. Following \citet{agarwal2023evaluating}, the original dataset of 4 337 graphs was reduced to 1 768, retaining only those molecules whose labels directly correlate with the presence or absence of specific toxicophores (motifs in the graphs), including NH2, NO2, aliphatic halide, nitroso, and azo-type groups, as defined by \citep{Kazius_McGuire_Bursi_2005}. The latter four datasets have been retrieved from the implementation of \citep{agarwal2023evaluating}. \textbf{PROTEINS (PRT)} \citep{DBLP:conf/ismb/BorgwardtOSVSK05} and \textbf{ENZYMES (ENZ)} \citep{DBLP:conf/ismb/BorgwardtOSVSK05} involve
graphs whose nodes represent secondary structure elements and edges indicate neighborhood in the amino-acid sequence, entailing 1 113 and 600 graphs respectively. \textbf{BZR (BZR)} and \textbf{COX2 (CX2)} \citep{DBLP:journals/jcisd/SutherlandOW03} are correlation networks of biological activities of compounds with their structural attributes, of 405 and 467 graphs.
These latter datasets are fetched from the benchmark collection of \citep{Morris+2020}. They all have binary labels except for ENZ, which has 6 classes. Finally, we did not exploit the additional continuous node features available for ENZ.
\subsubsection{Water Distribution Network}
We apply GraphSHAP-IQ to a dynamical system governed by local partial differential equations (PDEs). Specifically, the spread of chlorine in water distribution systems (WDS), where a GNN predicts the fraction of nodes chlorinated after some time, framing this task as a graph-level regression.
We chose this task because chemicals in WDS spread predominantly with water flow (advection). Therefore, the task should be mainly explained by the flow pattern of the network, which yields an intuitive expectation of the explanations. 

We generate a temporal dataset \textbf{WaterQuality (WAQ)} using the EPyT-Flow toolbox \citep{github:epytflow}. The dataset consists of 1000 temporal graphs with 30 time steps each. Spatially, the graphs represent the Hanoi WDS \citep{Vrachamis2018WaterDistributionNetworks_Hanoi}, a popular 32-node benchmark WDS. 

\subsection{GNN Architectures} \label{appx_sec_models}
Our approach universally applies to any message passing \gls*{GNN} and can be employed on graph datasets without requiring ground truth explanations. Since different GNN architectures exhibit unique learning patterns and performance characteristics, their explanations will also vary. To demonstrate the effectiveness of our approach, we followed the remarks and setting of~\citet{GraphFramEx}, by comparing the explanations of three popular GNN models: \gls*{GCN} \citep{Kipf2017GNNs_GCN}, \gls*{GIN} \citep{Xu2017GNNs_GIN}, and \gls*{GAT} \citep{Velickovic2017GNNs_GAT}. 
The overall model architecture is as follows: the graph input is unprocessed, we stack the GNN layers for $l\in\{1,2,3\}$, concatenate all the embeddings \citep{You2020DesignSF}, apply global sum-pooling \citep{Xu2017GNNs_GIN}, and finally, a dense linear layer returns the output predictions.
Furthermore, each GNN layer is followed by a sequence of modules: \texttt{batch\_normalization}, \texttt{dropout}, and \texttt{LeakyReLU} activation function \citep{You2020DesignSF}. We validated the models for different amounts of hidden units $h\in\{8,16,32,64,128\}$, and we defined the multilayer perceptron of \gls*{GIN} with 32 hidden units, while all the remaining design choices for the GNN layers follow their default settings (e.g., aggregation function, number of attention heads, ...).
For a selection of datasets, we also compared performances for deeper \glspl*{GNN}, cf. \cref{tab_performance_deep_gnns}.
The performances of the deeper models compared to more shallow \glspl*{GNN} are mixed. 
All architectures lead to quite similar performances.
Both shallow (less than 4 \gls*{GNN} layers) and deeper models (more than 3 \gls*{GNN} layers) lead to best performing architectures.

Regarding WAQ, since we now have important edge features (water flows), we cannot use vanilla GCN, GIN, or GAT as before. Instead, we designed a simple GNN that can process edge features, using 3 message-passing layers with mean-aggregation, \texttt{ReLU} activations, global average-pooling as the readout layer, and finally, a linear layer provides the regression outputs.

\begin{table}%[b]
\centering
\caption{Model accuracy for shallow and deep \glspl*{GNN} (1-6 layers).}
\label{tab_performance_deep_gnns}
\resizebox{0.99\columnwidth}{!}{\begin{tabular}{lcccc|cccccc|cccccc|cccccc}
\toprule
\multicolumn{5}{c|}{\textit{Dataset Description}} & 
\multicolumn{12}{c}{\textit{Model Accuracy by Layer (\%)}} \\
\multirow{2}{*}{\textbf{Dataset}} & \multirow{2}{*}{\textbf{Graphs}} & \multirow{2}{*}{\textbf{$d_{\text{out}}$}} & \multirow{2}{*}{\textbf{\begin{tabular}[c]{@{}c@{}}Nodes\\ (avg)\end{tabular}}} & \multirow{2}{*}{\textbf{\begin{tabular}[c]{@{}c@{}}Density\\ (avg)\end{tabular}}} & \multicolumn{6}{c}{\textbf{GCN}} & \multicolumn{6}{c}{\textbf{GAT}} & \multicolumn{6}{c}{\textbf{GIN}} \\
 &  &  &  &  & \textbf{1} & \textbf{2} & \textbf{3} & \textbf{4} & \textbf{5} & \textbf{6} & \textbf{1} & \textbf{2} & \textbf{3} & \textbf{4} & \textbf{5} & \textbf{6} & \textbf{1} & \textbf{2} & \textbf{3}  & \textbf{4} & \textbf{5} & \textbf{6} \\ \midrule
\begin{tabular}[c]{@{}l@{}}\gls{BNZ} \\ 
\citep{DBLP:conf/nips/Sanchez-Lengeling20}\end{tabular} 
    & 12000 & 2 & 20.6 & 22.8          % Description
    & 84.2 & 88.6 & 90.4 & \textbf{90.6} & \textbf{90.6} & \textbf{90.6}           % GCN
    & 83.1 & 85.1 & 85.7 & 84.5 & \textbf{87.8} & 86.3          % GAT
    & 84.9 & 90.5 & 90.8 & 92.1 & \textbf{{92.3}} & \textbf{{92.3}} \\          % GIN
\rowcolor{tablegray}\begin{tabular}[c]{@{}l@{}}\gls{PRT} \\ 
\citep{DBLP:conf/ismb/BorgwardtOSVSK05}\end{tabular} 
    & 1113 & 2 & 39.1 & 42.4           % Description
    & 75.2 & 71.1 & \textbf{74.0} & 70.3 & 70.3 & 69.4           % GCN
    & 75.3 & 60.5 & \textbf{{79.8}} & 69.4 & 68.5 & 68.5           % GAT
    & \textbf{79.3} & 74.9 & 67.7 & 72.1 & 73.0 & 73.9 \\          % GIN
\begin{tabular}[c]{@{}l@{}}\gls{ENZ} \\ 
\citep{DBLP:conf/ismb/BorgwardtOSVSK05}\end{tabular} 
    & 600 & 6 & 32.6 & 32.0            % Description
    & 34.2 & 37.5 & 35.8 & \textbf{43.3} & 40.0 & 31.7           % GCN
    &32.5 & 35.0 & 35.8 & 23.3 & 43.3 & \textbf{45.0}            % GAT
    & 36.7 & 35.0 & \textbf{39.2} & 25.0 & 30.0 & 30.0 \\          % GIN
\bottomrule
\end{tabular}}
\end{table}

\subsection{Training}
The adopted loss is the \emph{(binary) Cross-Entropy} for classification and the $L_1$ for the regression.
The chemical datasets have been split in train/validation/test sets with a ratio of $80/10/10$ in a stratified fashion, ensuring that the splits have a balanced amount of classes. The models are randomly initialized and trained for $500$ epochs with \texttt{early\_stopping} with patience of $100$ on the validation set and a batch size of $64$.
The models have been trained with \emph{Adam} optimizer with a starting learning rate of $0.01$, using a halving scheduler with the patience of $50$ epochs up to $1e^{-5}$, and \texttt{weight\_decay} factor of $5e^{-4}$.

The test accuracies for the best-validated model configuration reported in~\cref{tab_setup}, are comparable to existing benchmarks \citep{Errica2019AFC,  You2020DesignSF}, confirming that our approach is valid for complex and powerful GNN architectures, as long as Assumption~\ref{assumption_GNN} holds.
Moreover, deeper \glspl*{GNN} do not necessarily yield better performances, cf. \cref{tab_performance_deep_gnns}.

\clearpage
\section{Additional Experimental Results}\label{appx_sec_additional_experimental_results}
This section includes additional empirical results and further details on the experimental setup.
\cref{appx_sec_complexity} contains additional complexity results for benchmark datasets.
\cref{appx_sec_wdn} shows a timeline for the \gls*{WDN} example use case.
\cref{appx_sec_approximation} shows the approximation quality of GraphSHAP-IQ compared to a collection of state-of-the-art baseline methods.
\cref{appx_sec_explanation_graphs} showcases example \glspl*{SI} on molecule structures from different datasets.

\subsection{Complexity Analysis}\label{appx_sec_complexity}
In this experiment we provide further results on \cref{sec_exp_complexity}.
We evaluate the complexity on \glspl*{GNN} with varying number of convolutional layers on multiple benchmark datasets, as described in \cref{tab_setup}.
The complexity is measured by number of evaluations of the GNN-induced graph game, i.e. number of model calls of the \gls*{GNN}, which is the limiting factor of \glspl*{SI} \citep{Fumagalli.2023,Kolpaczki.2024,DBLP:conf/aaai/MuschalikFHH24}.
Due to \cref{theorem_complexity}, the complexity of GraphSHAP-IQ does only depend on the message passing ranges and is in particular independent of the message passing technique.
We therefore omit the explicit architecture in this experiment and focus only on the number of convolutional layers $\ell$.
For every graph in the benchmark datasets, we evaluate the complexity of GraphSHAP-IQ, where the model-agnostic method requires $2^n$, where $n$ is the number of nodes in the graph.
If the size of the largest neighborhood exceeds $22$, i.e. GraphSHAP-IQ requires at least $2^{23} \approx 8.3\times 10^6$ model calls, then the complexity is approximated by the upper bound presented in \cref{theorem_complexity}.
\cref{appx_fig_complexity_1}, \cref{appx_fig_complexity_2}, and \cref{appx_fig_complexity_3} display the complexity (in $\log10)$ by the size of the graphs for all instances (upper) and the median, Q1, and Q3 per graph size ($n$) for all benchmark datasets.
The model-agnostic computation is represented by the dashed line.
The results show that the ground-truth computation of \glspl*{SI} is substantially reduced by GraphSHAP-IQ for many instances in the benchmark datasets.

\begin{figure}[htbp]
    \centering
    \begin{minipage}[c]{0.3\textwidth}
    \centering
        \includegraphics[width=\textwidth]{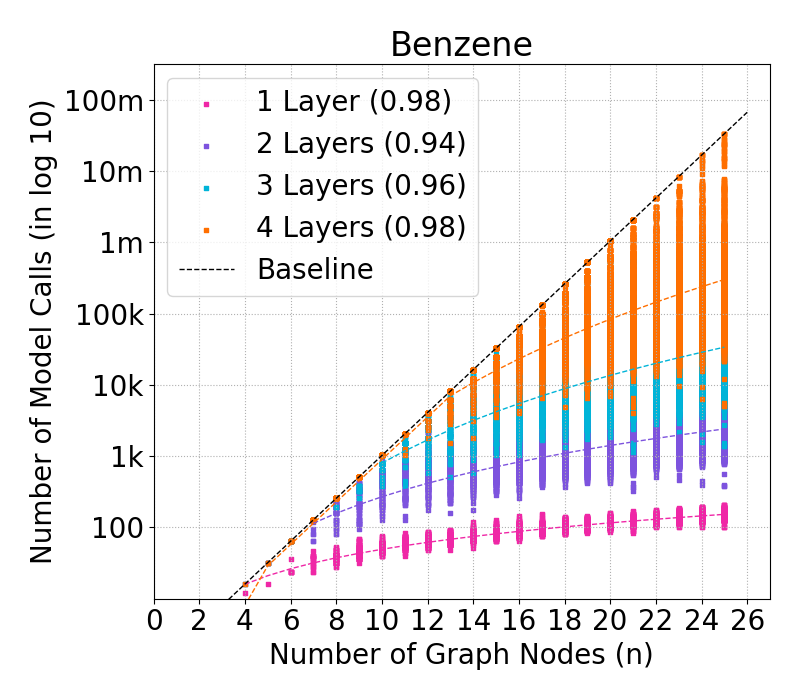}
        \includegraphics[width=\textwidth]{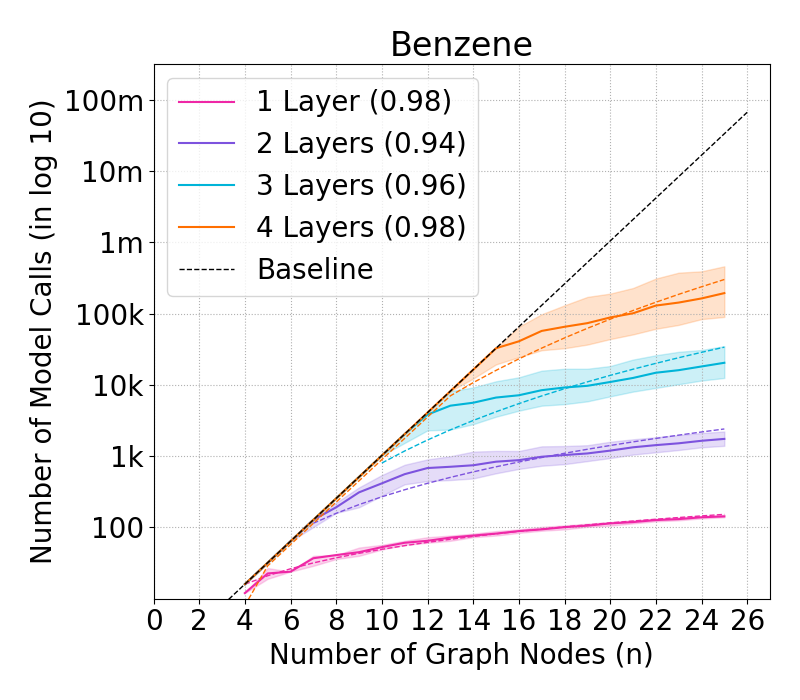}
    \end{minipage}
    \hfill
    \begin{minipage}[c]{0.3\textwidth}
    \centering
        \includegraphics[width=\textwidth]{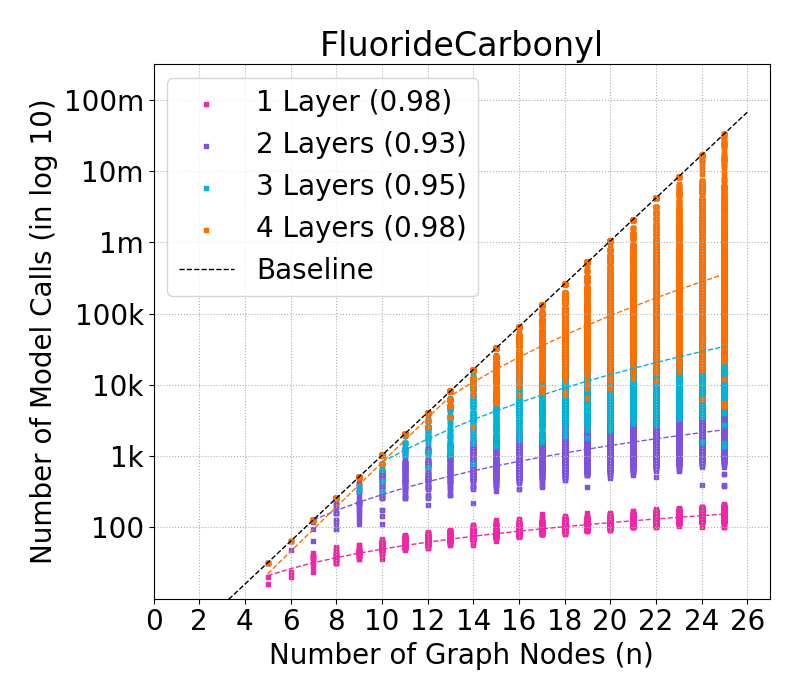}
        \includegraphics[width=\textwidth]{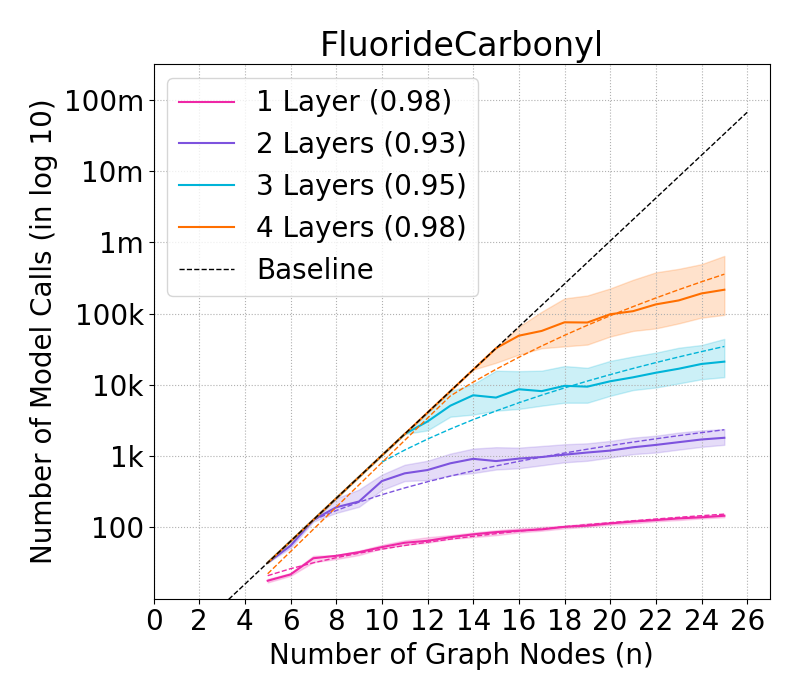}
    \end{minipage}
    \hfill
    \begin{minipage}[c]{0.3\textwidth}
    \centering
        \includegraphics[width=\textwidth]{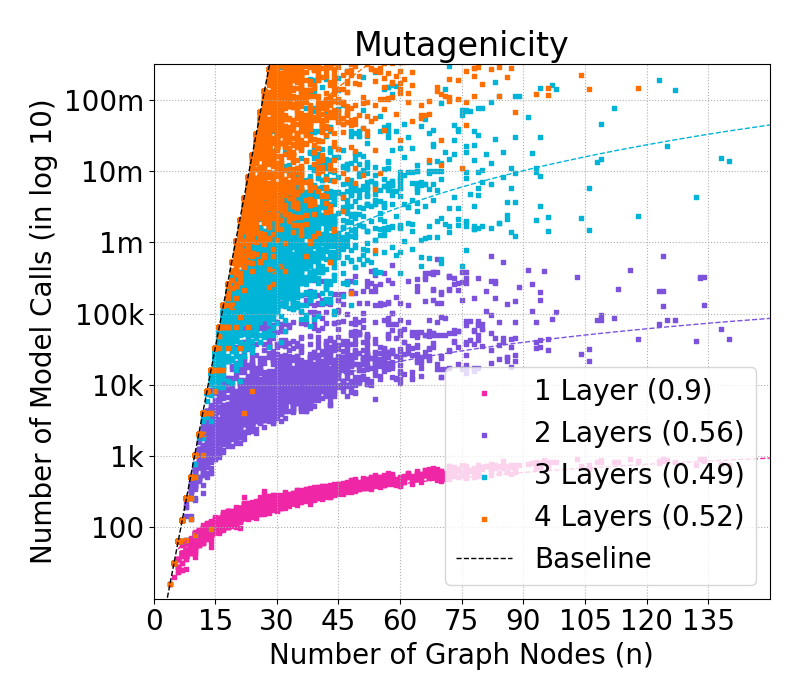}
        \includegraphics[width=\textwidth]{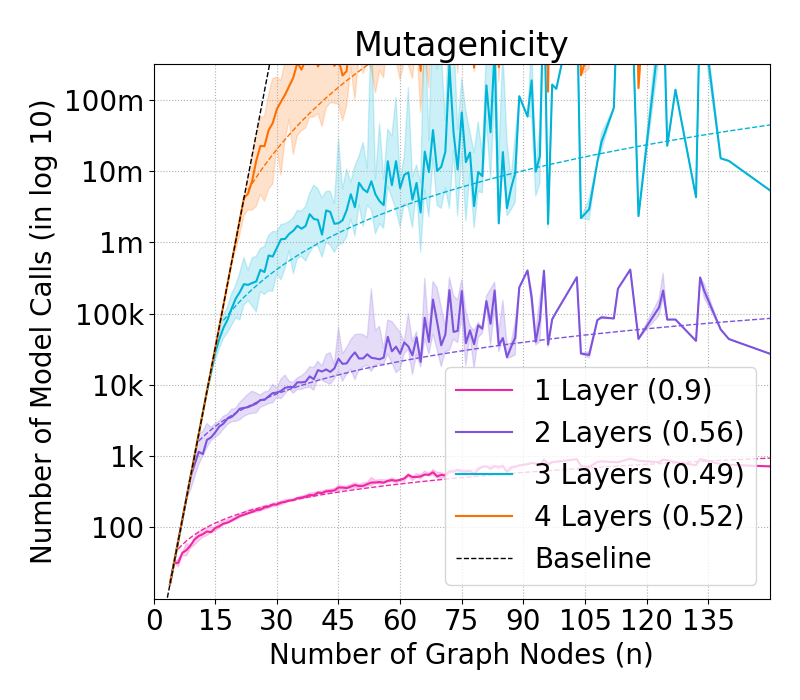}
    \end{minipage}
    \caption{Complexity of GraphSHAP-IQ against model-agnostic baseline (dashed) in model calls (in $\log10$) (y-axis) by number of nodes (x-axis) for all instances (upper) and median, Q1, Q3 (lower) for \gls*{BNZ} (left), \gls*{FLC} (middle), and \gls*{MTG} (right).}
    \label{appx_fig_complexity_1}
\end{figure}

\begin{figure}
    \centering
    \begin{minipage}[c]{0.3\textwidth}
    \centering
        \includegraphics[width=\textwidth]{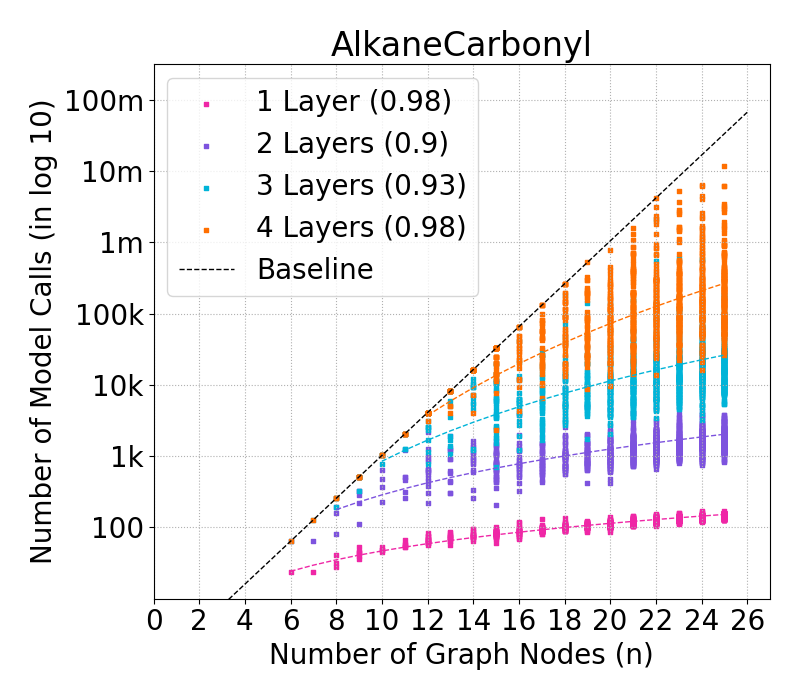}
        \includegraphics[width=\textwidth]{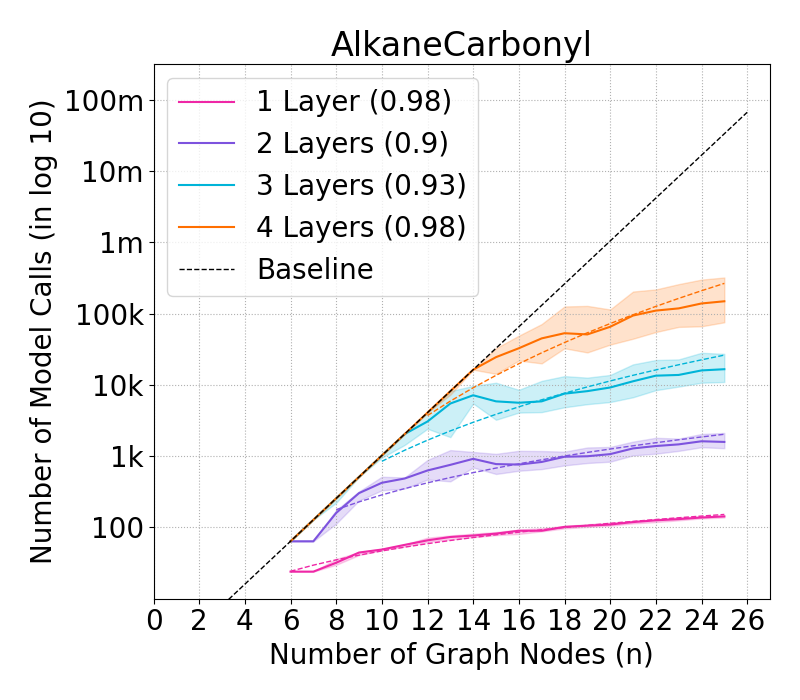}
    \end{minipage}
    \hfill
    \begin{minipage}[c]{0.3\textwidth}
    \centering
        \includegraphics[width=\textwidth]{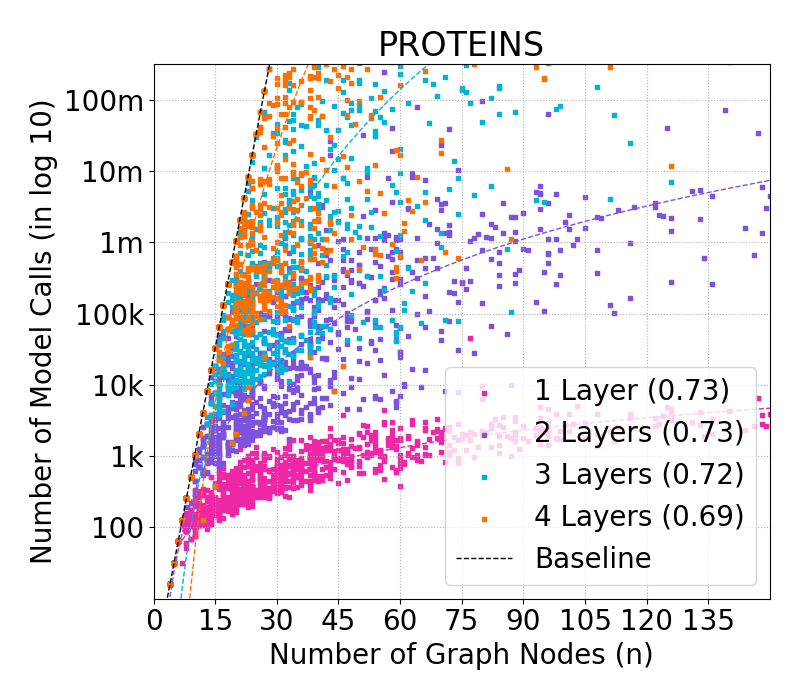}
        \includegraphics[width=\textwidth]{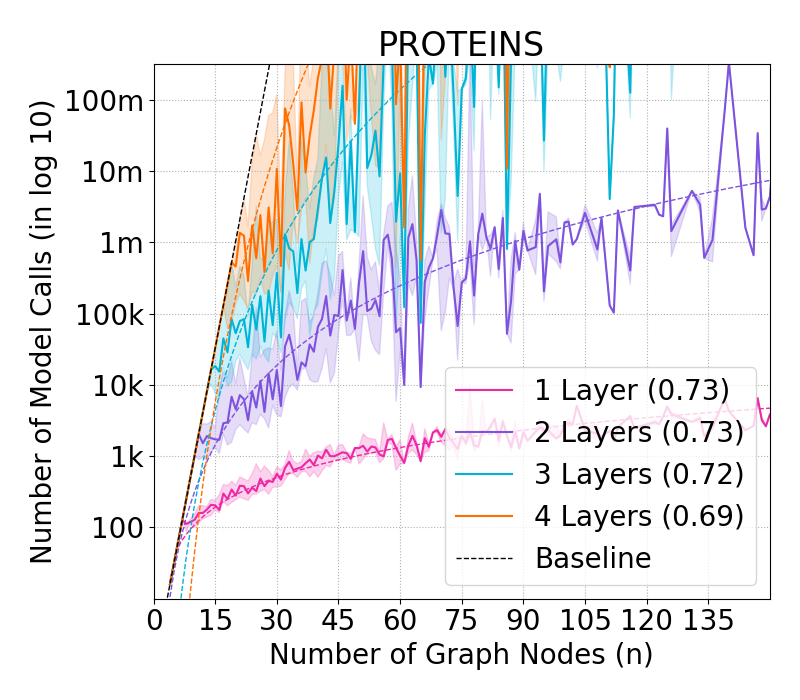}
    \end{minipage}
    \hfill
    \begin{minipage}[c]{0.3\textwidth}
    \centering
        \includegraphics[width=\textwidth]{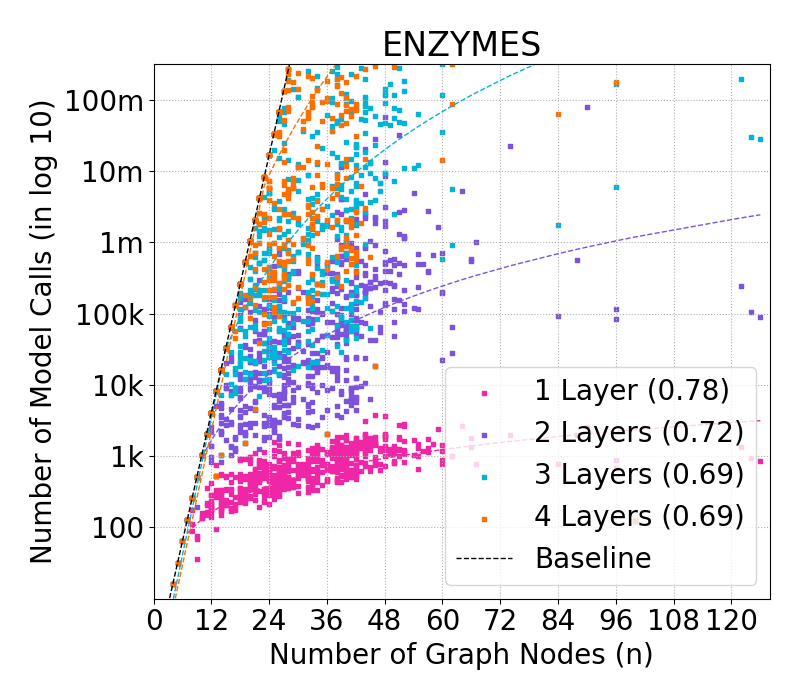}
        \includegraphics[width=\textwidth]{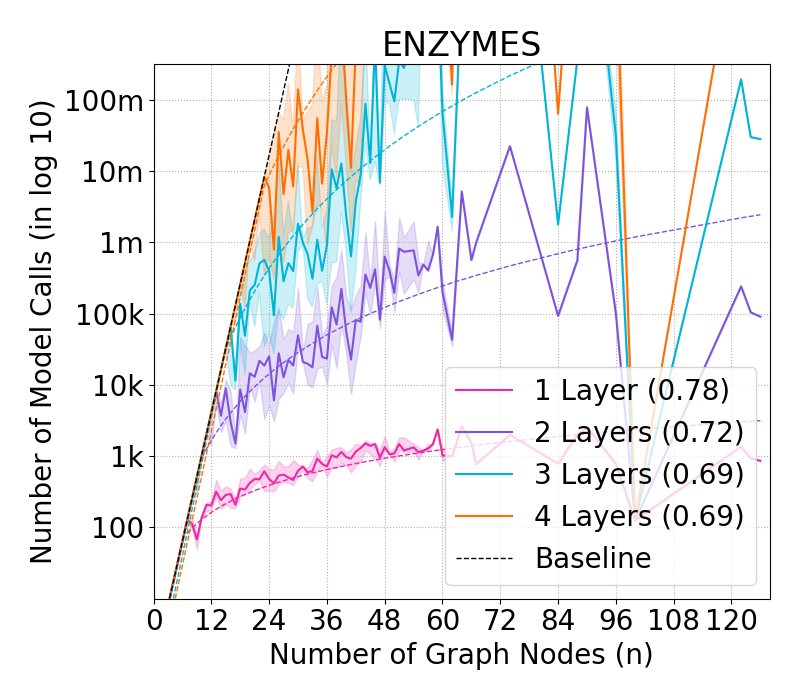}
    \end{minipage}
    \caption{Complexity of GraphSHAP-IQ against model-agnostic baseline (dashed) in model calls (in $\log10$) (y-axis) by number of nodes (x-axis) for all instances (upper) and median, Q1, Q3 (lower) for \gls*{ALC} (left), \gls*{PRT} (middle), and \gls*{ENZ} (right).}
    \label{appx_fig_complexity_2}
\end{figure}

\begin{figure}
\centering
    \begin{minipage}[c]{0.3\textwidth}
    \centering
        \includegraphics[width=\textwidth]{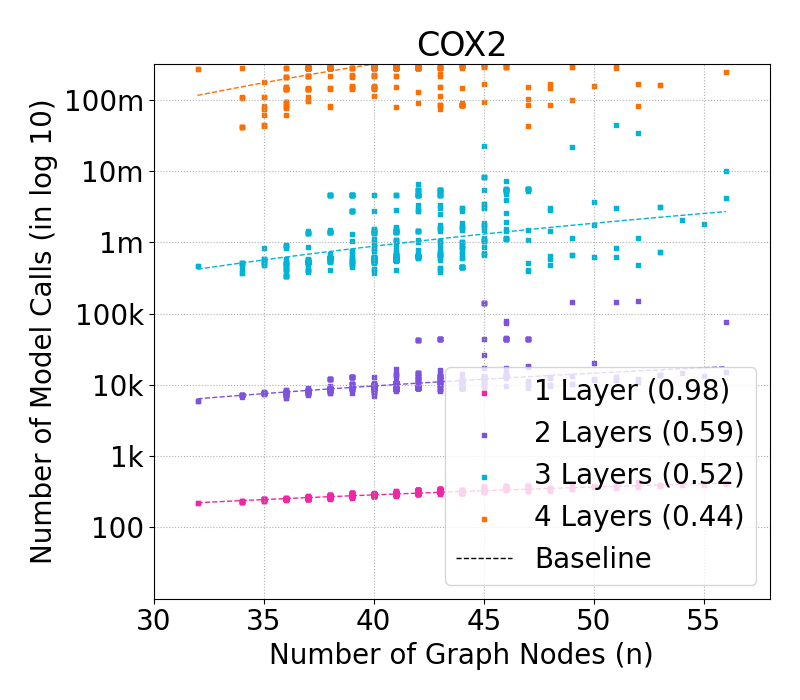}
        \includegraphics[width=\textwidth]{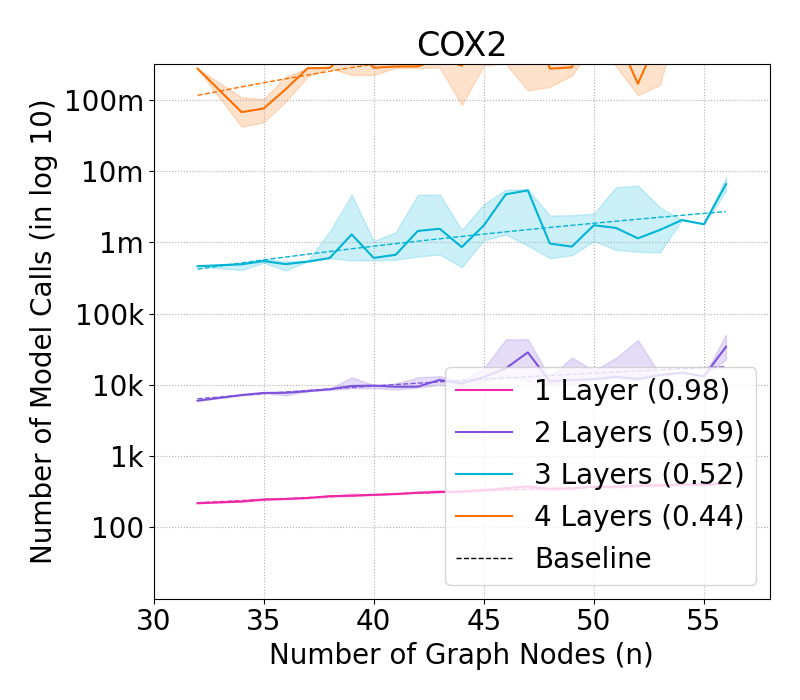}
    \end{minipage}
    \hspace{2em}
    \begin{minipage}[c]{0.3\textwidth}
    \centering
        \includegraphics[width=\textwidth]{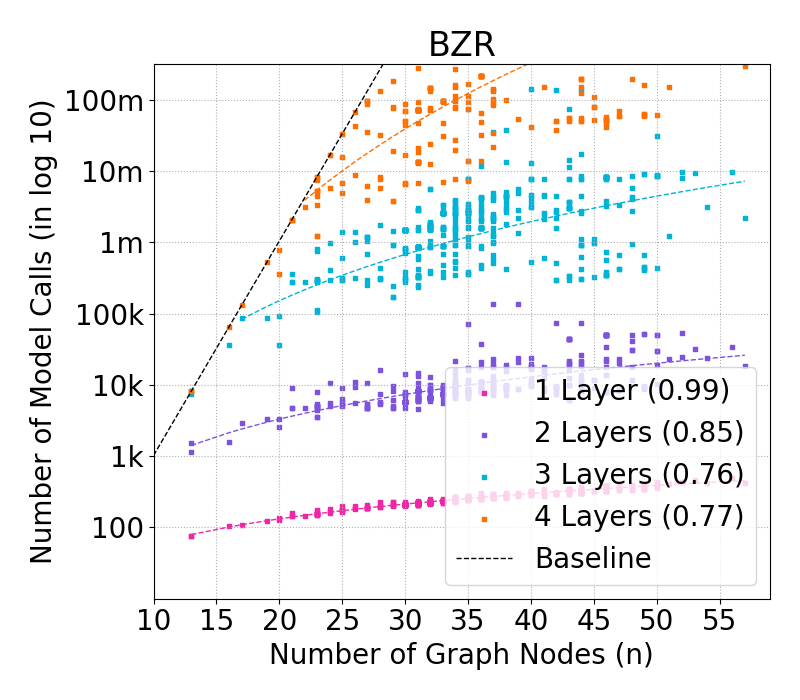}
        \includegraphics[width=\textwidth]{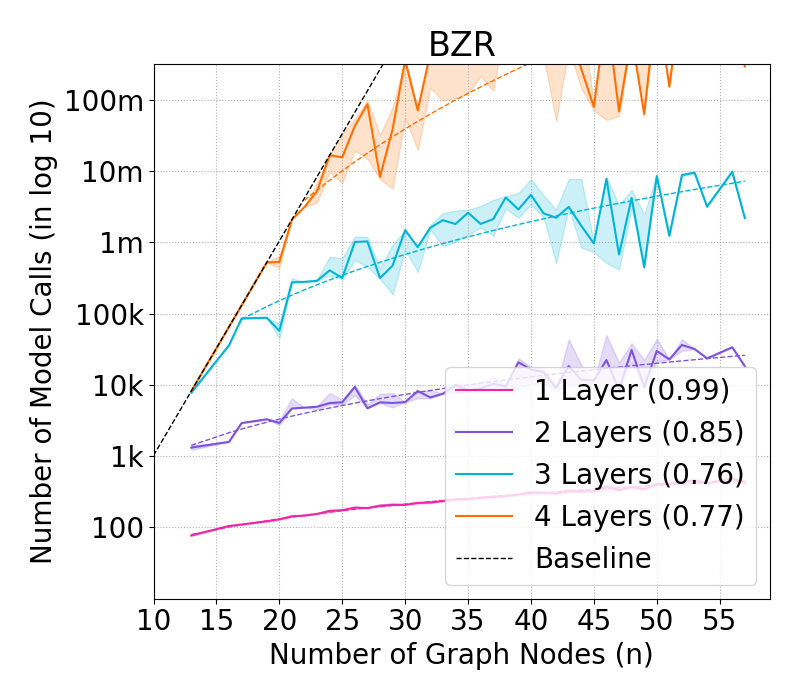}
    \end{minipage}
    \caption{Complexity of GraphSHAP-IQ against model-agnostic baseline (dashed) in model calls (in $\log10$) (y-axis) by number of nodes (x-axis) for all instances (upper) and median, Q1, Q3 (lower) for \gls*{CX2} (left) and \gls*{BZR} (right).}
    \label{appx_fig_complexity_3}
\end{figure}

\clearpage
\subsection{Runtime Analysis}\label{appx_sec_runtime_analysis}
In this section, we conduct a small runtime analysis to confirm that the main driver of computational complexity for GraphSHAP-IQ is indeed the number of model calls.
This is standard in the approximation literature of \glspl*{SI} \citep{Tsai.2022} and confirmed for all baseline methods \citep{fumagalli2024kernelshapiq}.
We select $100$ graph instances from \gls*{MTG} with $20$ to $40$ nodes, where GraphSHAP-IQ requires less than $10000$ model calls for exact computation for a 2-Layer \gls*{GCN}.
We then compute exact \glspl*{SI} for order 1 (\gls*{SV}), 2 (2-\gls*{SII}), and 3 (3-\gls*{SII}) with GraphSHAP-IQ.
For each graph instance, we run all baseline methods using the same budget that GraphSHAP-IQ required for exact computation.
\cref{appx_fig_runtime} shows the number of model calls and runtime (upper row), as well as number of graph nodes and runtime (lower row) for order 1 (left), 2 (middle), and 3 (right).
As expected, GraphSHAP-IQ's runtime scales linearly with number of model calls, and is basically unaffected by the size of the graph.
Moreover, given the same number of model calls, GraphSHAP-IQ's runtime is similar to efficient baselines (SHAP-IQ and Permutation Sampling).
Notably, increasing the explanation order barely affects the runtime of GraphSHAP-IQ, whereas it substantially increases the runtime of competitive baselines (KernelSHAP-IQ, SVARM-IQ, SHAP-IQ).

\begin{figure}[h]
\centering
    \begin{minipage}[c]{0.3\textwidth}
    \centering
    \includegraphics[width=\textwidth]{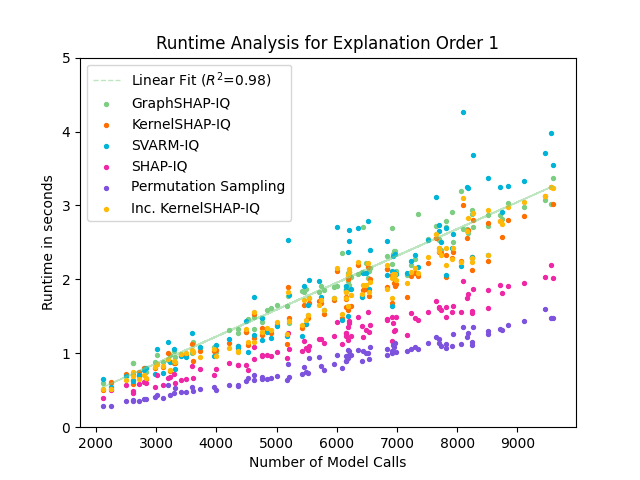}
    \includegraphics[width=\textwidth]{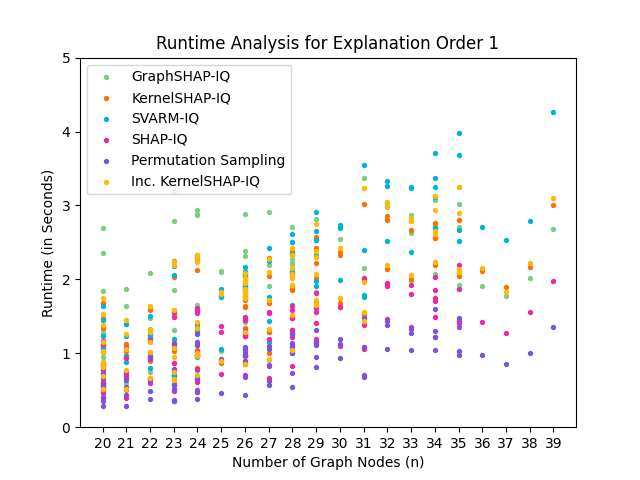}
    \end{minipage}
    \begin{minipage}[c]{0.3\textwidth}
    \centering
    \includegraphics[width=\textwidth]{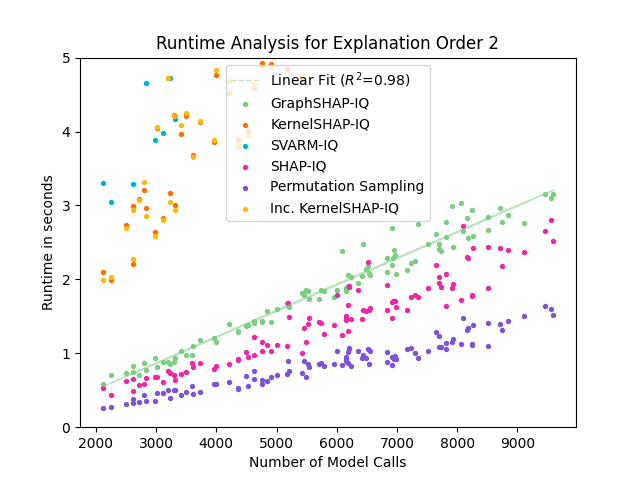}
    \includegraphics[width=\textwidth]{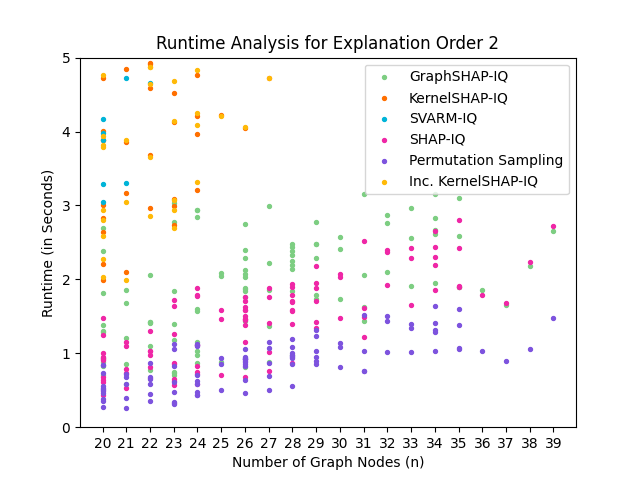}
    \end{minipage}
    \begin{minipage}[c]{0.3\textwidth}
    \centering
    \includegraphics[width=\textwidth]{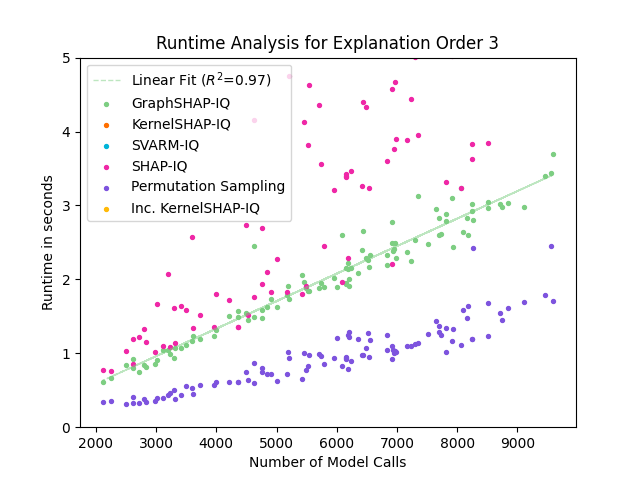}
    \includegraphics[width=\textwidth]{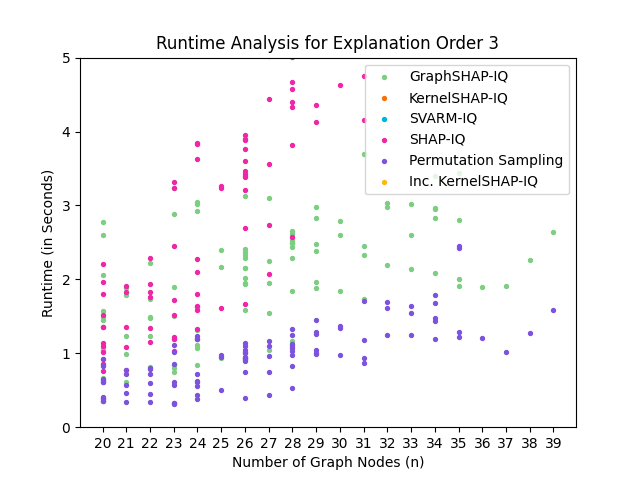}
    \end{minipage}
    \caption{Runtime comparison of GraphSHAP-IQ and baseline methods  for \glspl*{SV} (left), 2-\gls*{SII} (middle), and 3-\gls*{SII} (right). All methods were given the number of model calls, which were required for exact computation with GraphSHAP-IQ.
    The runtime (y-axis) against number of model calls (x-axis) plot (upper row) shows that the runtime of GraphSHAP-IQ scales linearly ($R^2>0.97)$ with increasing number of model calls. In contrast, the runtime (y-axis) against number of graph nodes (x-axis) plot (lower row) shows that the runtime is independent from the size of the graph. Lastly, increasing the explanation order substantially increases the runtime of most baseline methods, but barely affects GraphSHAP-IQ.}\label{appx_fig_runtime}
\end{figure}

\clearpage
\subsection{Explaining Water Quality in Water Distribution Networks (WDNs)}\label{appx_sec_wdn}
We investigate the validity of our approach in explaining a critical real-world scenario: adding chlorine to \gls*{WDS} is a common disinfection practice used to inactivate microorganisms that can cause waterborne diseases. Chlorine is an effective disinfectant due to its ability to oxidize and denature microbial cells. However, it can be toxic to aquatic organisms, particularly at high concentrations, and can impart an unpleasant taste or odor to the water, affecting consumer acceptance. In practice, chlorine is added to a water \emph{reservoir} and flows unevenly through the network. Therefore, it is essential to model these flows and concentrations, which is not trivial since it is a dynamical system governed by local PDEs. We frame this problem as graph regression task, where the goal is to predict the relative chlorine concentration for each node at each time step. In \cref{appx_fig_wdn}, we report the Hanoi \gls*{WDN} \citep{Vrachamis2018WaterDistributionNetworks_Hanoi} and simulate a chlorine injection at node reservoir $31$ for $t=0$. We observe how the $2$-\gls*{SII} highlight the progressive importance of the nodes --- from the reservoir throughout the network. 

\begin{figure}[htbp]
    \centering
    \begin{minipage}[c]{0.24\textwidth}
    \centering
        \includegraphics[width=\textwidth]{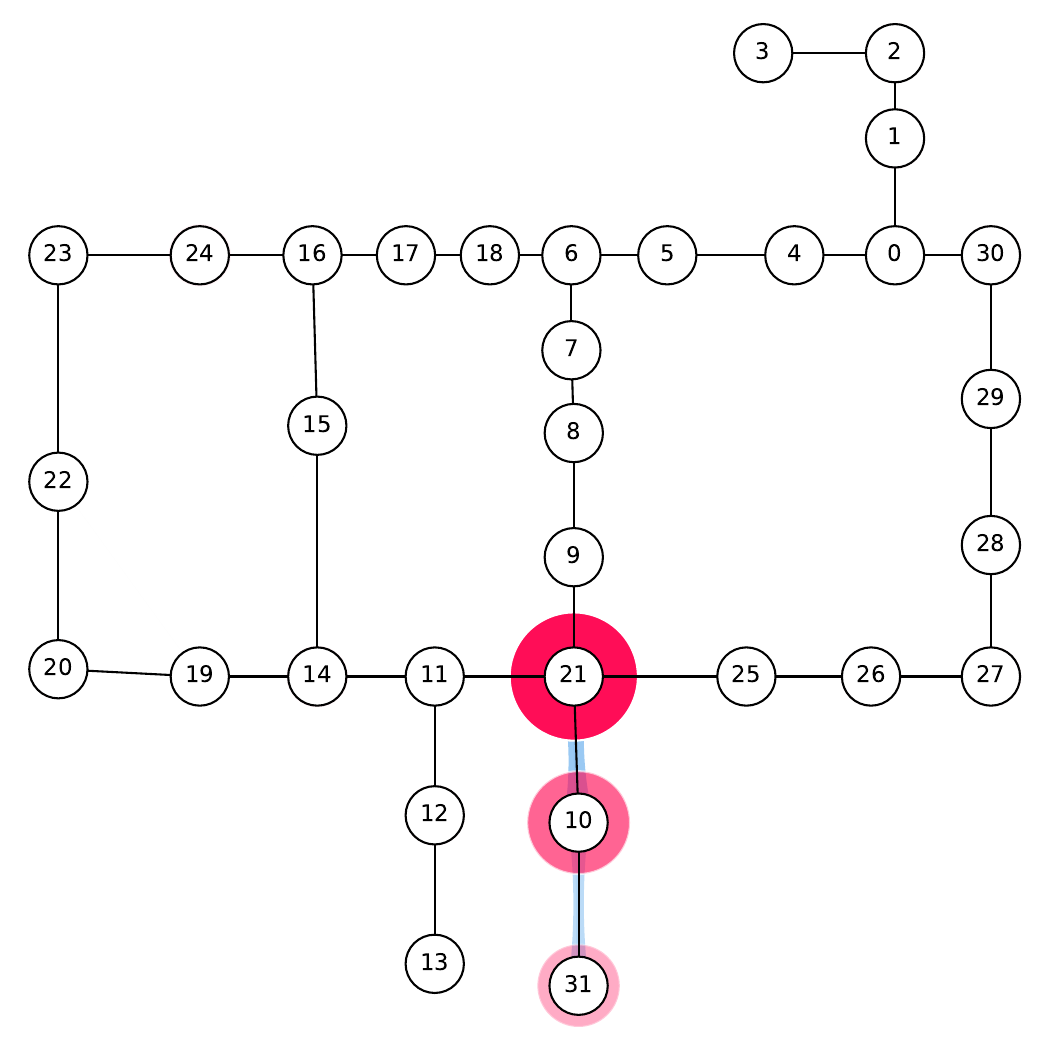} 
        \\
        \textbf{Time Step 1}
    \end{minipage}
    \begin{minipage}[c]{0.24\textwidth}
    \centering
        \includegraphics[width=\textwidth]{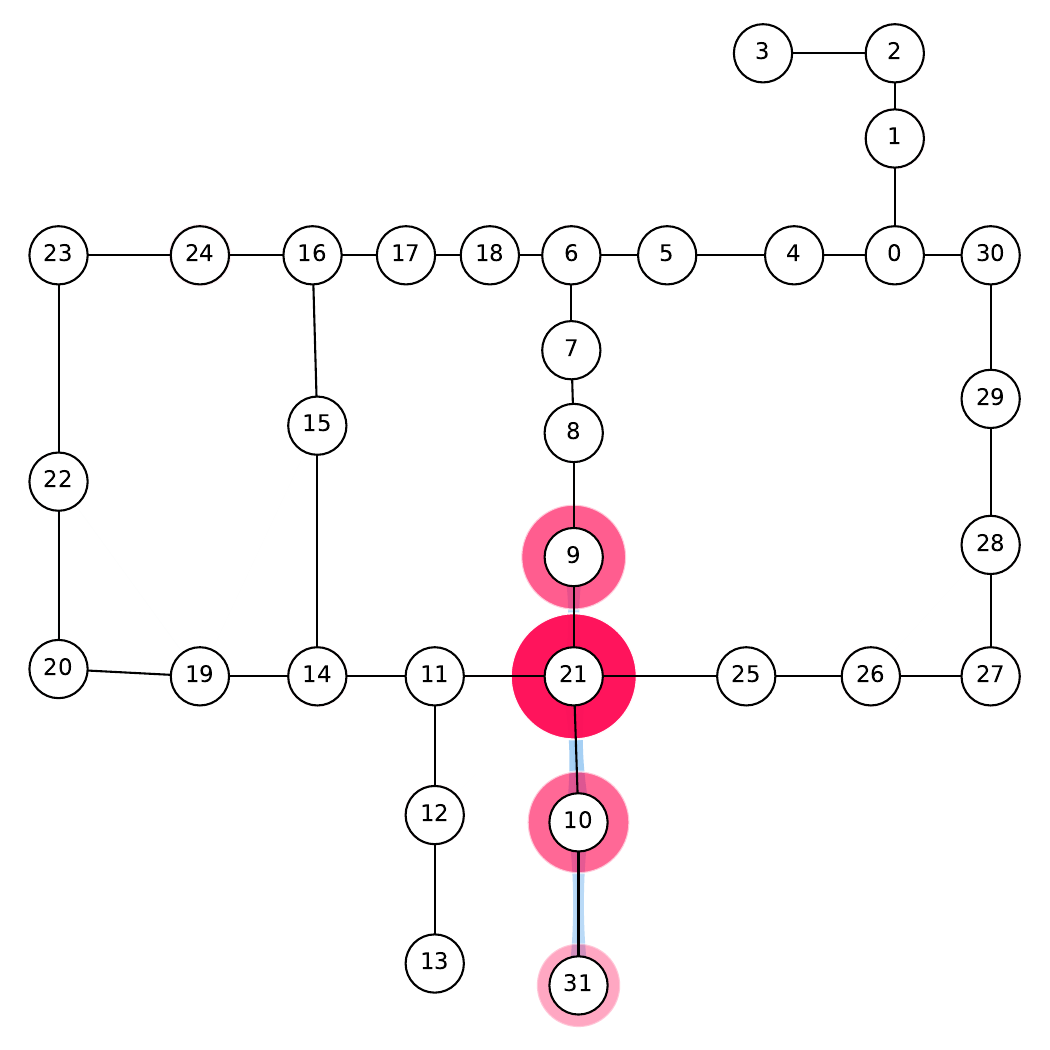}   
        \\
        \textbf{Time Step 2}
    \end{minipage}
    \begin{minipage}[c]{0.24\textwidth}
    \centering
        \includegraphics[width=\textwidth]{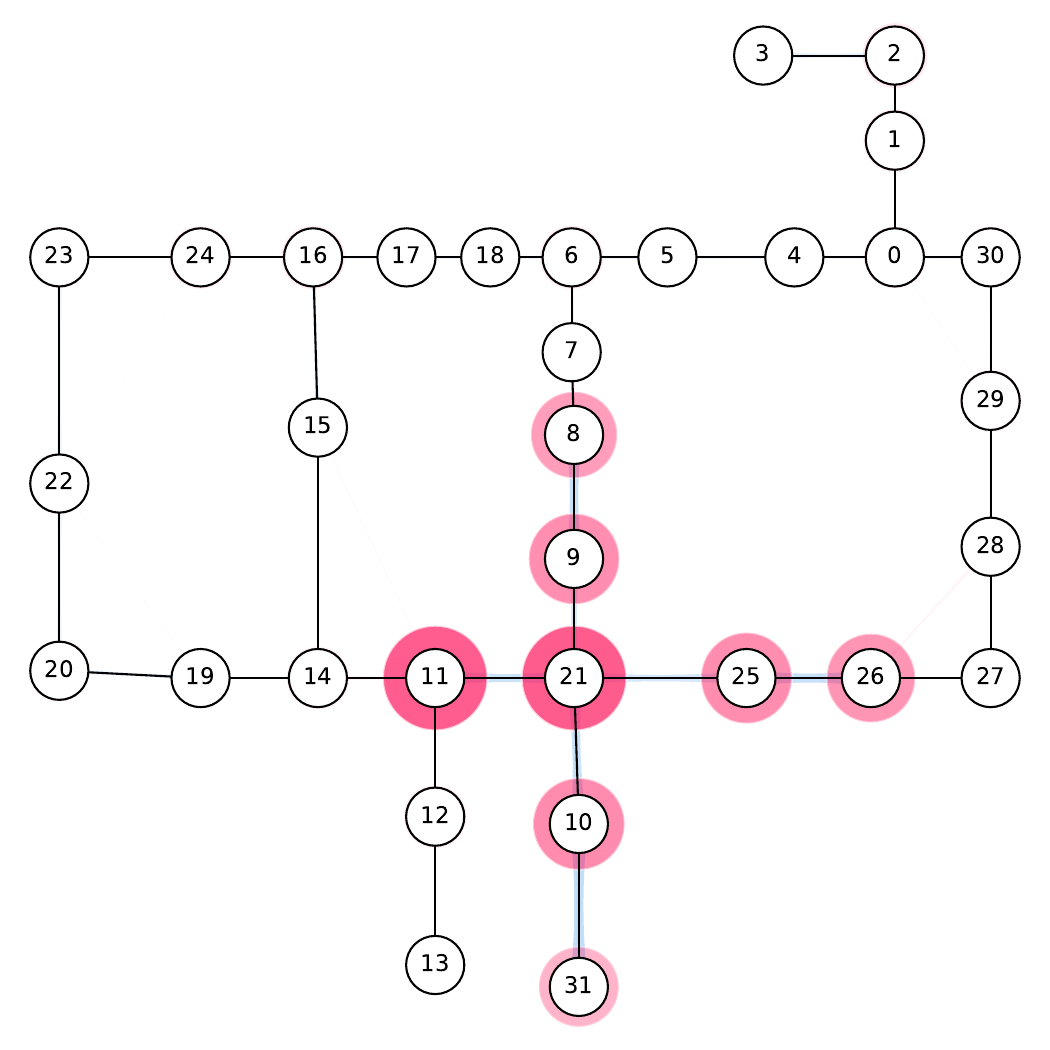}      
        \\
        \textbf{Time Step 5}
    \end{minipage}
    \begin{minipage}[c]{0.24\textwidth}
    \centering
        \includegraphics[width=\textwidth]{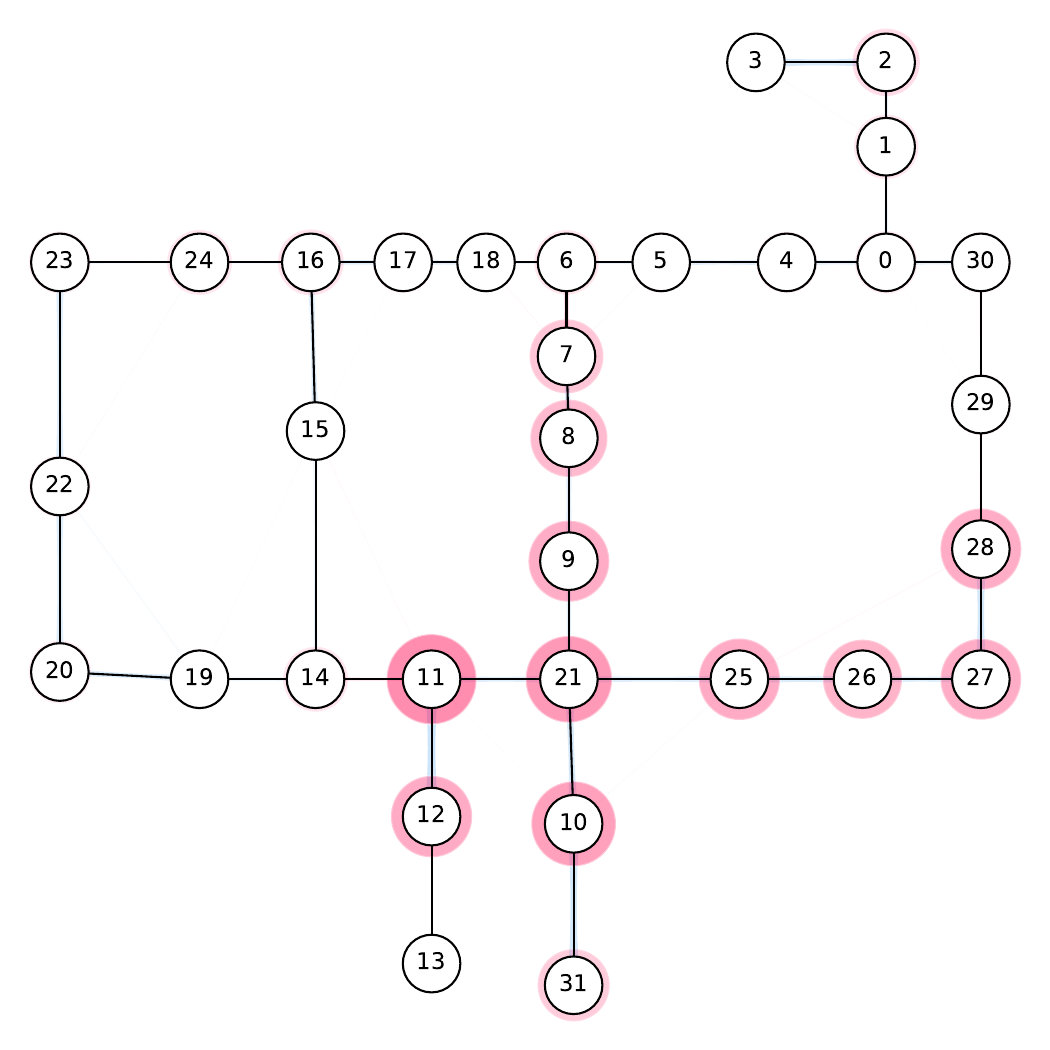}  
        \\
        \textbf{Time Step 10}
    \end{minipage}
    \caption{Spread of chlorination through a \gls*{WDN} over time as explained by $2$-\gls*{SII}.}
    \label{appx_fig_wdn}
\end{figure}

\subsection{Approximation Quality of GraphSHAP-IQ}\label{appx_sec_approximation}

Additionally to the experiments conducted in \cref{sec_approximation}, this section contains further evaluations on three datasets and a detailed description of the \gls*{SV} baselines (\cref{appx_sec_sv_baseline}) and the 2-\gls*{SII} baselines (\cref{appx_sec_2_sii_baseline}).
\cref{appx_fig_approx_mtg,appx_fig_approx_prt,appx_fig_approx_bzr} show the approximation quality of different \gls*{GNN} architectures for the \gls*{MTG}, \gls*{PRT}, and \gls*{BZR} benchmark datasets.
\cref{appx_sec_approx_description} outlines the approximation experiments in detail.

\subsubsection{Setup Description}\label{appx_sec_approx_description}
To compare GraphSHAP-IQ's approximation result with existing model-agnostic baselines, we compute \glspl*{SI} for three benchmark datasets (\gls*{MTG}, \gls*{PRT}, and \gls*{BZR}).
For each dataset, we randomly select 10 graphs containing $30 \leq n \leq 40$ nodes each. 
Based on \cref{sec_exp_complexity}, we further limit the selection to graphs which can be exactly computed via GraphSHAP-IQ with at most $10\,000$ (for \gls*{MTG} and \gls*{PRT}) and $2^{15} = 32\,768$ (for \gls*{BZR}) model evaluations.
First, we compute \glspl*{SI} with GraphSHAP-IQ for each \gls*{MI} interaction order $\lambda$ starting with $\lambda = 1$ until $\lambda \geq \max_{i \in N}(\vert \Nbh_i^{(\ell)} \vert)-1$ (c.f. Corollary~\ref{appx_cor_graphshapiq}).
Note that for different graphs even for the same datasets the maximum $\lambda$ value may differ.
For each $\lambda$, we observe the required model evaluations by GraphSHAP-IQ.
Second, we compute \glspl*{SI} with the baseline methods provided the same approximation budget as required by GraphSHAP-IQ for each $\lambda$.
Therein, we directly compare how the baselines (which may be run with any arbitrary approximation budget) compare with GraphSHAP-IQ. 
For each estimation budget (number of model evaluation), we estimate the \glspl*{SI} via two independent iterations for each baseline and average the evaluation metrics (each dot in \cref{appx_fig_approx_mtg,appx_fig_approx_prt,appx_fig_approx_bzr,appx_fig_approx_sv} are averaged evaluation metrics over two runs).
As evaluation metric we choose the mean-squared-error (MSE, lower is better).
For each graph, we compute the MSE between the ground truth \glspl*{SI}, as computed by exact GraphSHAP-IQ, and the estimated values.

\subsubsection{Description of SV Baselines}\label{appx_sec_sv_baseline}
As \gls*{SV} baselines we apply current state-of-the-art sampling based methods that operate on different representations of the \gls*{SV} and/or aggregation technique.
In general, all of the following baseline methods (with L-Shapley as the only exception) are sampling coalitions of players, evaluate the model on these coalitions, observe the output of the model, and finally aggregate the observed outputs into the \gls*{SV}.
All \gls*{SV} baselines are implemented according to the \textit{shapiq} \citep{Fumagalli.2023} open-source software package in Python.

\textbf{KernelSHAP} as proposed in \citep{DBLP:conf/nips/LundbergL17} is a prominent \gls*{SV} approximation method and is applied akin to its original conception in addition to the sampling tricks discussed in \citep{Covert.2021,Fumagalli.2023}.

\textbf{k-Additive KernelSHAP} \citep{Pelegrina.2023} is an extension to KernelSHAP, which also utilizes higher-order \glspl*{SI} in the computation procedure of the \gls*{SV}. Therein, it shows better approximation qualities than KernelSHAP. For our experiments, we set the higher-order interactions to $2$.

\textbf{Unbiased KernelSHAP} \citep{Covert.2021} extends on KernelSHAP and offers a provably unbiased alternative. It was recently shown that Unbiased KernelSHAP is linked to the \gls*{k-SII} SHAP-IQ estimator \citep{Fumagalli.2023}.

\textbf{Permutation Sampling} for \gls*{SV} \citep{Castro.2009} is a standard estimation method that iterates over random permuations of the player set to determine the coalitions to be used for the model evaluations. Because a permutation always needs to be traversed in its entirety, the number of model calls may be lower than GraphSHAP-IQ's or the rest of the baselines.

\textbf{SVARM} \citep{DBLP:conf/aaai/KolpaczkiBMH24} is another sampling-based baseline operating on a stratified representation of the \glspl*{SV}. Akin to KernelSHAP it can be queried on an arbitrary number of coalitions.

\textbf{L-Shapley} \citep{DBLP:conf/iclr/ChenSWJ19} is a deterministic method for computing the Shapley values for structured data. Unlike the rest of the \gls*{SV} approximation methods, L-Shapley cannot be evaluated on an arbitrary set of coalitions or number of model evaluations. L-Shapley functions similarly to GraphSHAP-IQ in that it deterministically evaluates coalitions based on neighborhoods in the graphs. To circumvent this, we let L-Shapley akin to GraphSHAP-IQ. After L-Shapley exceeds GraphSHAP-IQ's number of model evaluation we stop the iteration and use the last estimates (the estimates where L-Shapley always exceeds GraphSHAP-IQ in terms of model evaluations).

\subsubsection{Description of 2-SII Baselines}\label{appx_sec_2_sii_baseline}

Similar to \cref{appx_sec_sv_baseline}, we apply current state-of-the-art sampling based methods for 2-\gls*{SII} and in general \gls*{k-SII}.
All of the following baseline methods are sampling-based.
We implement the experiments based on the \textit{shapiq} \citep{Fumagalli.2023} open-source Python software package.

\textbf{KernelSHAP-IQ} \citep{fumagalli2024kernelshapiq} directly extends the \gls*{SV} KernelSHAP approximation paradigm to \gls*{SII} and, hence, to \gls*{k-SII}. KernelSHAP-IQ leads to high quality estimates.

\textbf{Inconsistent KernelSHAP-IQ} \citep{fumagalli2024kernelshapiq} is a different version of KernelSHAP-IQ, which is linked to the k-Additive KernelSHAP \citep{Pelegrina.2023}. While often leading to better estimations in lower number of model evaluations, Inconsistent KernelSHAP-IQ does not converge to the ground-truth values like KernelSHAP-IQ.

\textbf{SHAP-IQ} \citep{Fumagalli.2023} is a sampling-based mean estimator for computing \glspl*{SII} among other interaction indices like \gls*{STII}. It is theoretically linked to Unbiased KernelSHAP \citep{Fumagalli.2023} and, thus, transfers its \gls*{SV} procedure to \gls*{SII}.

\textbf{Permutation Sampling} for \gls*{SII}\citep{Tsai.2022} directly transfers the permutation sampling procedure from its \gls*{SV} counterpart \citep{Castro.2009} to estimate \gls*{SII}. Similar to the \gls*{SV} algorithm, the \gls*{SII} variant requires a full pass through a permutation which may lead to less model evaluations with this baseline.

\textbf{SVARM-IQ} \citep{Kolpaczki.2024} extends SVARM's \citep{DBLP:conf/aaai/KolpaczkiBMH24} stratified representation from the \gls*{SV} to the \gls*{SII}. Often SVARM-IQ outperforms state-of-the-art approximation methods.

\subsection{Explanation Graphs for Molecule Structures}\label{appx_sec_explanation_graphs}
This section contains further exemplary \gls*{SI}-graphs for molecule structures of the \gls*{MTG} and \gls*{BNZ} datasets.
\cref{appx_fig_mutag_graphs} shows two molecules from the \gls*{MTG} dataset, where one of those molecules is the same as in \cref{fig_intro_illustration}.
Further, \cref{appx_fig_benzene_graphs} shows additional \gls*{SI}-Graphs for molecules of the \gls*{BNZ} dataset. Lastly, \cref{appx_fig_molecuels_gnn_comps} shows how the \gls*{SI}-Graphs differ for different \gls*{GNN} architectures.

\begin{figure}
    \centering
    \includegraphics[width=0.8\textwidth]{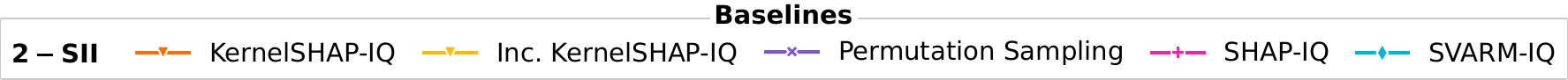}
    \begin{minipage}[c]{0.325\textwidth}
    \centering
        \includegraphics[width=\textwidth]{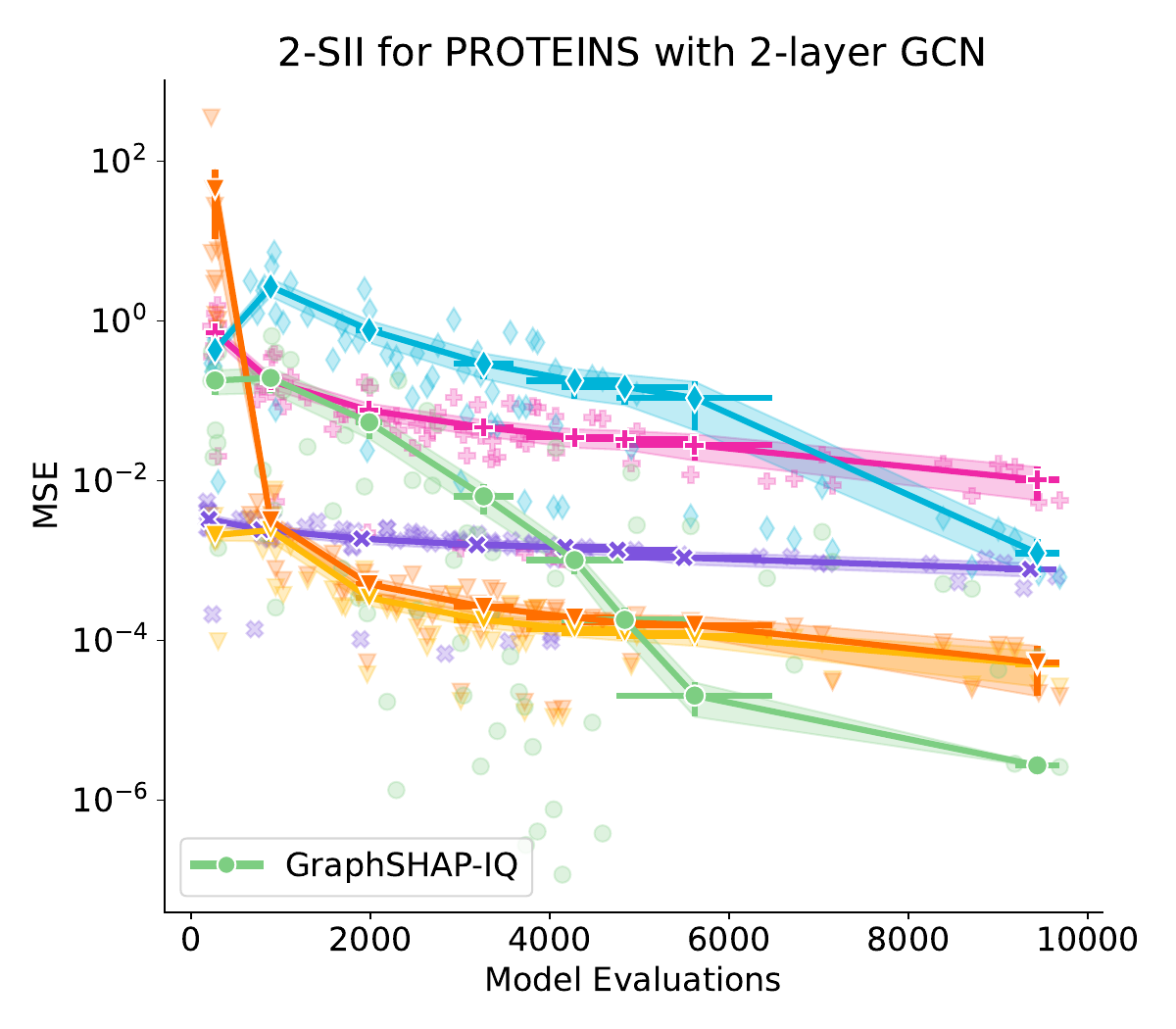}
        \includegraphics[width=\textwidth]{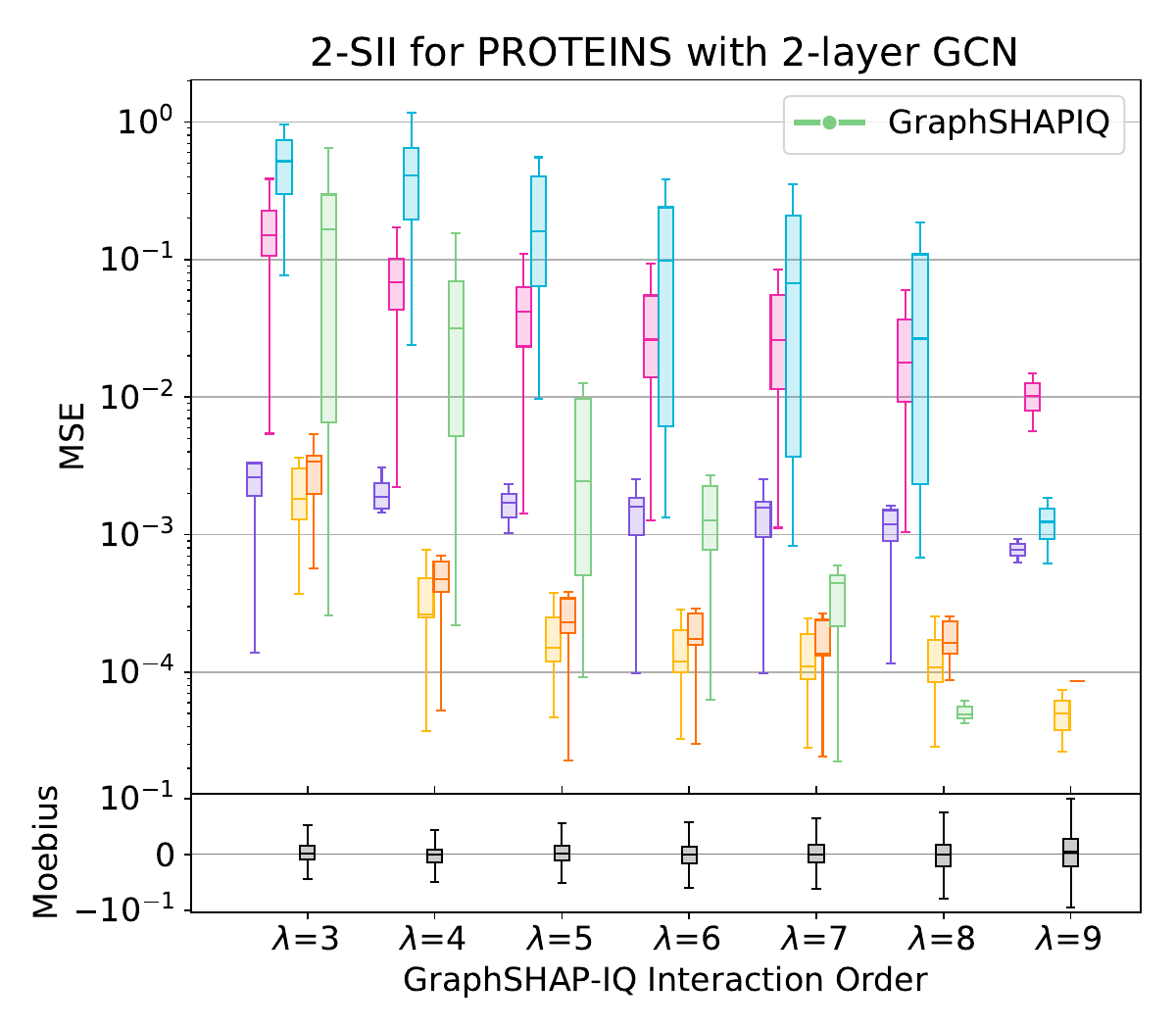}
    \end{minipage}
    \hfill
    \begin{minipage}[c]{0.325\textwidth}
    \centering
        \includegraphics[width=\textwidth]{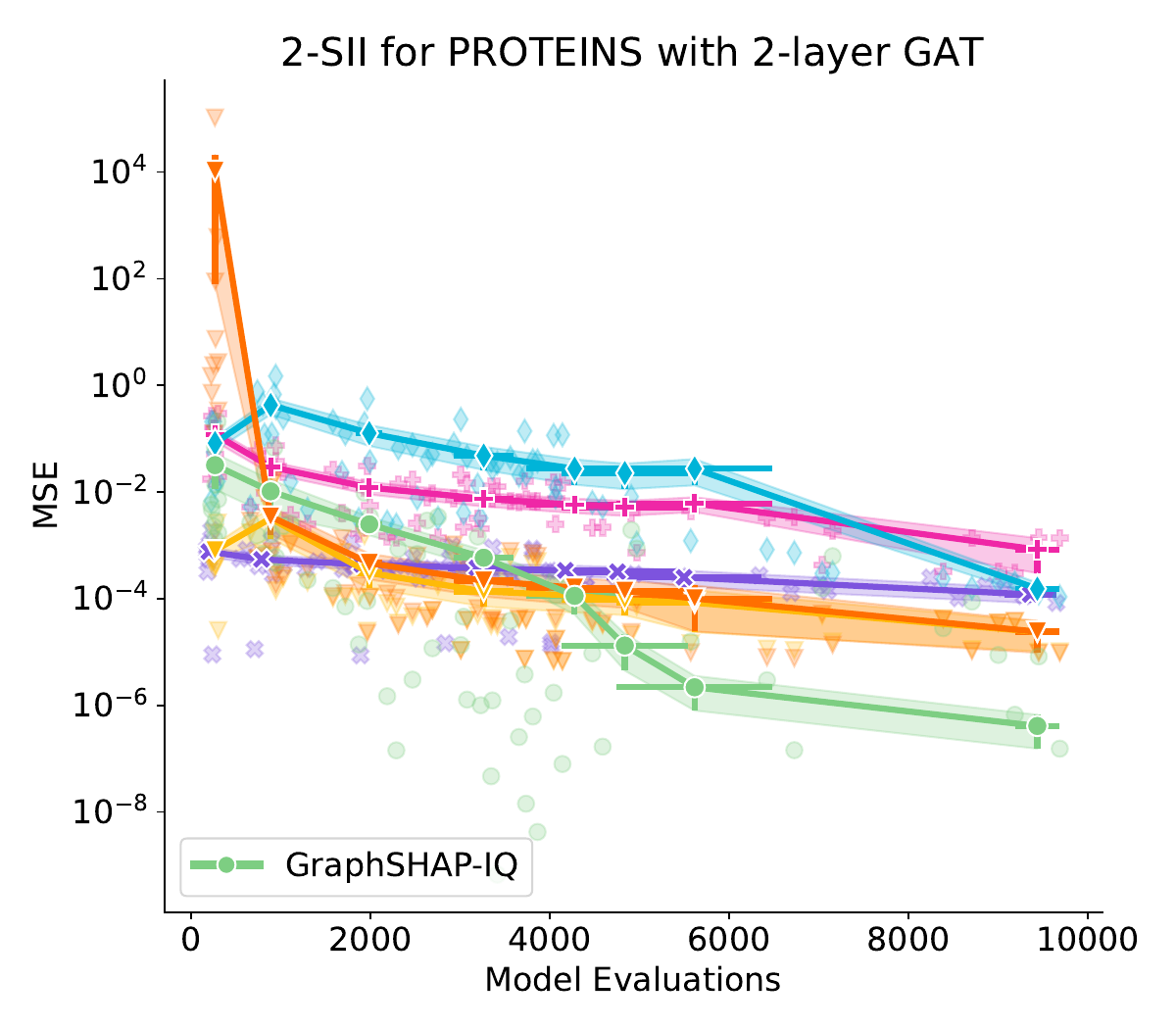}
        \includegraphics[width=\textwidth]{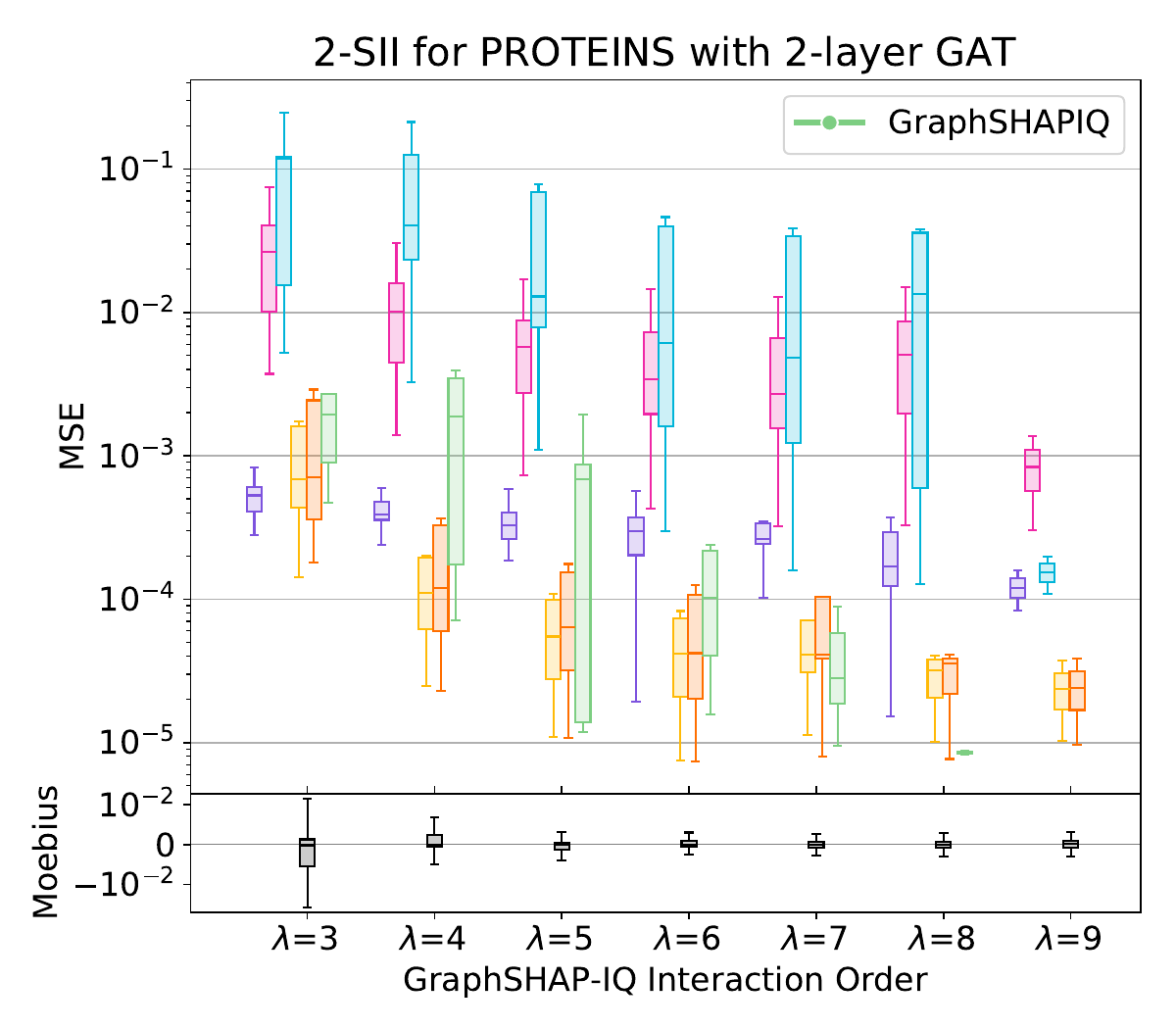}
    \end{minipage}
    \hfill
    \begin{minipage}[c]{0.325\textwidth}
    \centering
        \includegraphics[width=\textwidth]{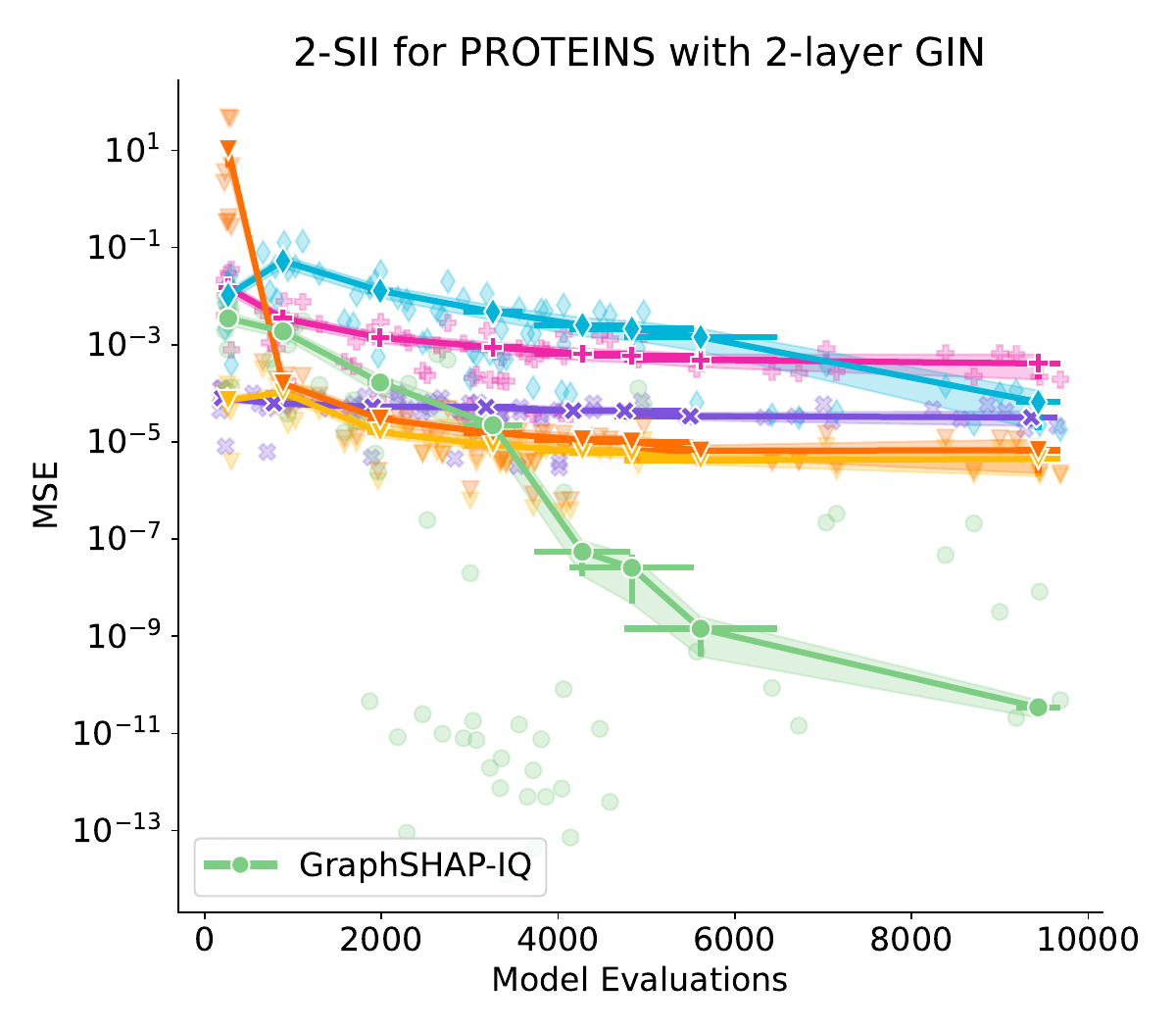}
        \includegraphics[width=\textwidth]{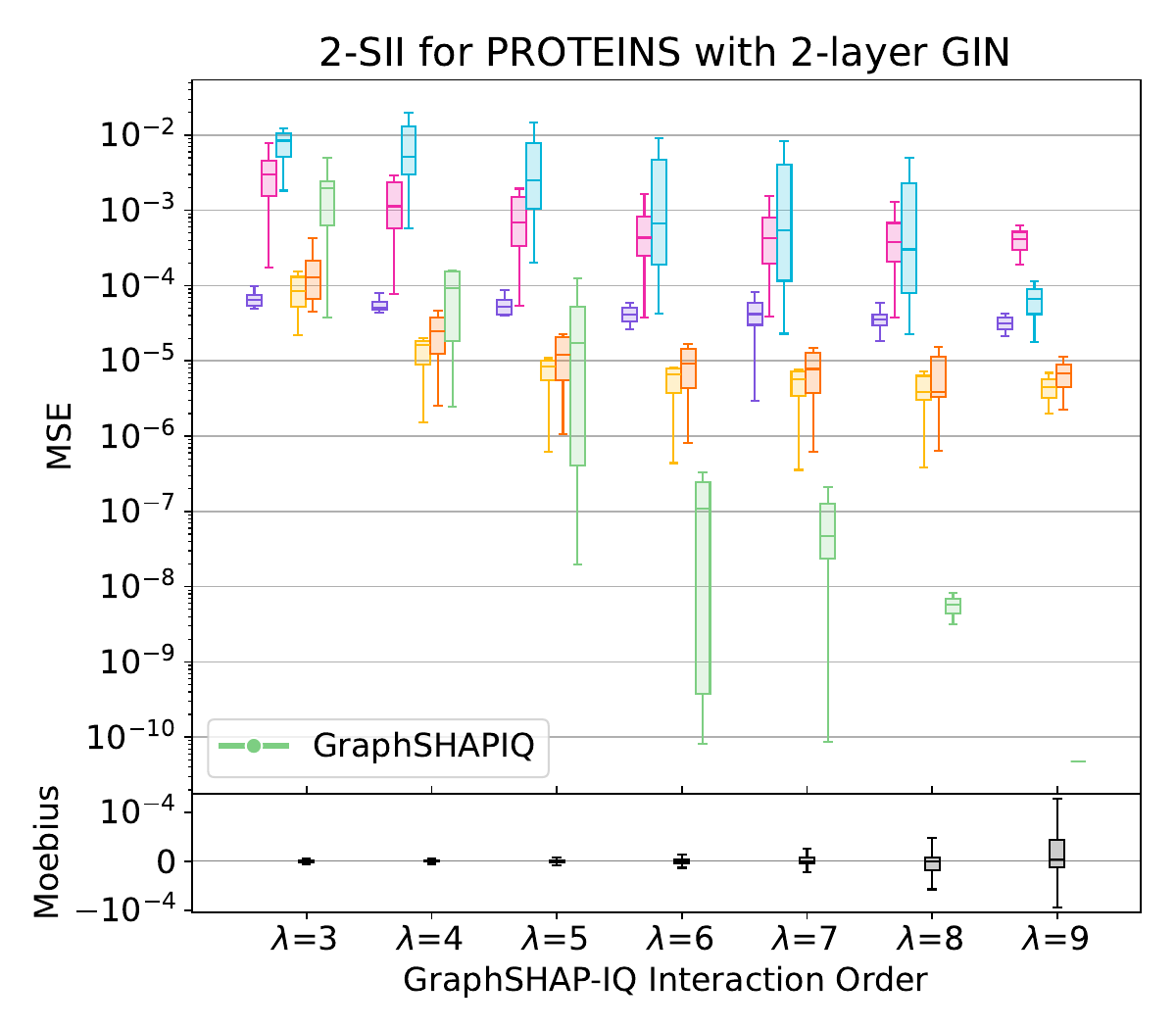}
    \end{minipage}
    \caption{Comparison of GraphSHAP-IQ's approximation quality with model-agnostic baselines on 10 graphs with $30 \leq n \leq 40$ nodes from the \gls*{PRT} dataset for a 2-layers \gls*{GCN} (left), \gls*{GAT} (middle) and \gls*{GIN} (right). The top row presents the MSE for each estimation (dots) and averaged over $\lambda$ (line) with the standard error of the mean (confidence band). The bottom row shows the same information including the \glspl*{MI} as box plots for each $\lambda$.}
    \label{appx_fig_approx_prt}
\end{figure}

\begin{figure}
    \centering
    \includegraphics[width=0.8\textwidth]{figures/experiments/legend_sii.pdf}
    \begin{minipage}[c]{0.325\textwidth}
    \centering
        \includegraphics[width=\textwidth]{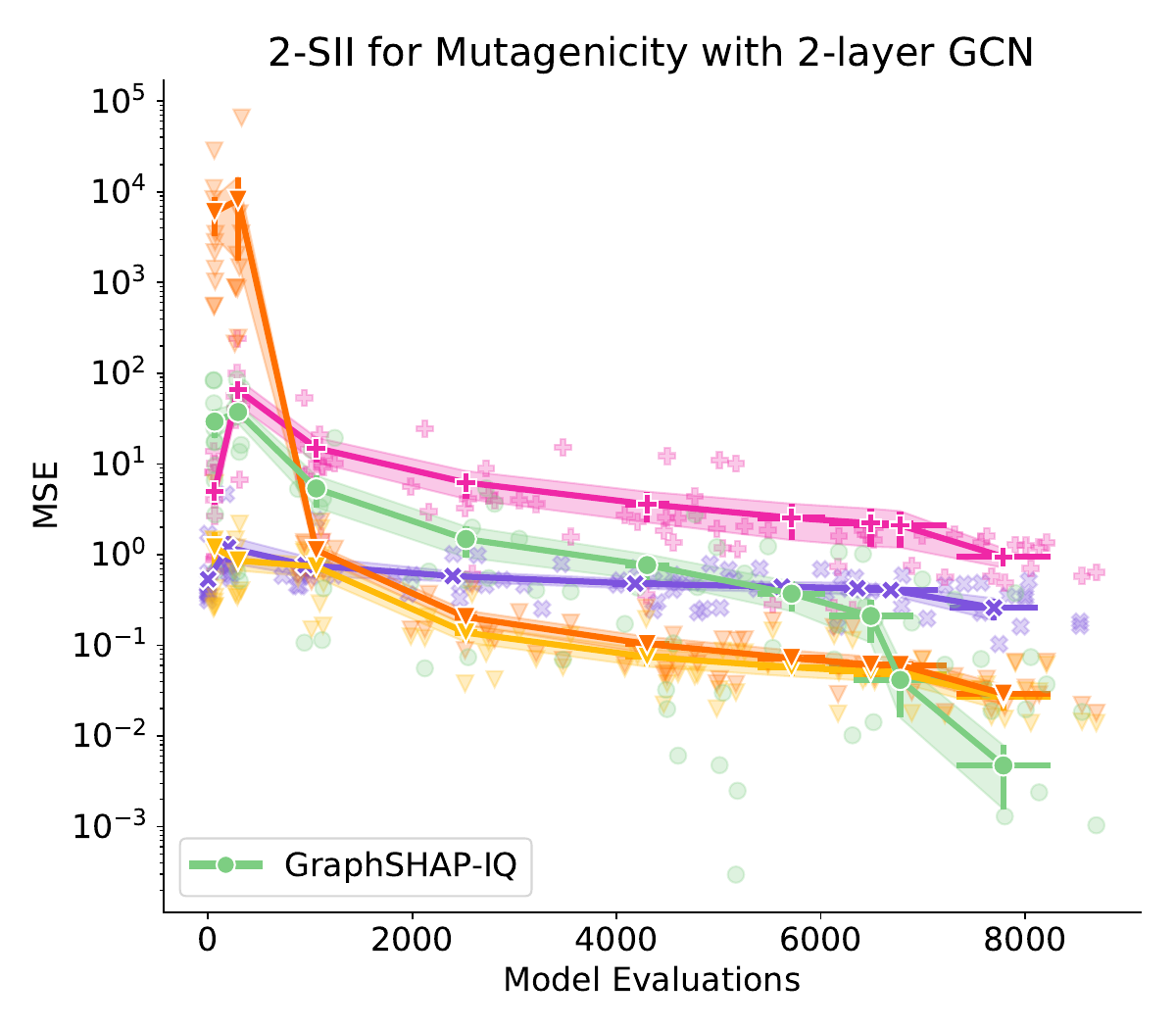}
        \includegraphics[width=\textwidth]{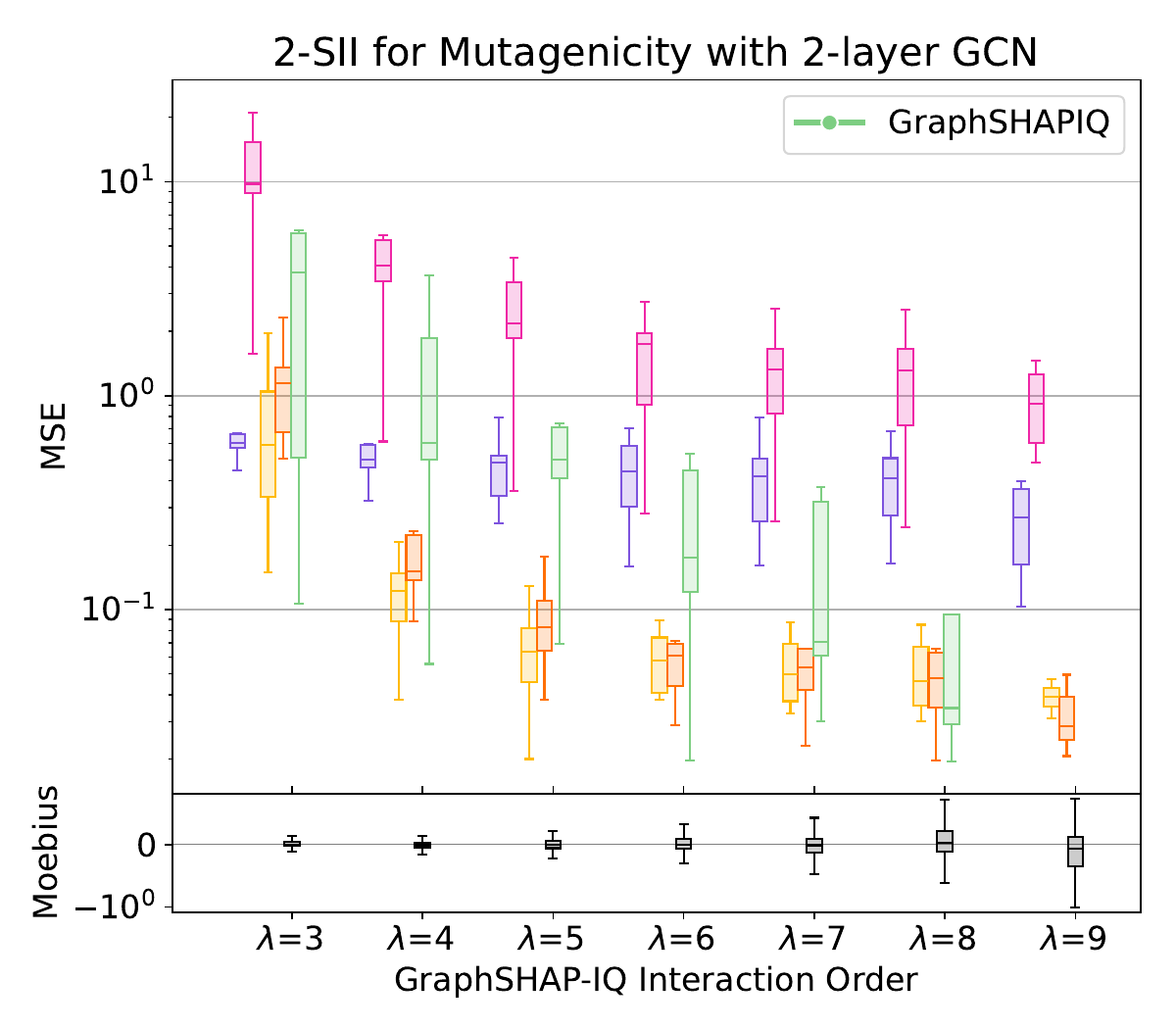}
    \end{minipage}
    \hfill
    \begin{minipage}[c]{0.325\textwidth}
    \centering
        \includegraphics[width=\textwidth]{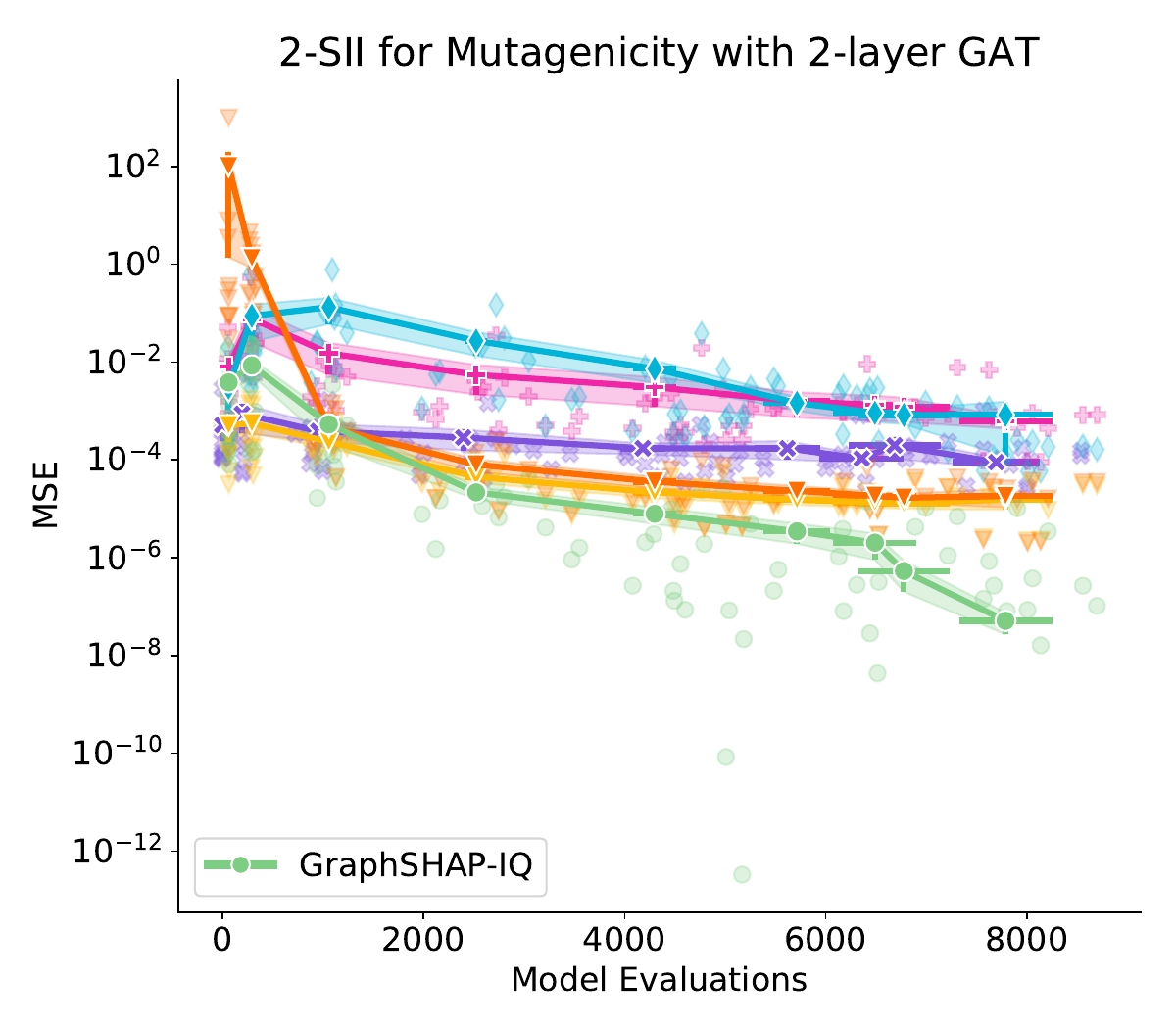}
        \includegraphics[width=\textwidth]{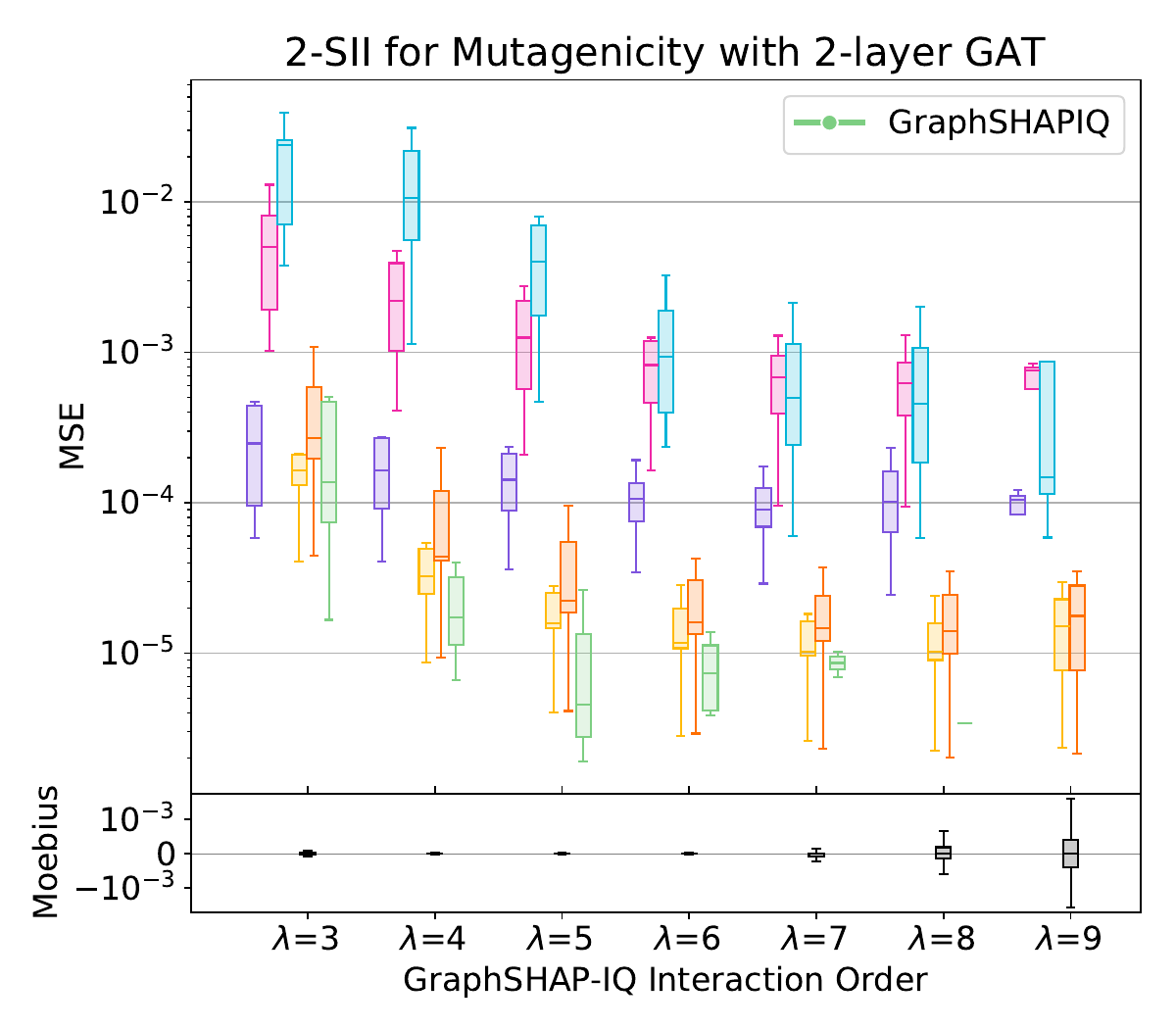}
    \end{minipage}
    \hfill
    \begin{minipage}[c]{0.325\textwidth}
    \centering
        \includegraphics[width=\textwidth]{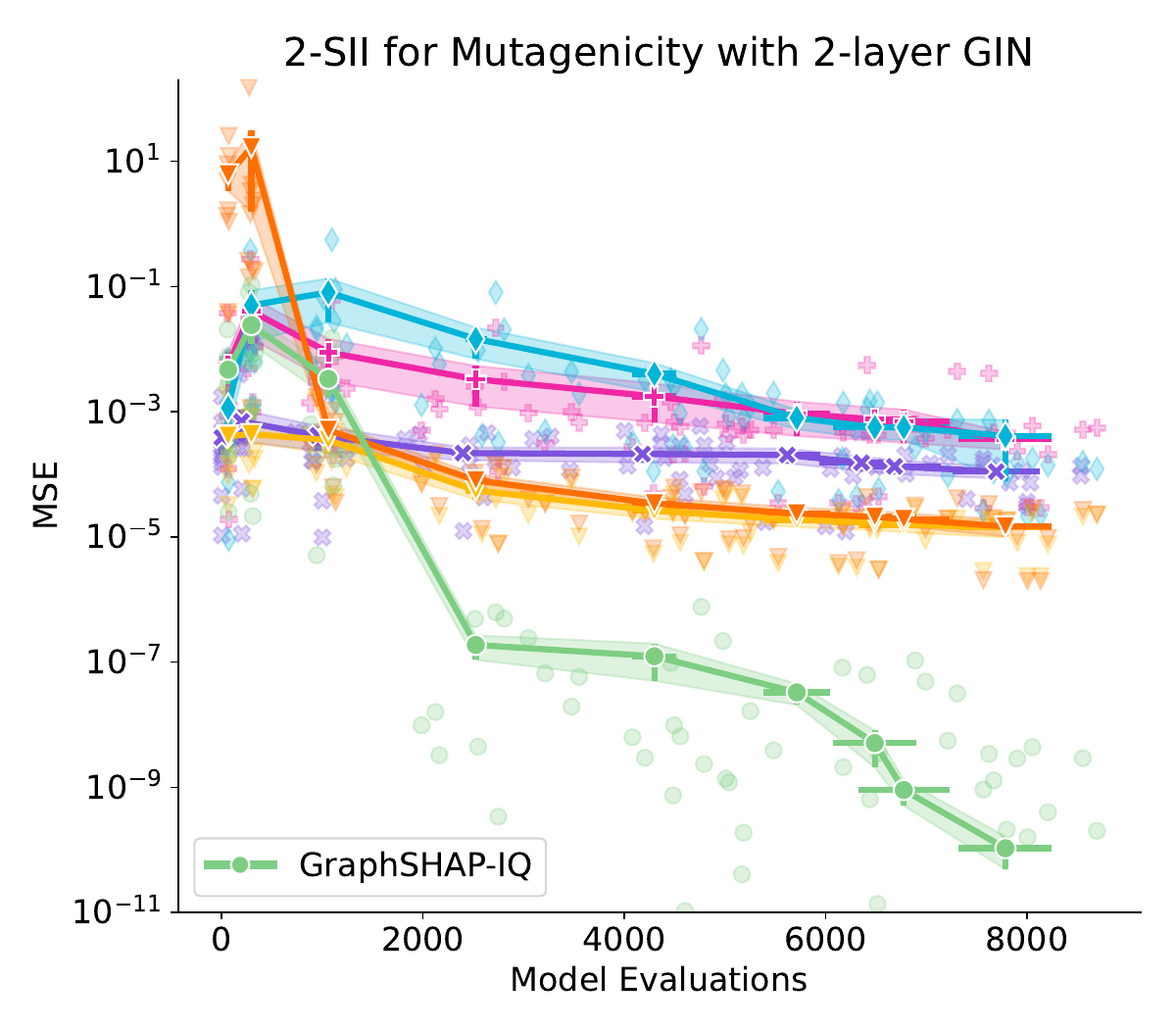}
        \includegraphics[width=\textwidth]{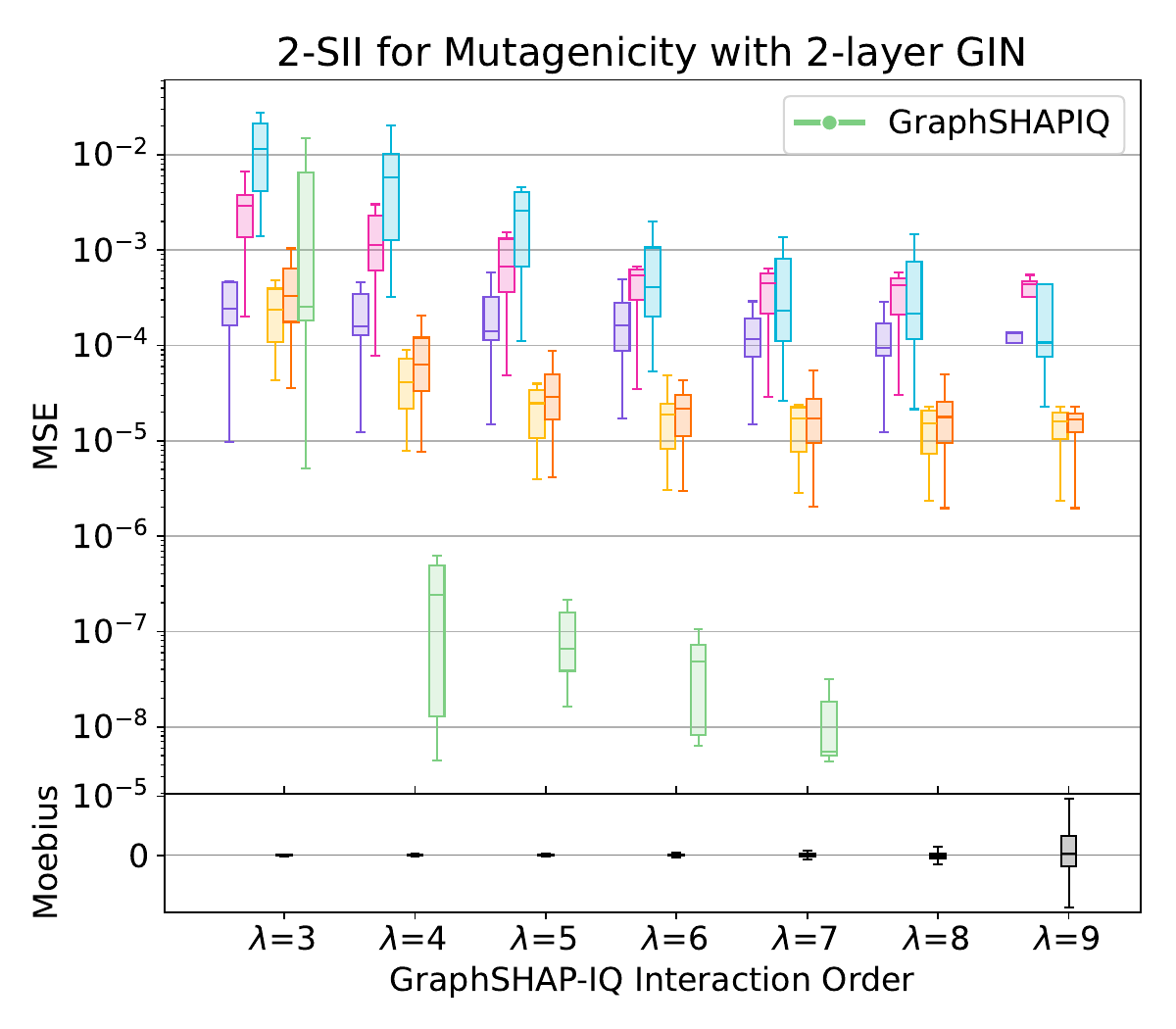}
    \end{minipage}
    \caption{Comparison of GraphSHAP-IQ's approximation quality with model-agnostic baselines on 10 graphs with $30 \leq n \leq 40$ nodes from the \gls*{MTG} dataset for a 2-layers \gls*{GCN} (left), \gls*{GAT} (middle) and \gls*{GIN} (right). The top row presents the MSE for each estimation (dots) and averaged over $\lambda$ (line) with the standard error of the mean (confidence band). The bottom row shows the same information including the \glspl*{MI} as box plots for each $\lambda$.}
    \label{appx_fig_approx_mtg}
\end{figure}

\begin{figure}
    \centering
    \includegraphics[width=0.8\textwidth]{figures/experiments/legend_sii.pdf}
    \begin{minipage}[c]{0.325\textwidth}
    \centering
        \includegraphics[width=\textwidth]{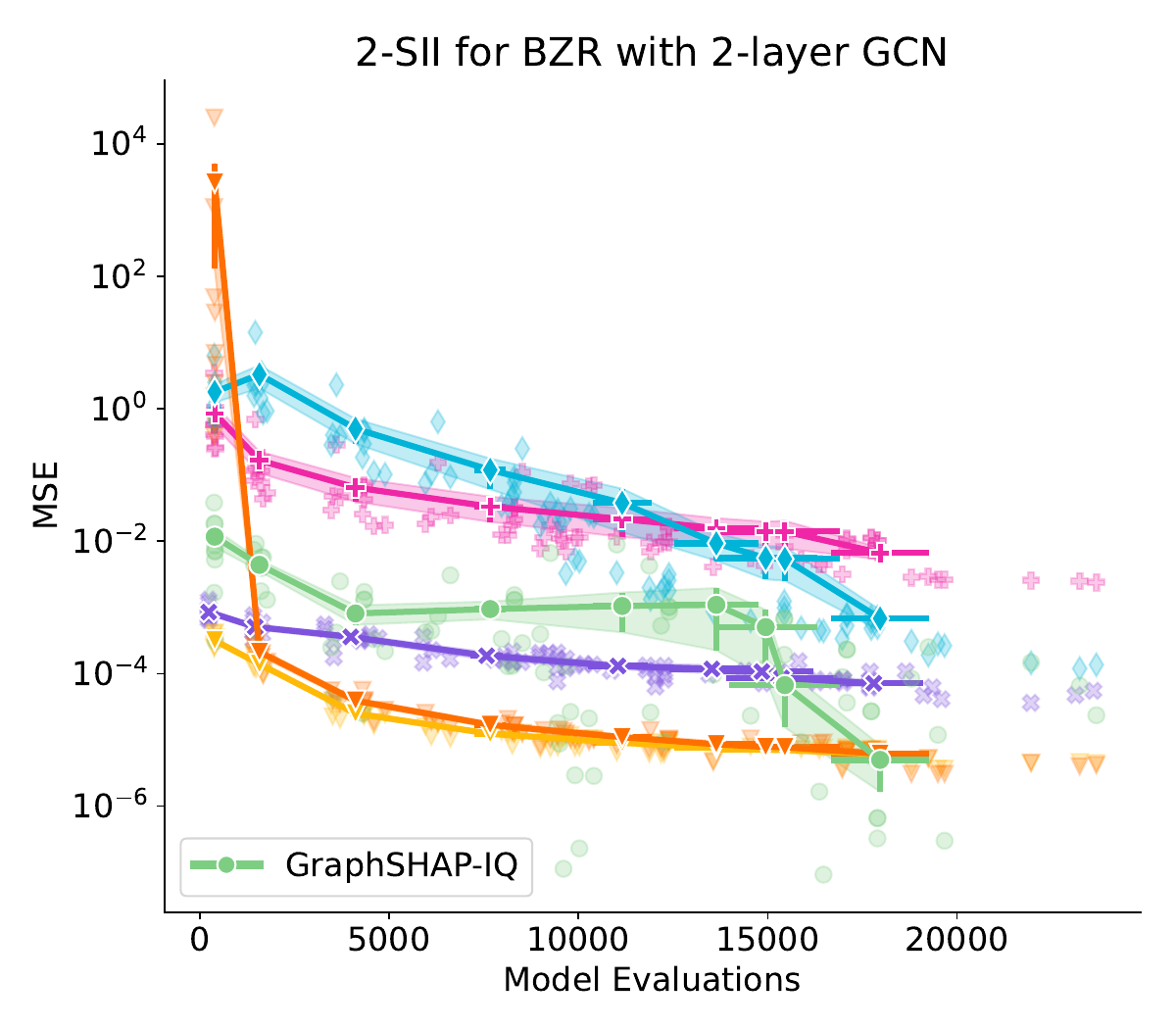}
        \includegraphics[width=\textwidth]{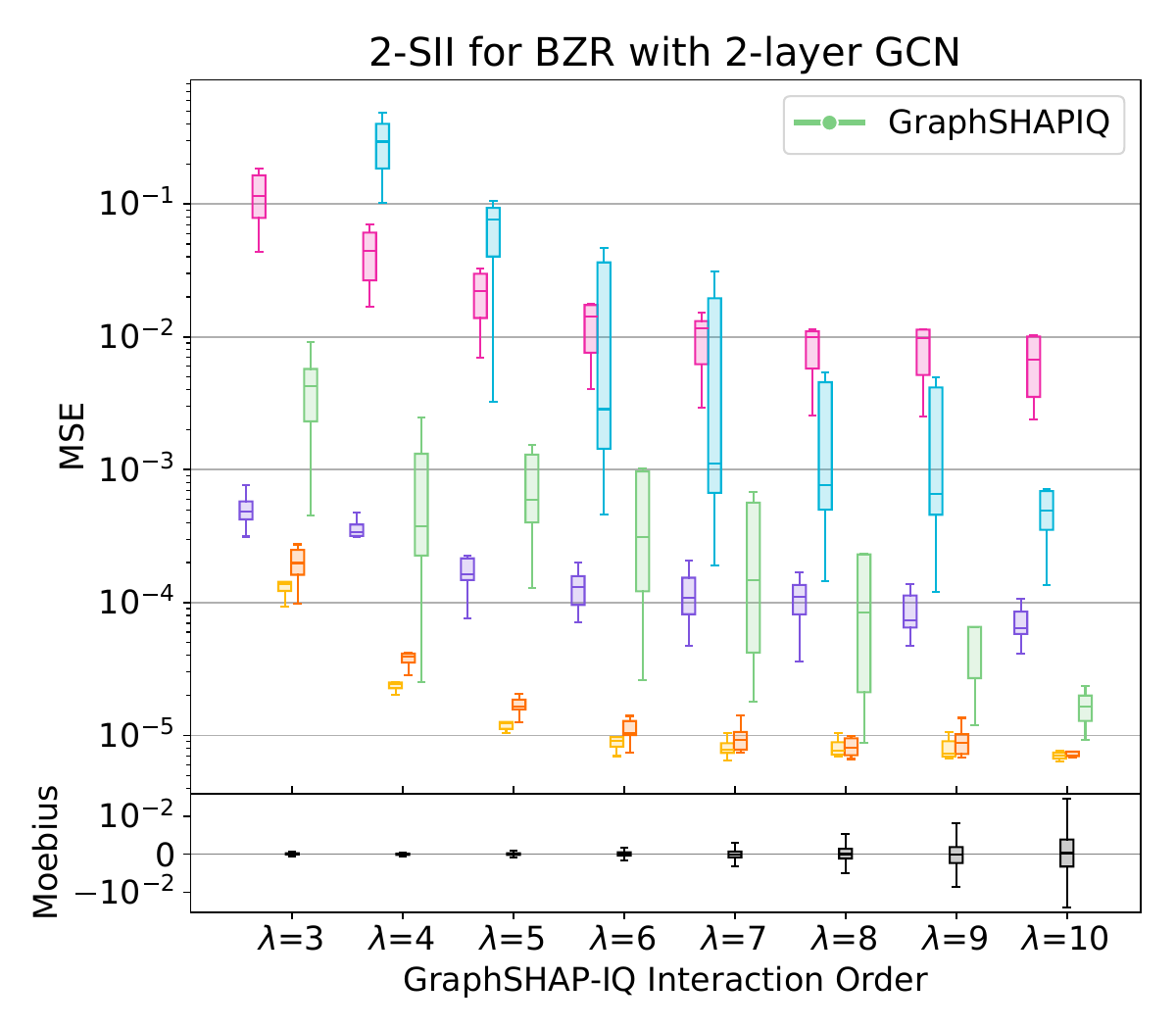}
    \end{minipage}
    \hfill
    \begin{minipage}[c]{0.325\textwidth}
    \centering
        \includegraphics[width=\textwidth]{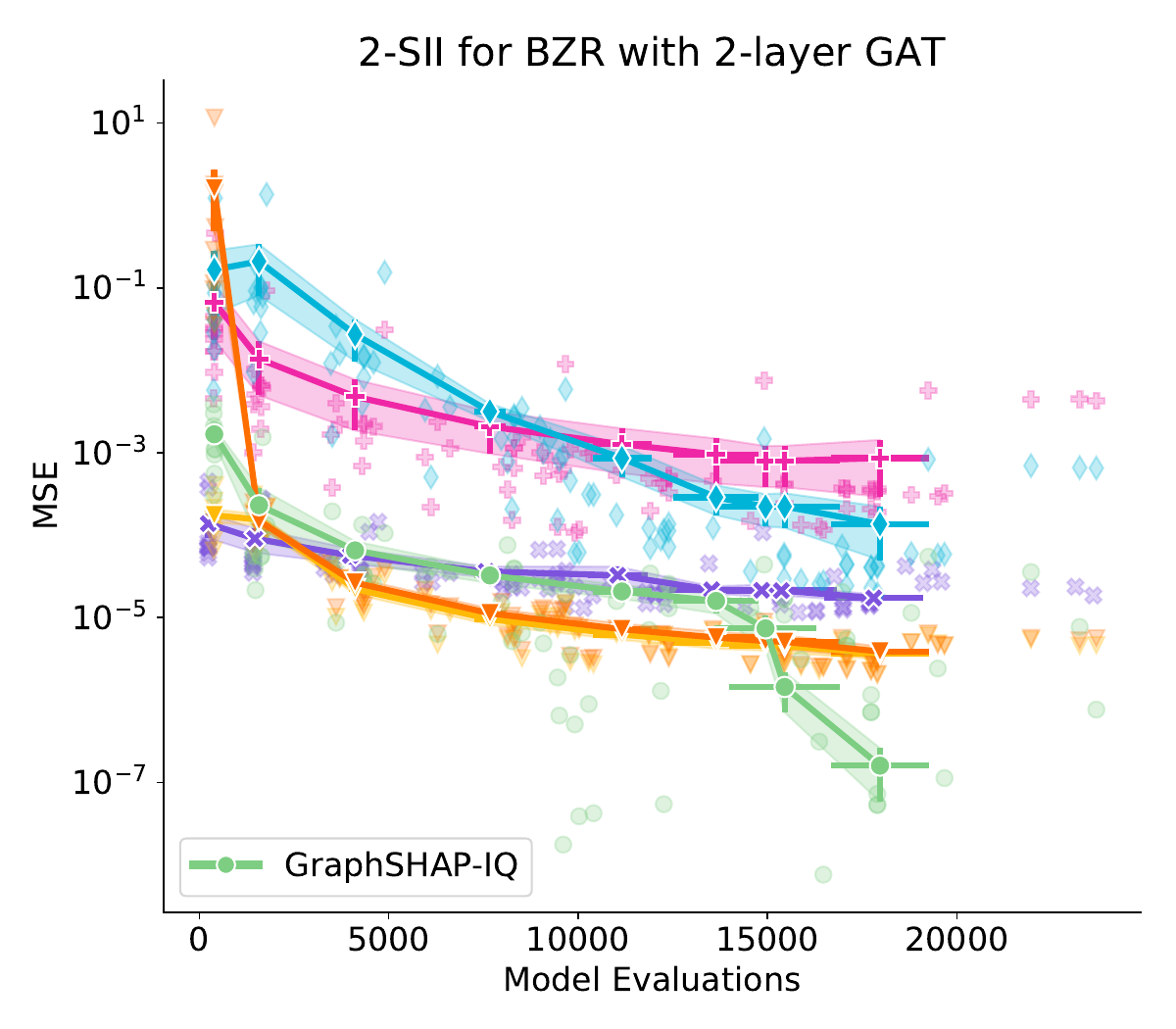}
        \includegraphics[width=\textwidth]{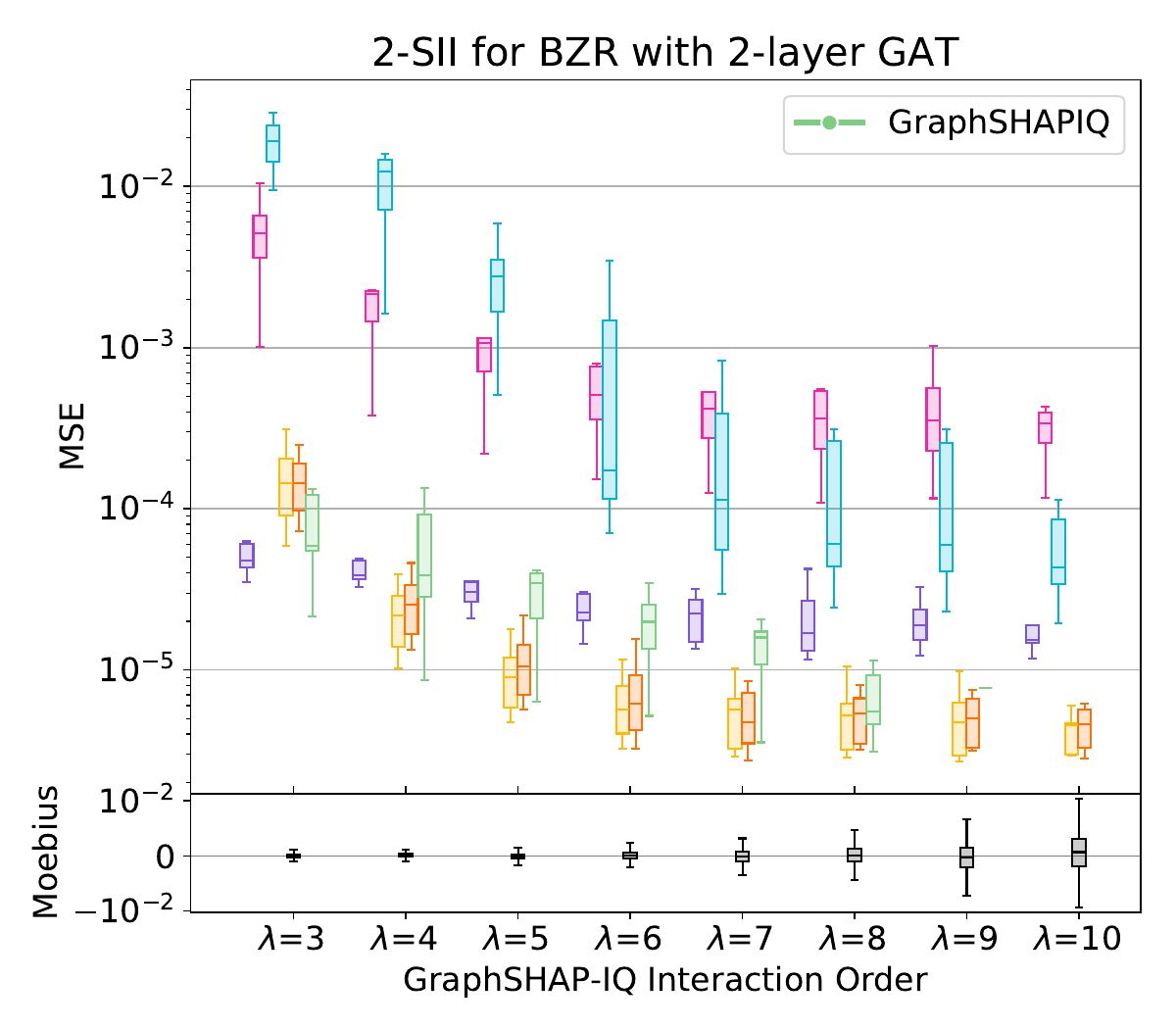}
    \end{minipage}
    \hfill
    \begin{minipage}[c]{0.325\textwidth}
    \centering
        \includegraphics[width=\textwidth]{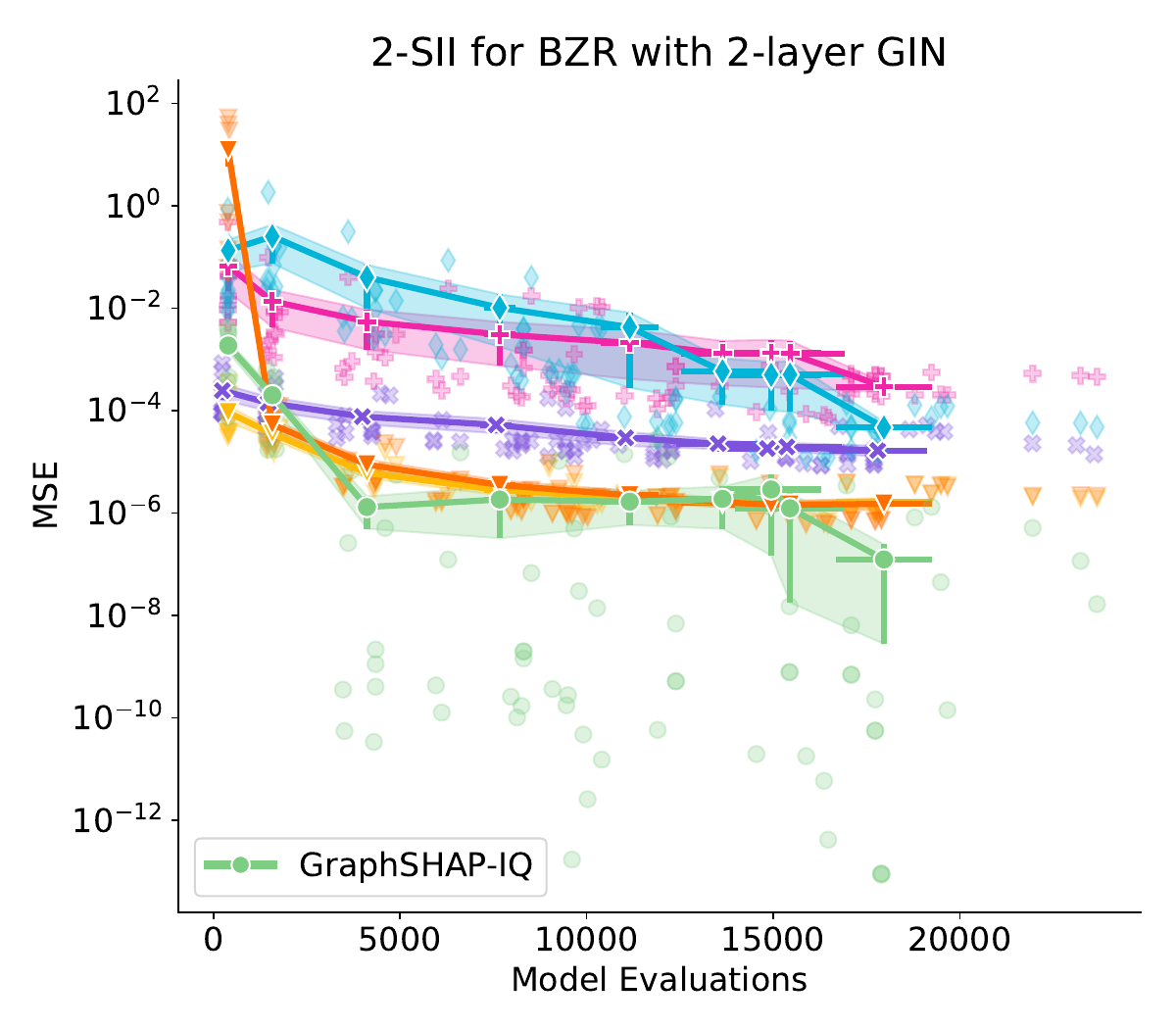}
        \includegraphics[width=\textwidth]{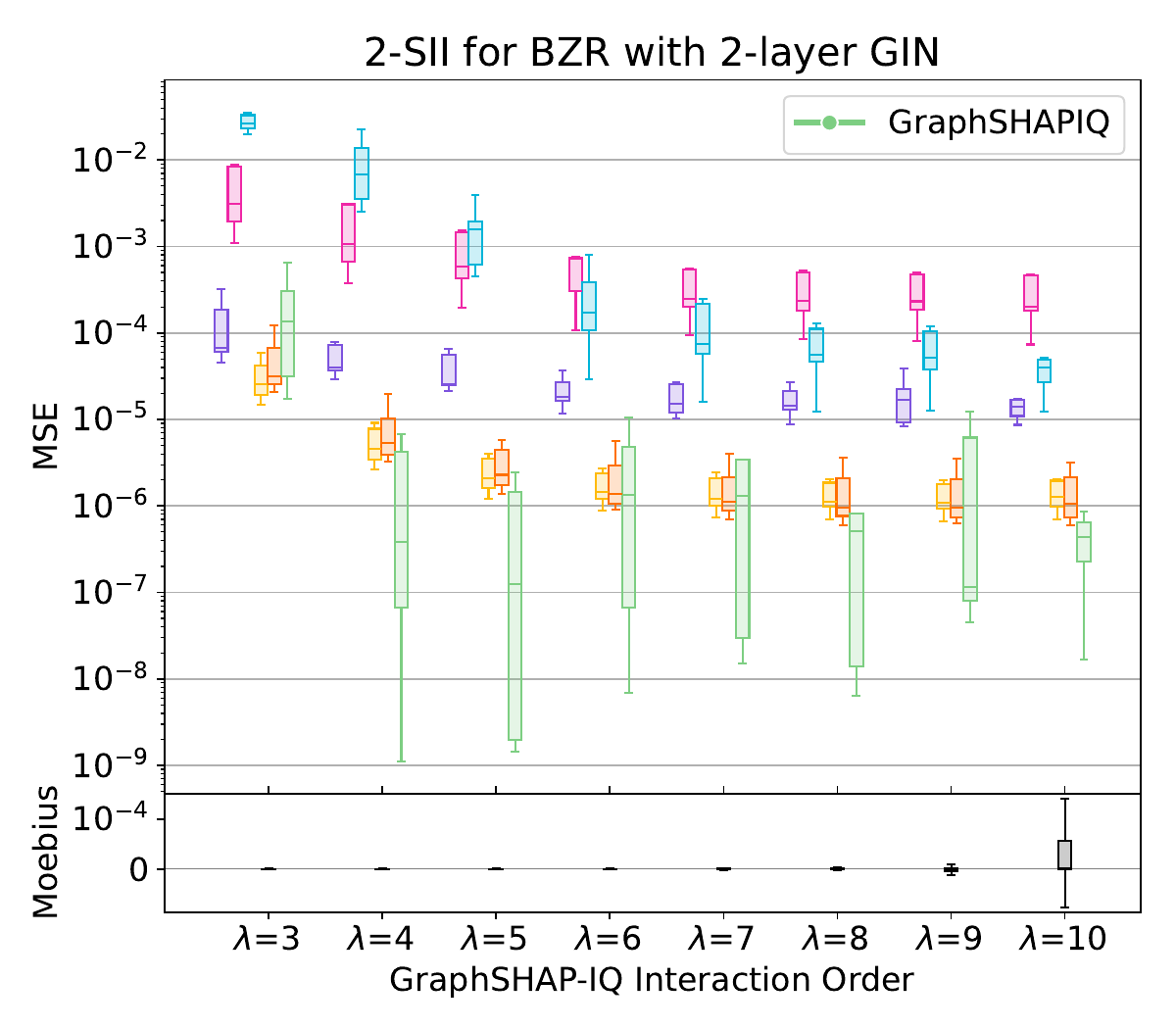}
    \end{minipage}
    \caption{Comparison of GraphSHAP-IQ's approximation quality with model-agnostic baselines on 10 graphs with $30 \leq n \leq 40$ nodes from the \gls*{BZR} dataset for a 2-layers \gls*{GCN} (left), \gls*{GAT} (middle) and \gls*{GIN} (right). The top row presents the MSE for each estimation (dots) and averaged over $\lambda$ (line) with the standard error of the mean (confidence band). The bottom row shows the same information including the \glspl*{MI} as box plots for each $\lambda$.}
    \label{appx_fig_approx_bzr}
\end{figure}

\begin{figure}
    \centering
    \includegraphics[width=0.8\textwidth]{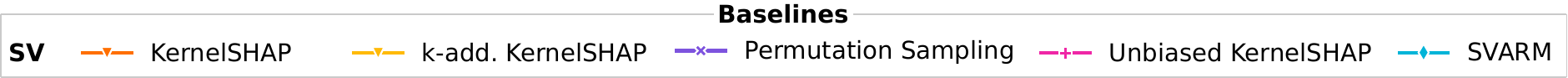}
    \begin{minipage}[c]{0.325\textwidth}
    \centering
        \includegraphics[width=\textwidth]{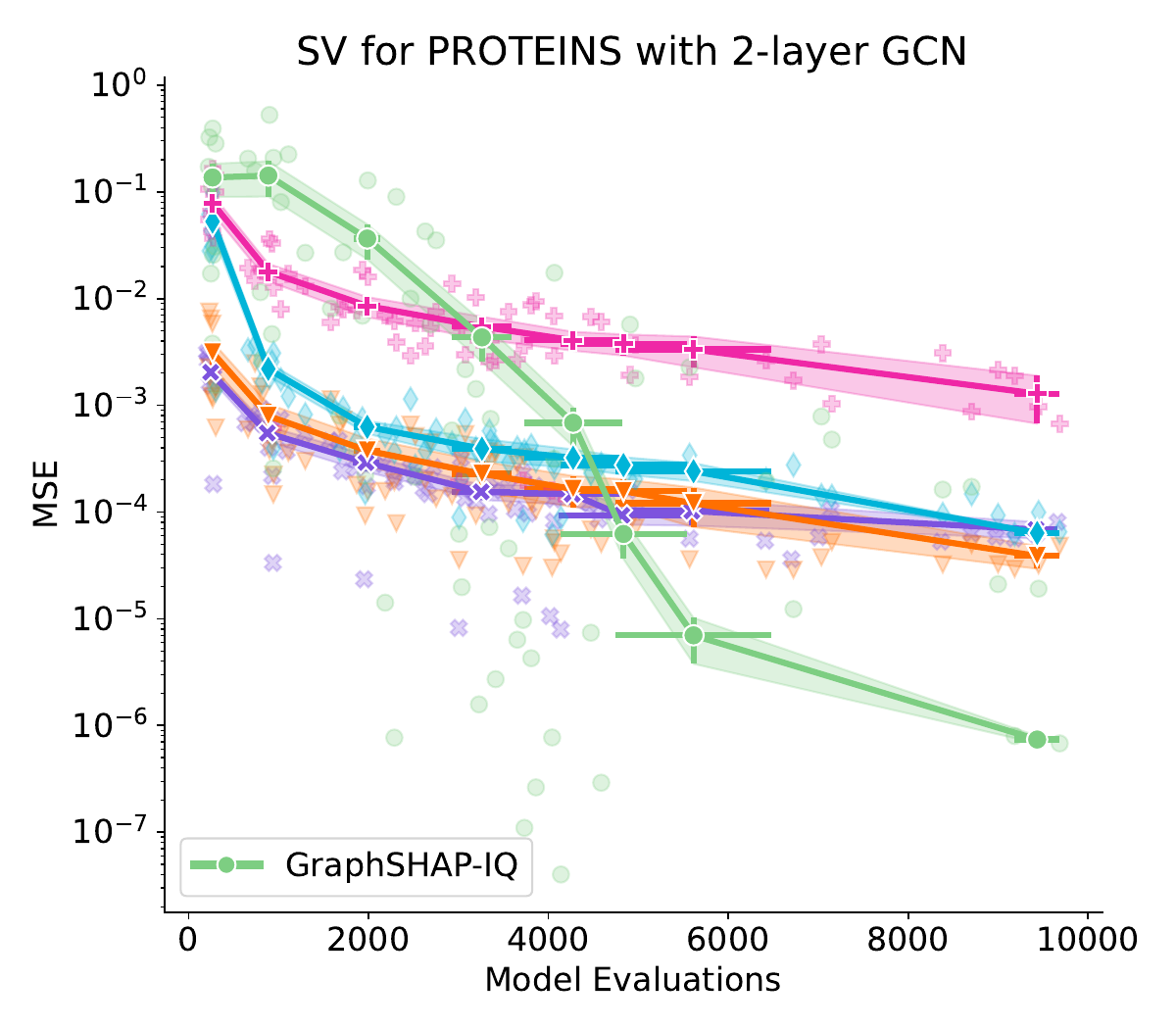}
        \includegraphics[width=\textwidth]{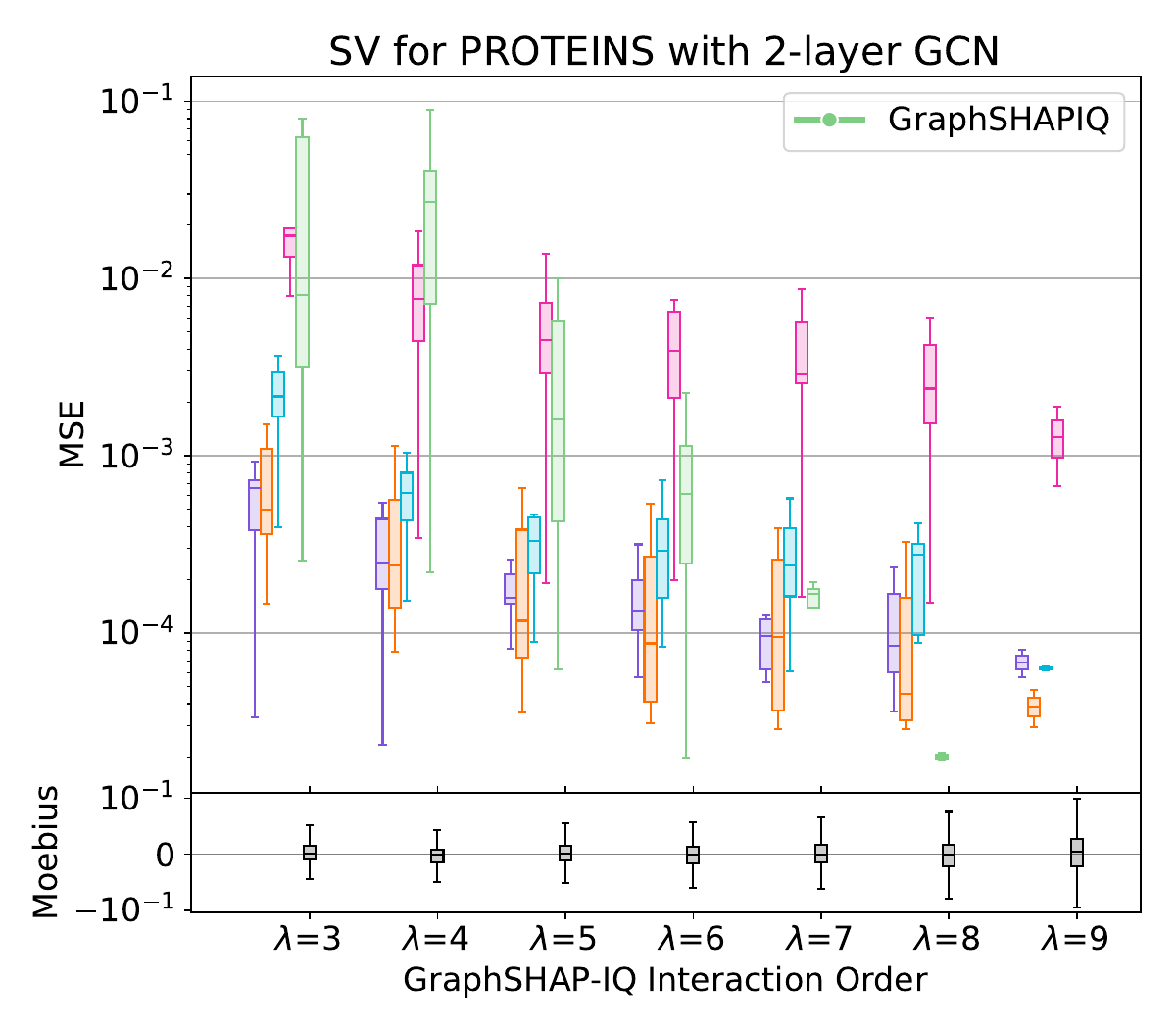}
    \end{minipage}
    \hfill
    \begin{minipage}[c]{0.325\textwidth}
    \centering
        \includegraphics[width=\textwidth]{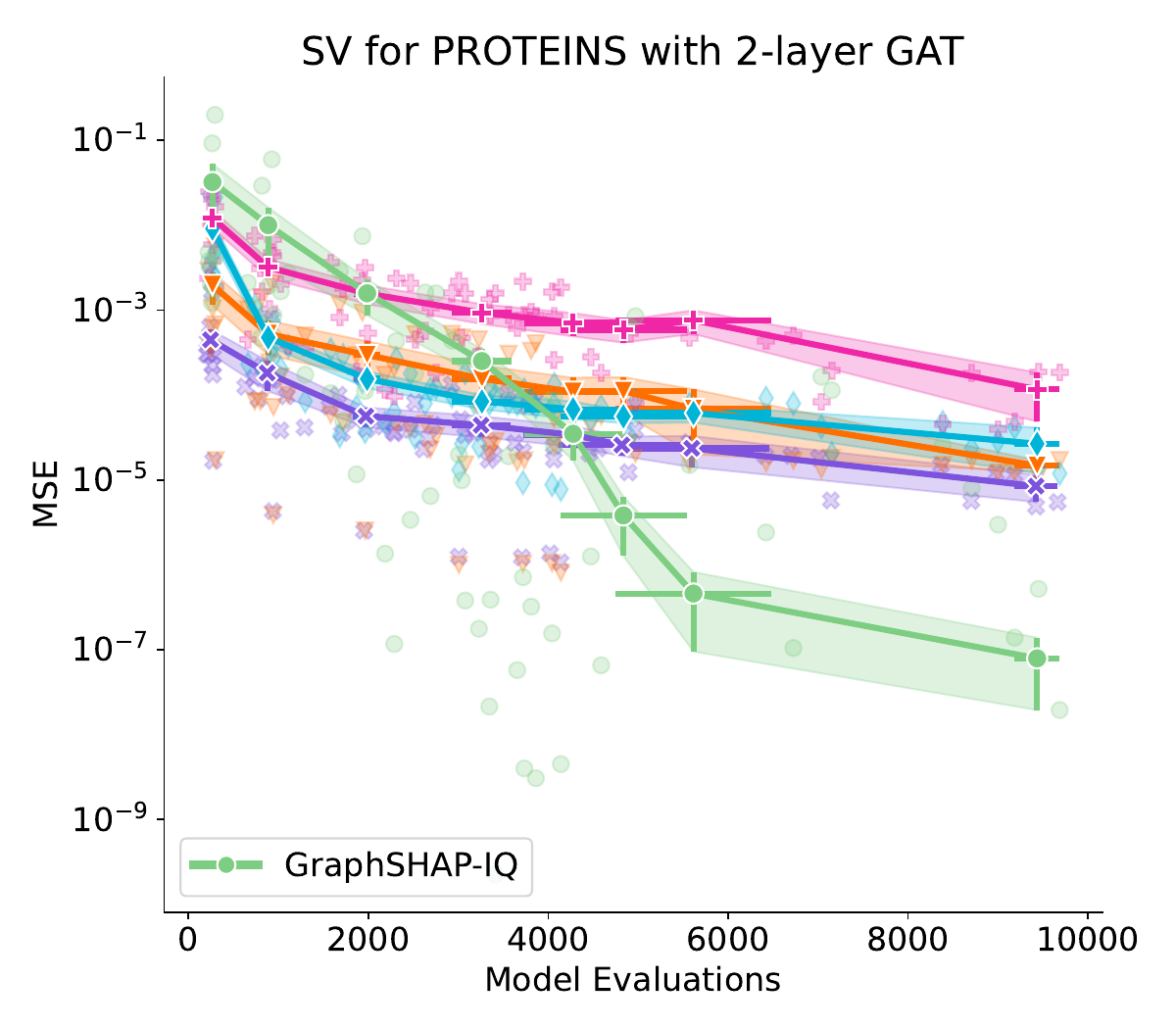}
        \includegraphics[width=\textwidth]{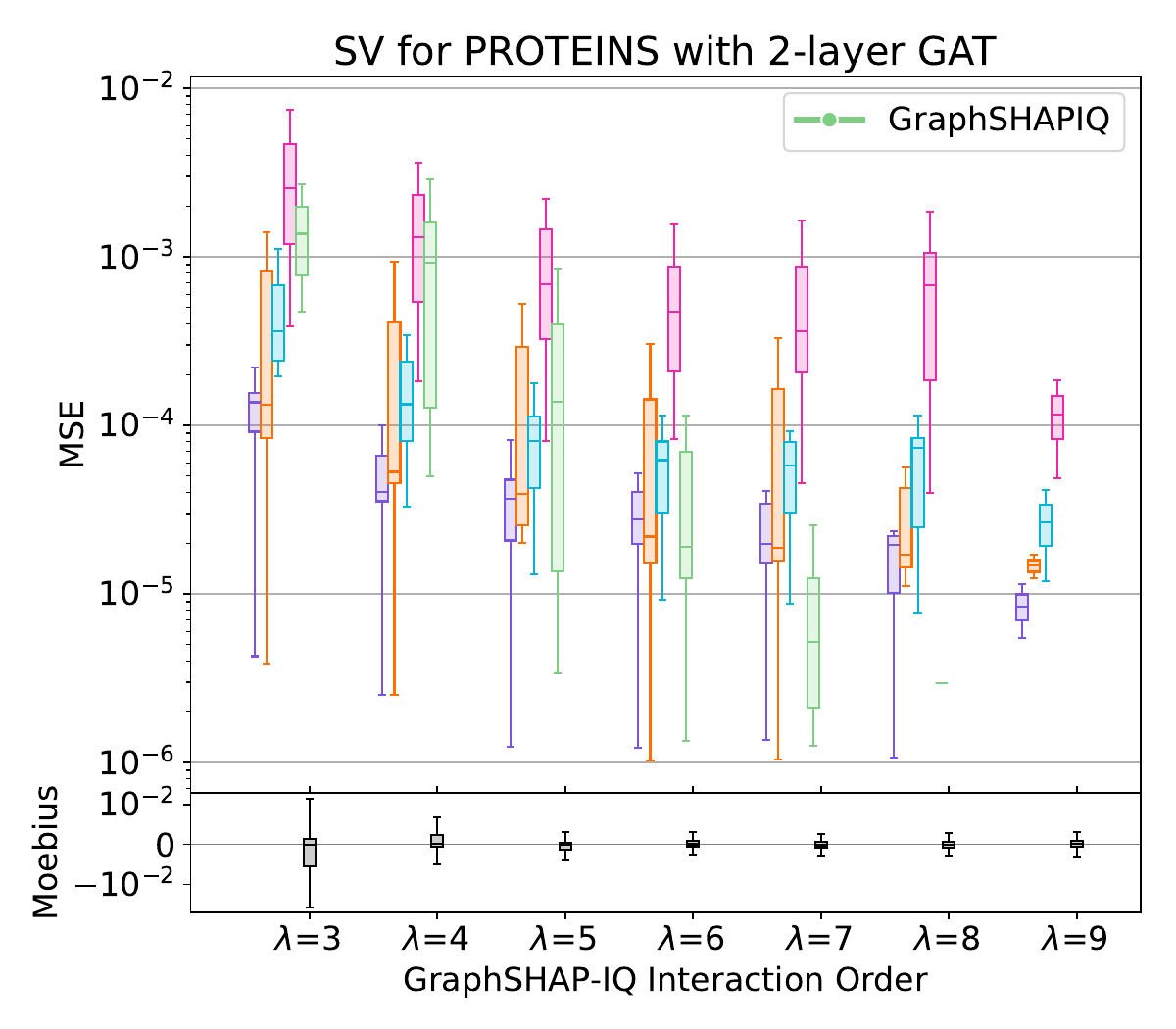}
    \end{minipage}
    \hfill
    \begin{minipage}[c]{0.325\textwidth}
    \centering
        \includegraphics[width=\textwidth]{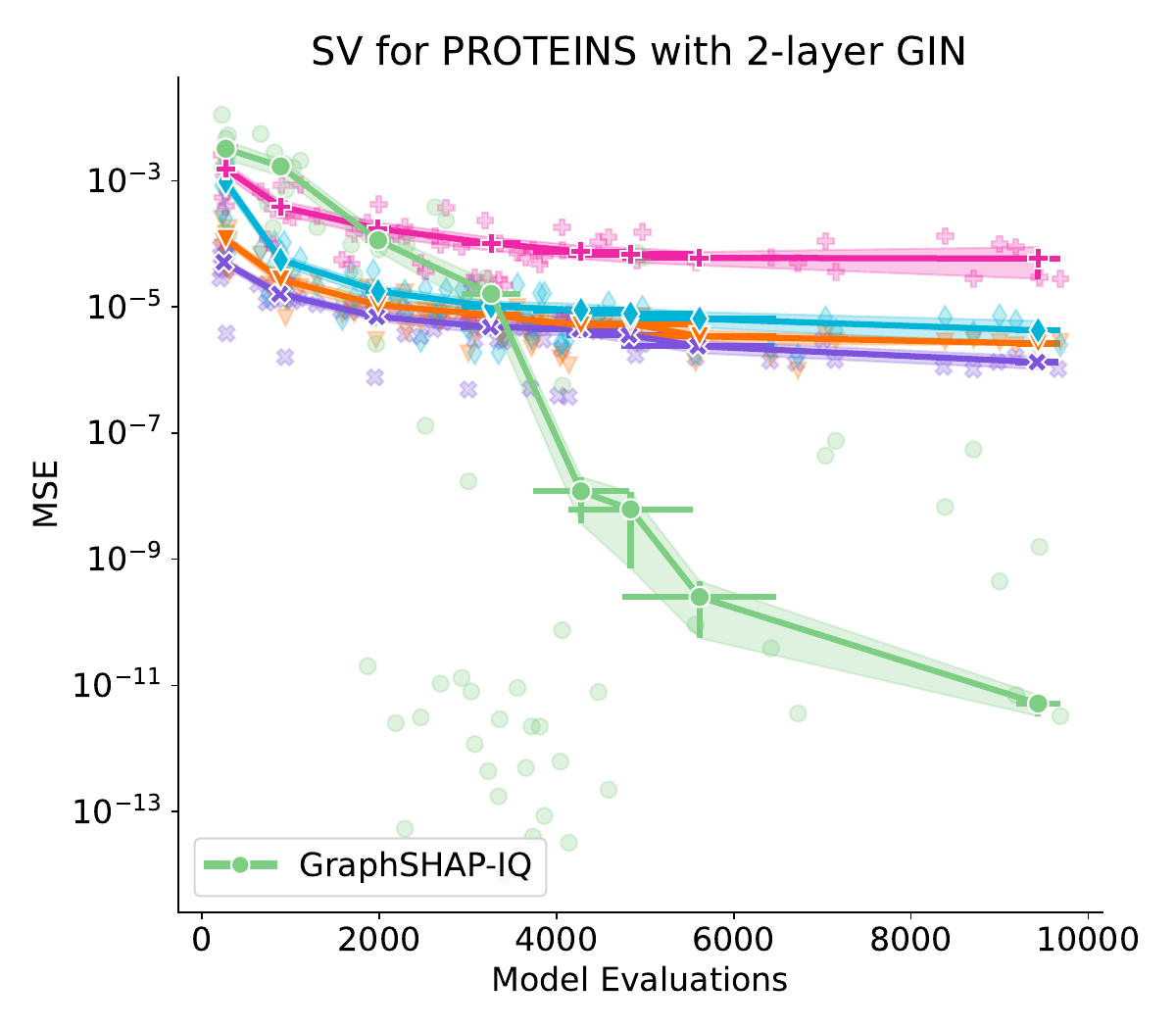}
        \includegraphics[width=\textwidth]{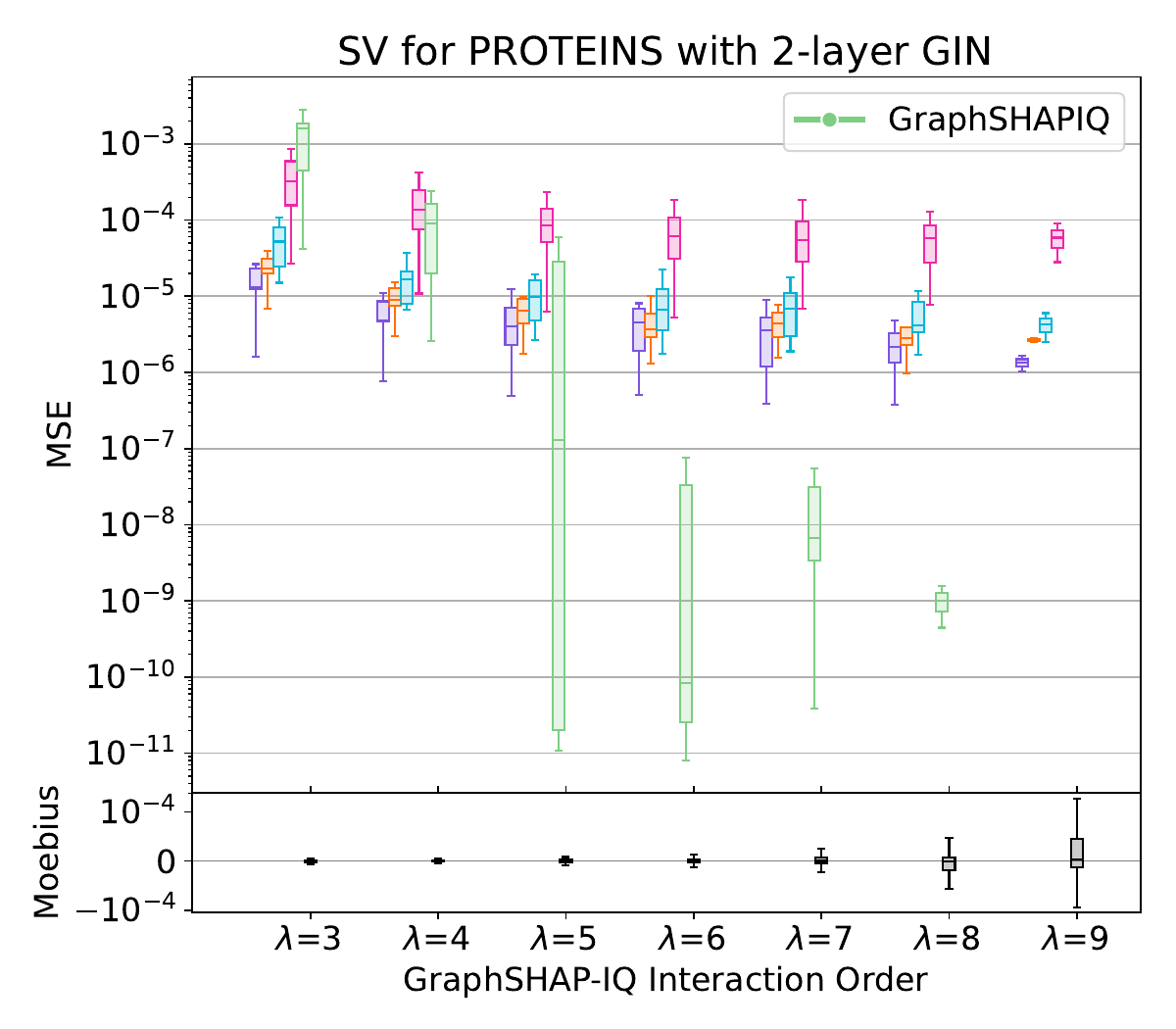}
    \end{minipage}
    \caption{Comparison of GraphSHAP-IQ's approximation quality with model-agnostic baselines on 10 graphs with $30 \leq n \leq 40$ nodes from the \gls*{PRT} dataset for a 2-layers \gls*{GCN} (left), \gls*{GAT} (middle) and \gls*{GIN} (right). The top row presents the MSE for each estimation (dots) and averaged over $\lambda$ (line) with the standard error of the mean (confidence band). The bottom row shows the same information including the \glspl*{MI} as box plots for each $\lambda$.}
    \label{appx_fig_approx_sv}
\end{figure}

\begin{figure}
    \centering
    \begin{minipage}[c]{0.2\textwidth}
    \textbf{SV-Graph:}
    \end{minipage}
    \hfill
    \begin{minipage}[c]{0.38\textwidth}
    \centering
        \includegraphics[width=\textwidth]{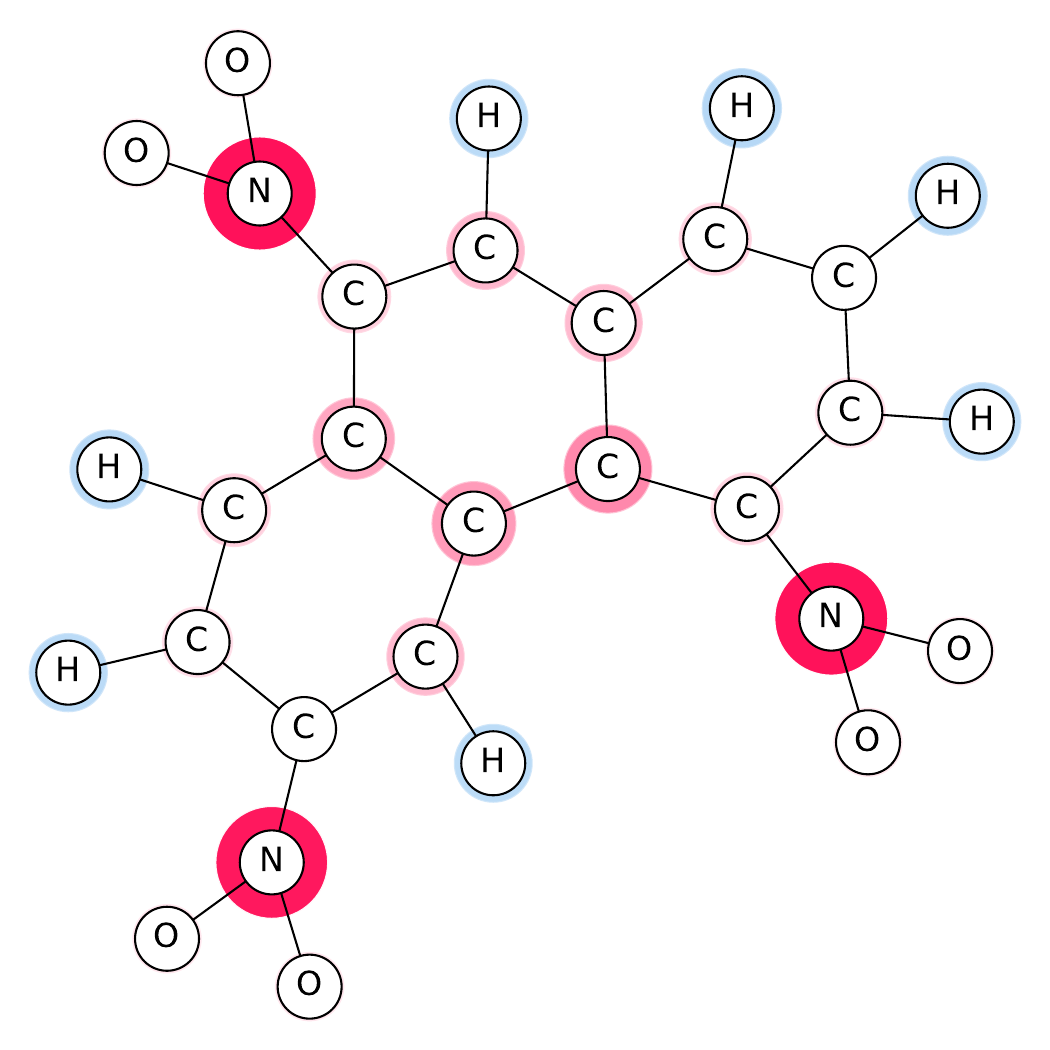}        
    \end{minipage}
    \hfill
    \begin{minipage}[c]{0.38\textwidth}
    \centering
        \includegraphics[width=\textwidth]{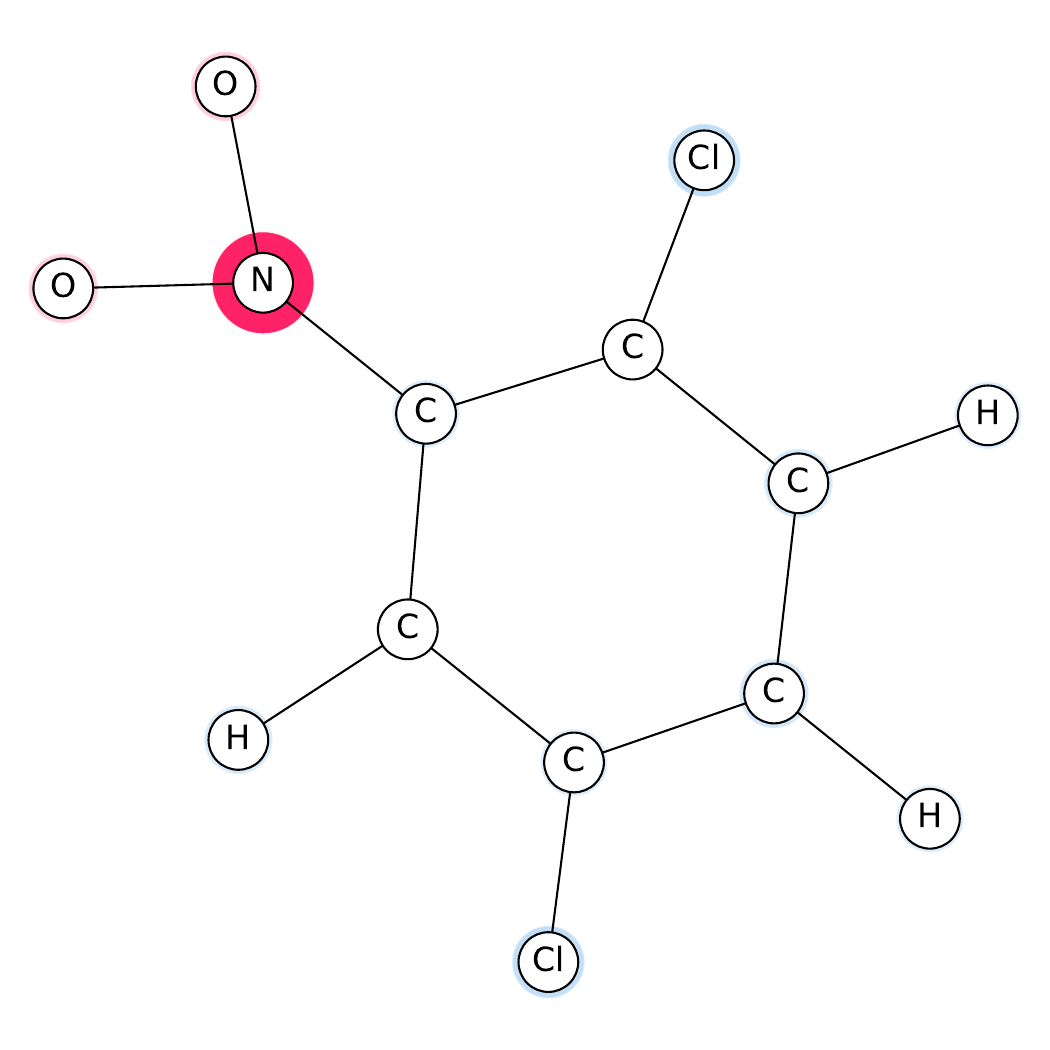}
    \end{minipage}
    \\
    \begin{minipage}[c]{0.2\textwidth}
    \textbf{$2$-SII-Graph:}
    \end{minipage}
    \hfill
    \begin{minipage}[c]{0.38\textwidth}
    \centering
        \includegraphics[width=\textwidth]{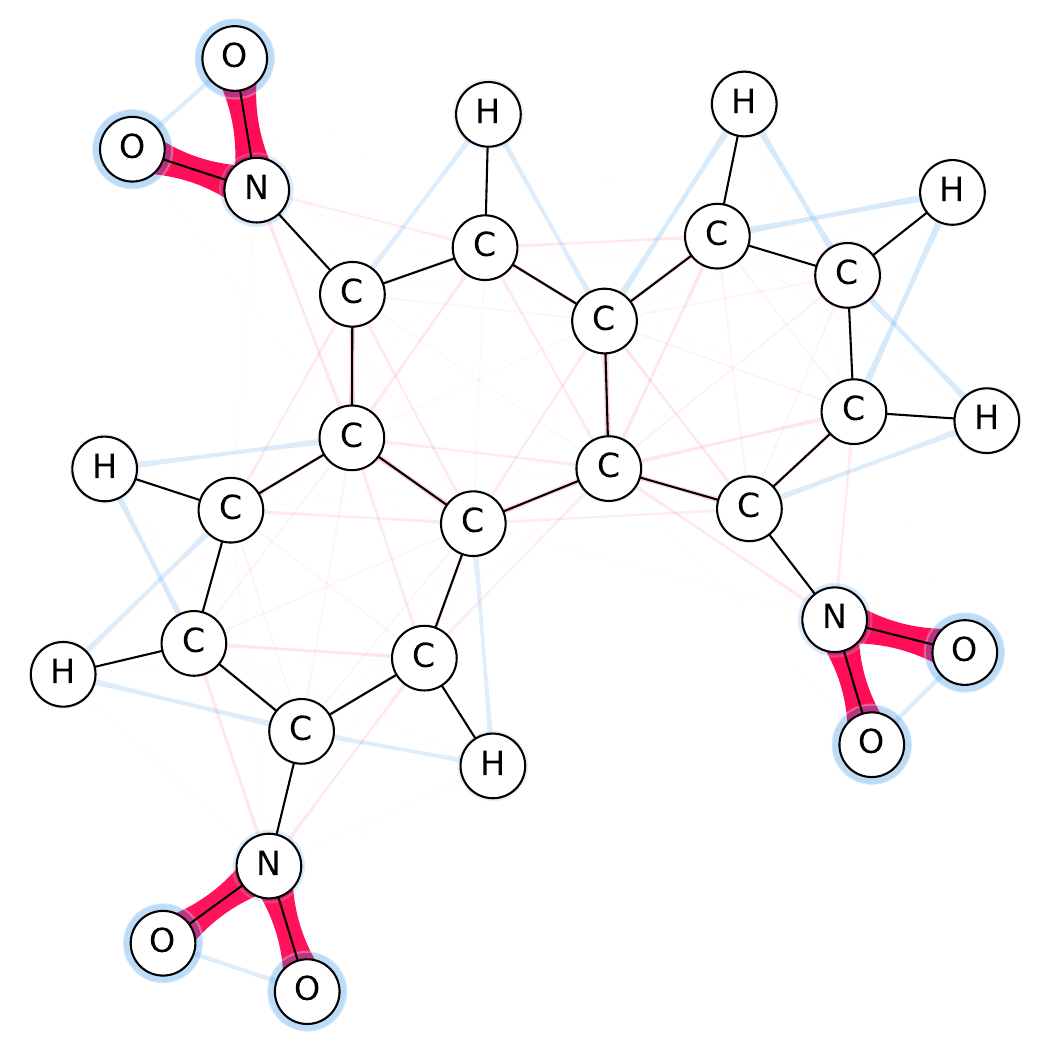}        
    \end{minipage}
    \hfill
    \begin{minipage}[c]{0.38\textwidth}
    \centering
        \includegraphics[width=\textwidth]{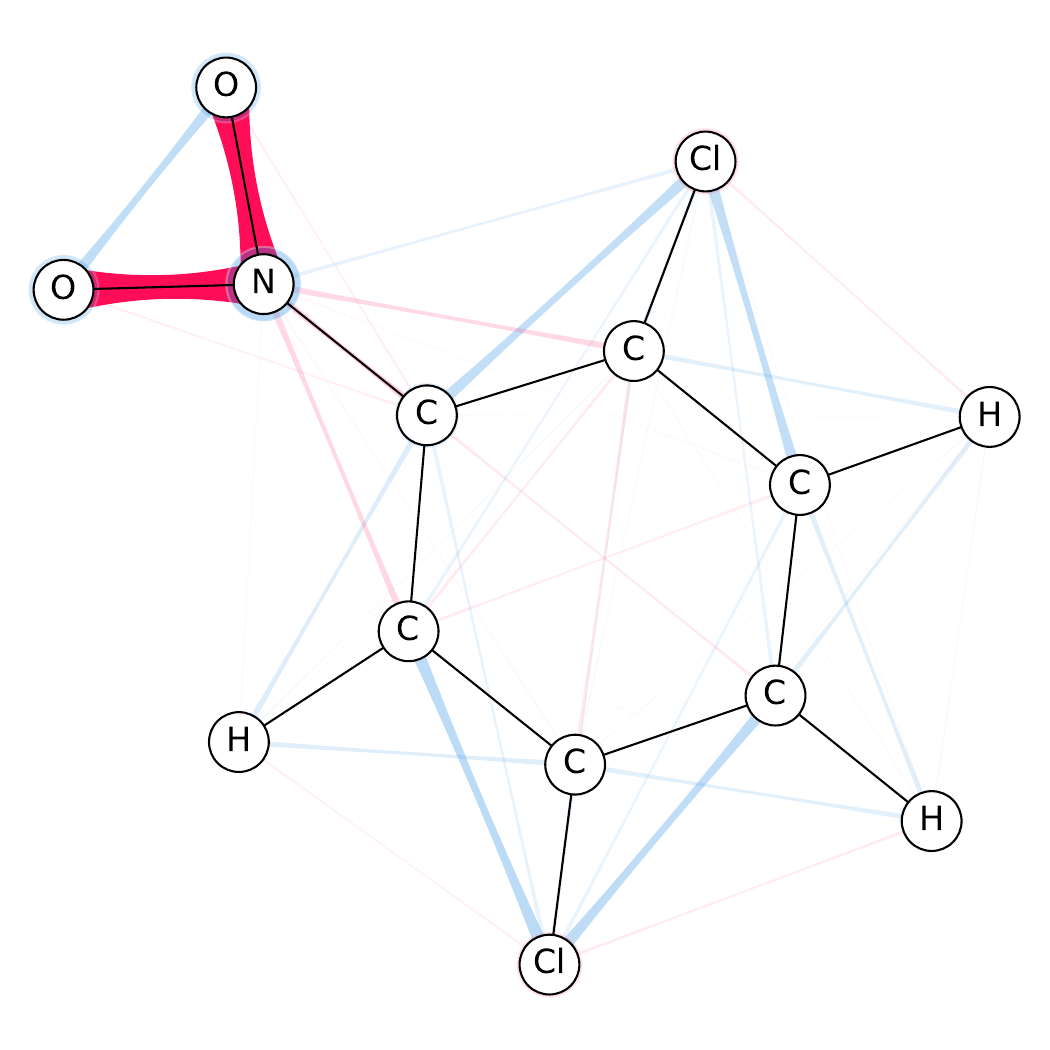}
    \end{minipage}
    \\
    \begin{minipage}[c]{0.2\textwidth}
    \textbf{3-SII-Graph:}
    \end{minipage}
    \hfill
    \begin{minipage}[c]{0.38\textwidth}
    \centering
        \includegraphics[width=\textwidth]{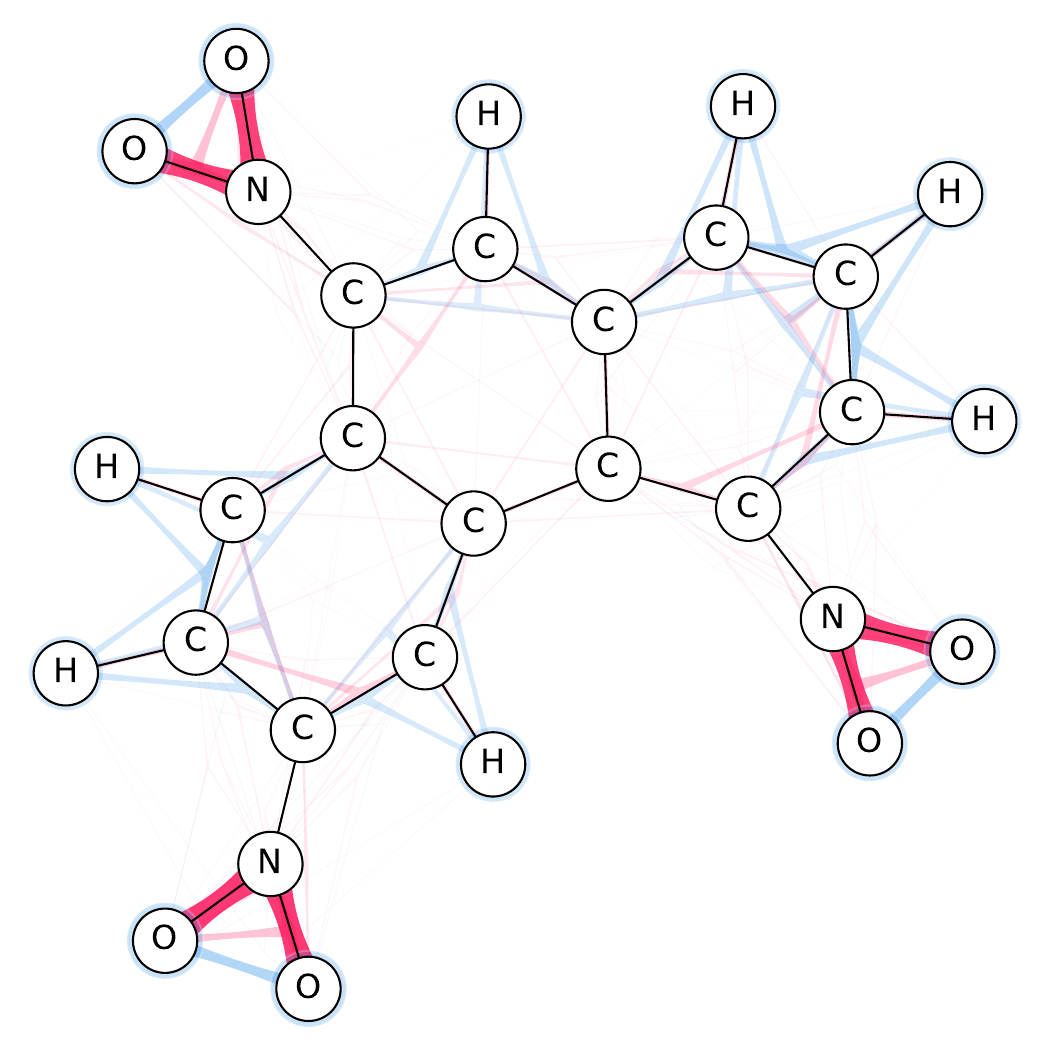}        
    \end{minipage}
    \hfill
    \begin{minipage}[c]{0.38\textwidth}
    \centering
        \includegraphics[width=\textwidth]{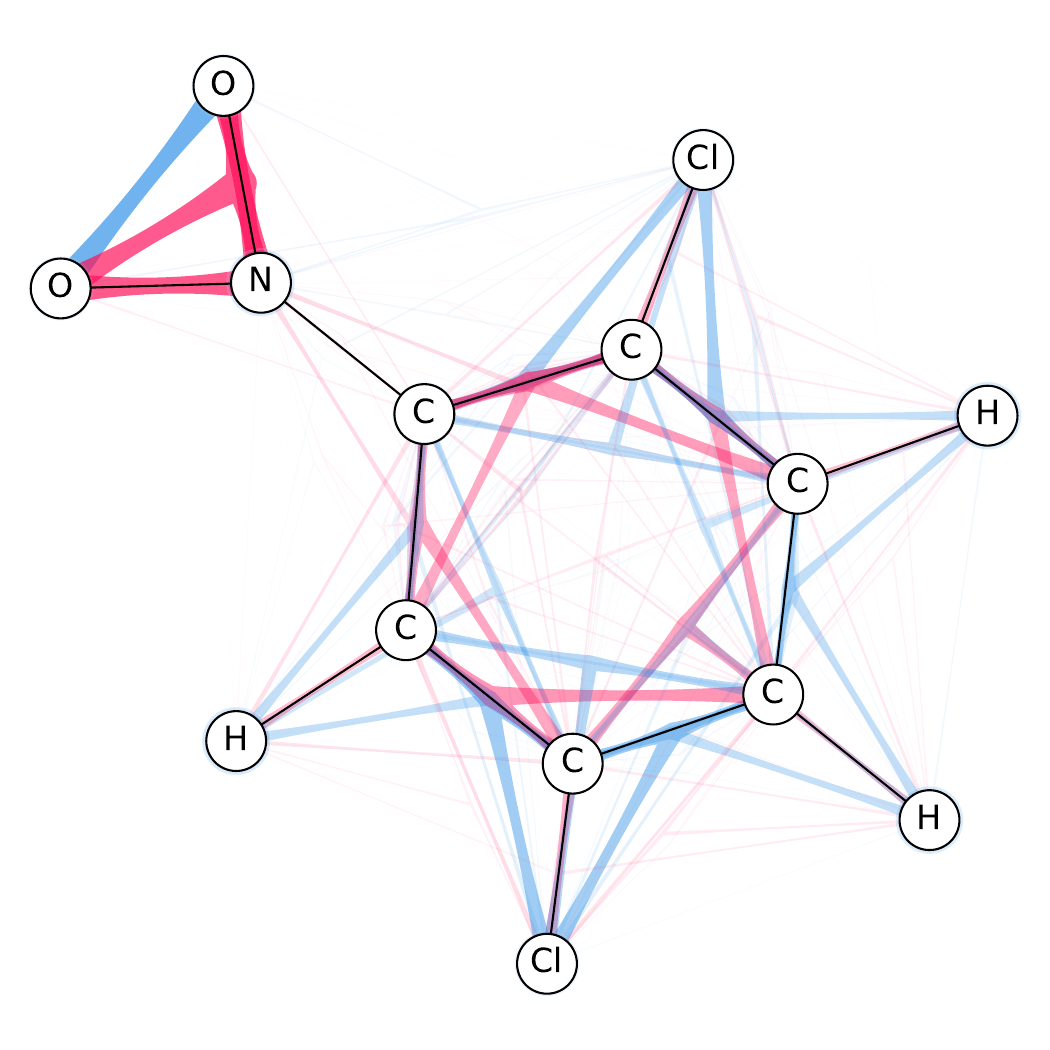}
    \end{minipage}
    \\
    \begin{minipage}[c]{0.2\textwidth}
    \textbf{MI-Graph:}
    \end{minipage}
    \hfill
    \begin{minipage}[c]{0.38\textwidth}
    \centering
        \includegraphics[width=\textwidth]{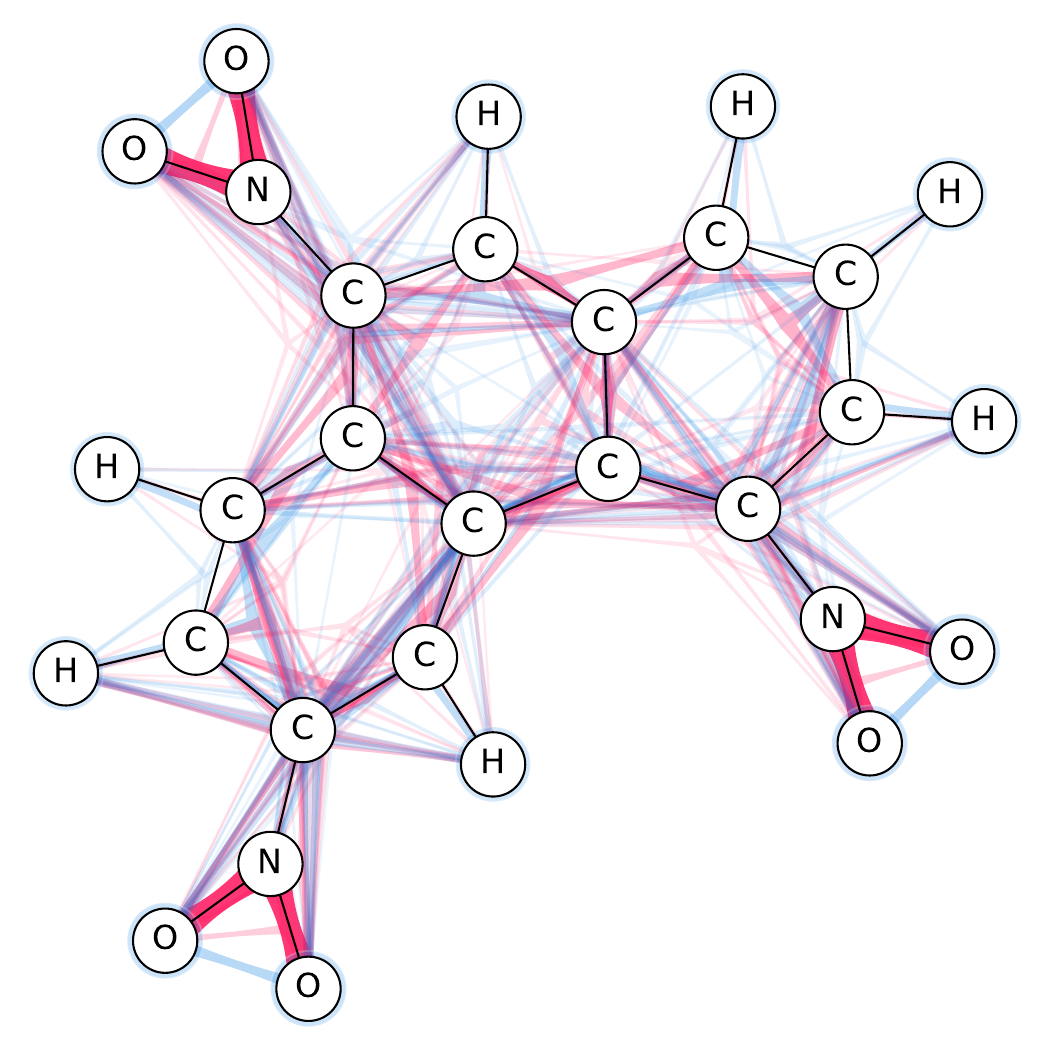}        
    \end{minipage}
    \hfill
    \begin{minipage}[c]{0.38\textwidth}
    \centering
        \includegraphics[width=\textwidth]{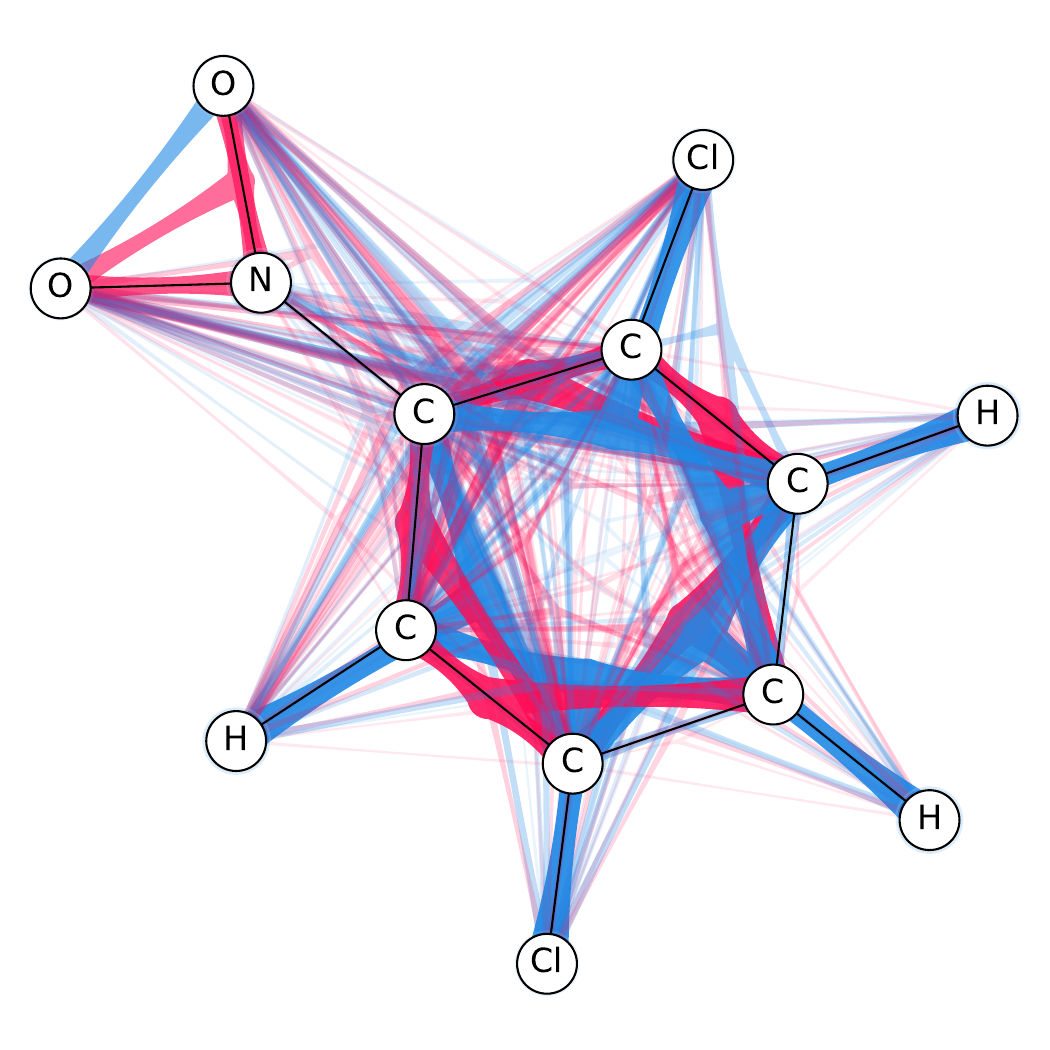}
    \end{minipage}
    \caption{Additional \gls*{SI}-Graphs for molecule structures of the \emph{MTG} dataset; molecule \emph{71} with $30$ atoms (left) and molecule \emph{189} with $14$ atoms (right). The model is a 2-layer GCN.}
    \label{appx_fig_mutag_graphs}
\end{figure}

\begin{figure}
    \centering
    \begin{minipage}[c]{0.2\textwidth}
    \textbf{SV-Graph:}
    \end{minipage}
    \hfill
    \begin{minipage}[c]{0.38\textwidth}
    \centering
        \includegraphics[width=\textwidth]{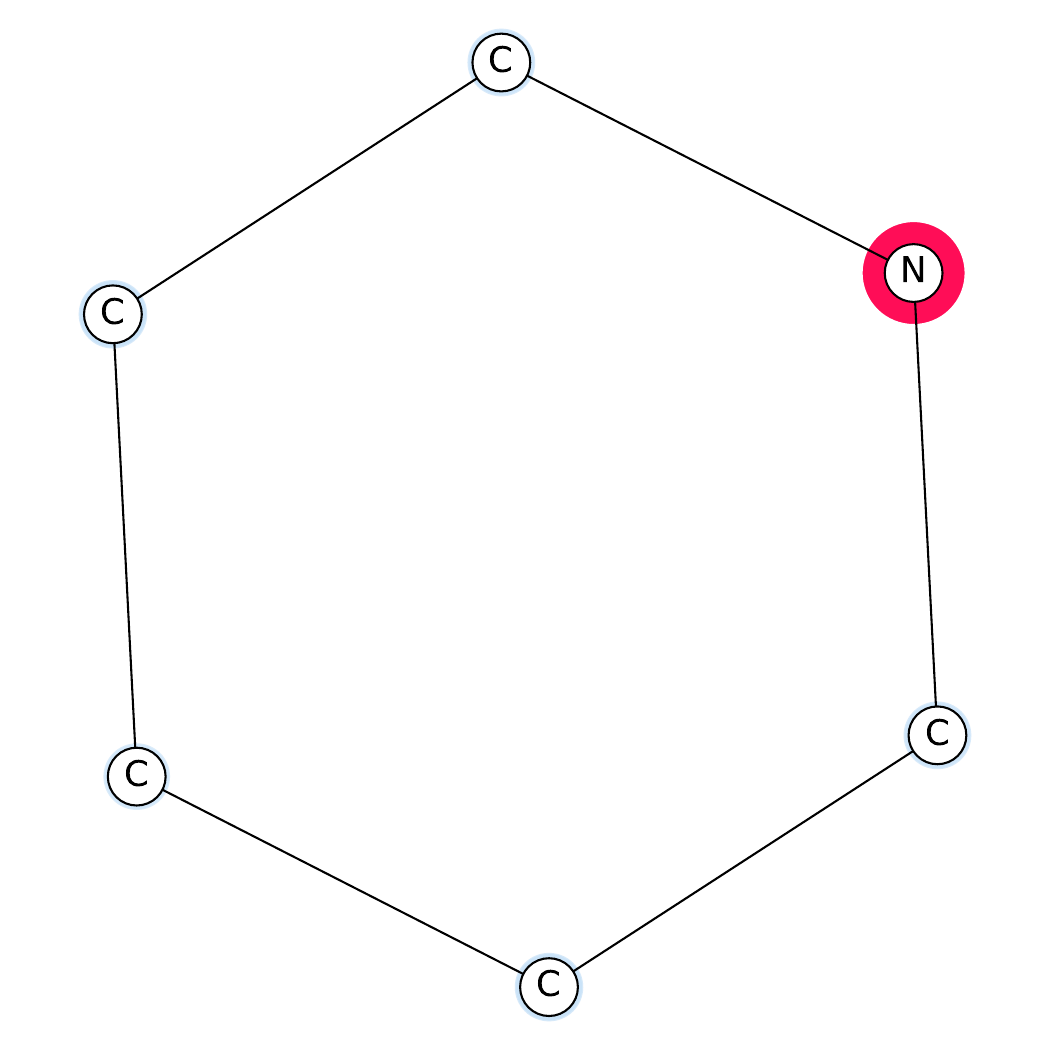}        
    \end{minipage}
    \hfill
    \begin{minipage}[c]{0.38\textwidth}
    \centering
        \includegraphics[width=\textwidth]{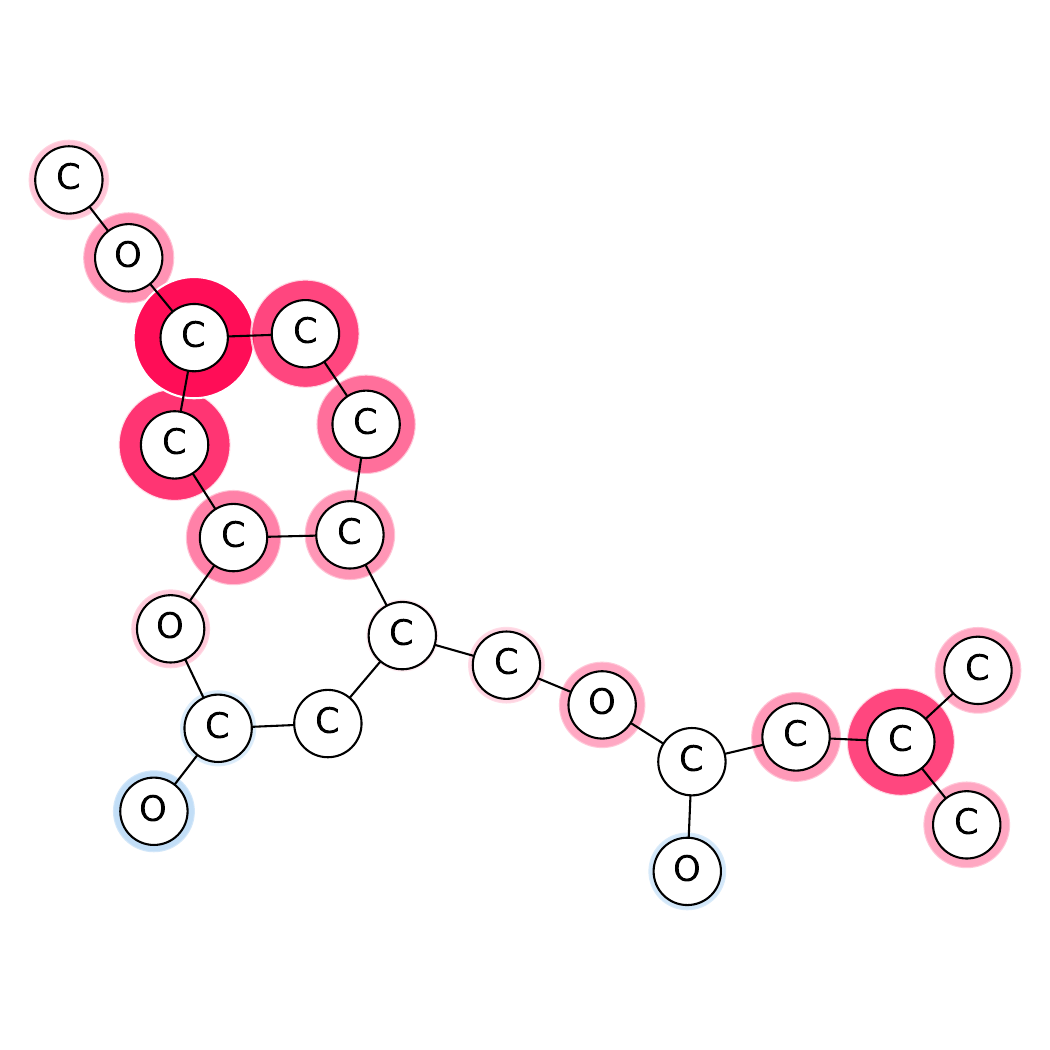}
    \end{minipage}
    \\
    \begin{minipage}[c]{0.2\textwidth}
    \raggedright
    \textbf{STII-Graph:\\(order 2)}
    \end{minipage}
    \hfill
    \begin{minipage}[c]{0.38\textwidth}
    \centering
        \includegraphics[width=\textwidth]{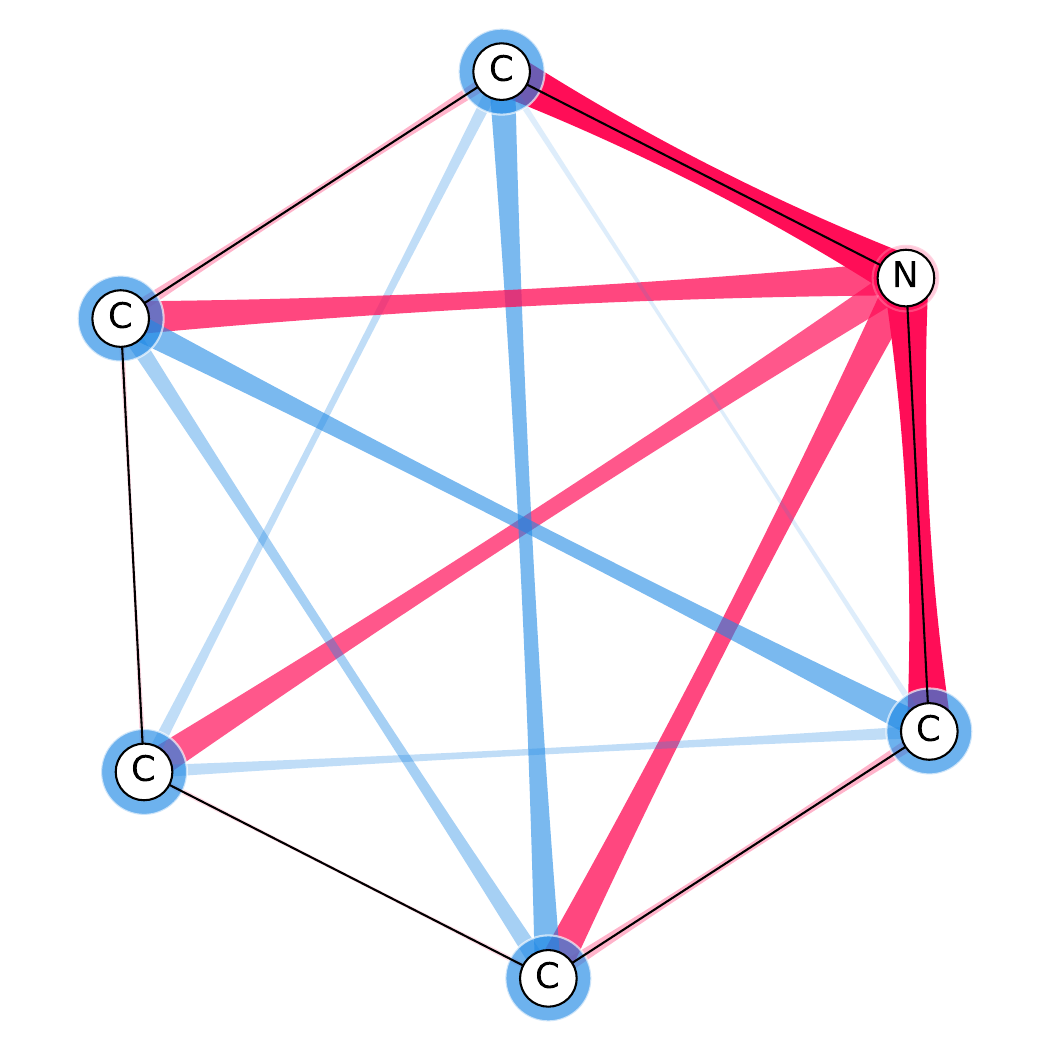}        
    \end{minipage}
    \hfill
    \begin{minipage}[c]{0.38\textwidth}
    \centering
        \includegraphics[width=\textwidth]{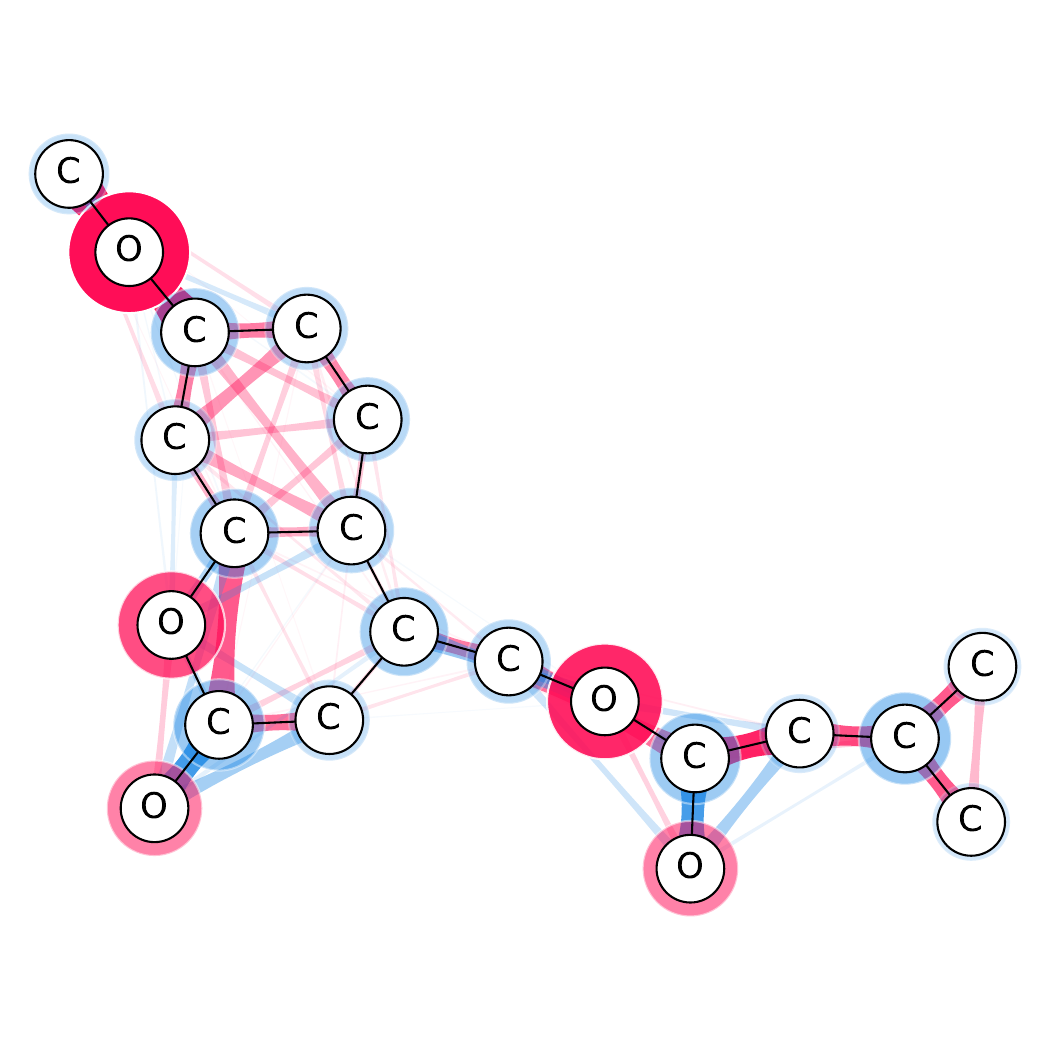}
    \end{minipage}
    \\
    \begin{minipage}[c]{0.2\textwidth}
    \textbf{$2$-SII-Graph:}
    \end{minipage}
    \hfill
    \begin{minipage}[c]{0.38\textwidth}
    \centering
        \includegraphics[width=\textwidth]{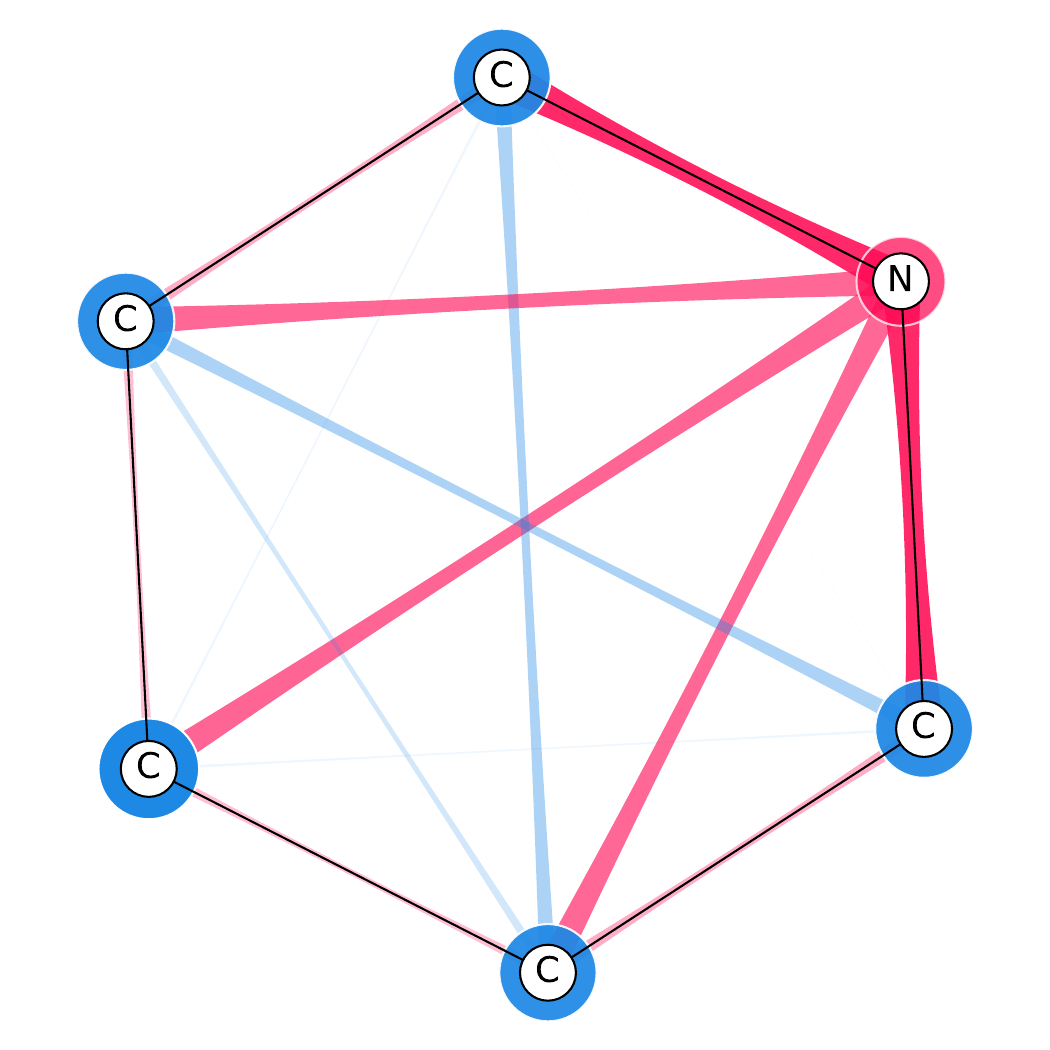}        
    \end{minipage}
    \hfill
    \begin{minipage}[c]{0.38\textwidth}
    \centering
        \includegraphics[width=\textwidth]{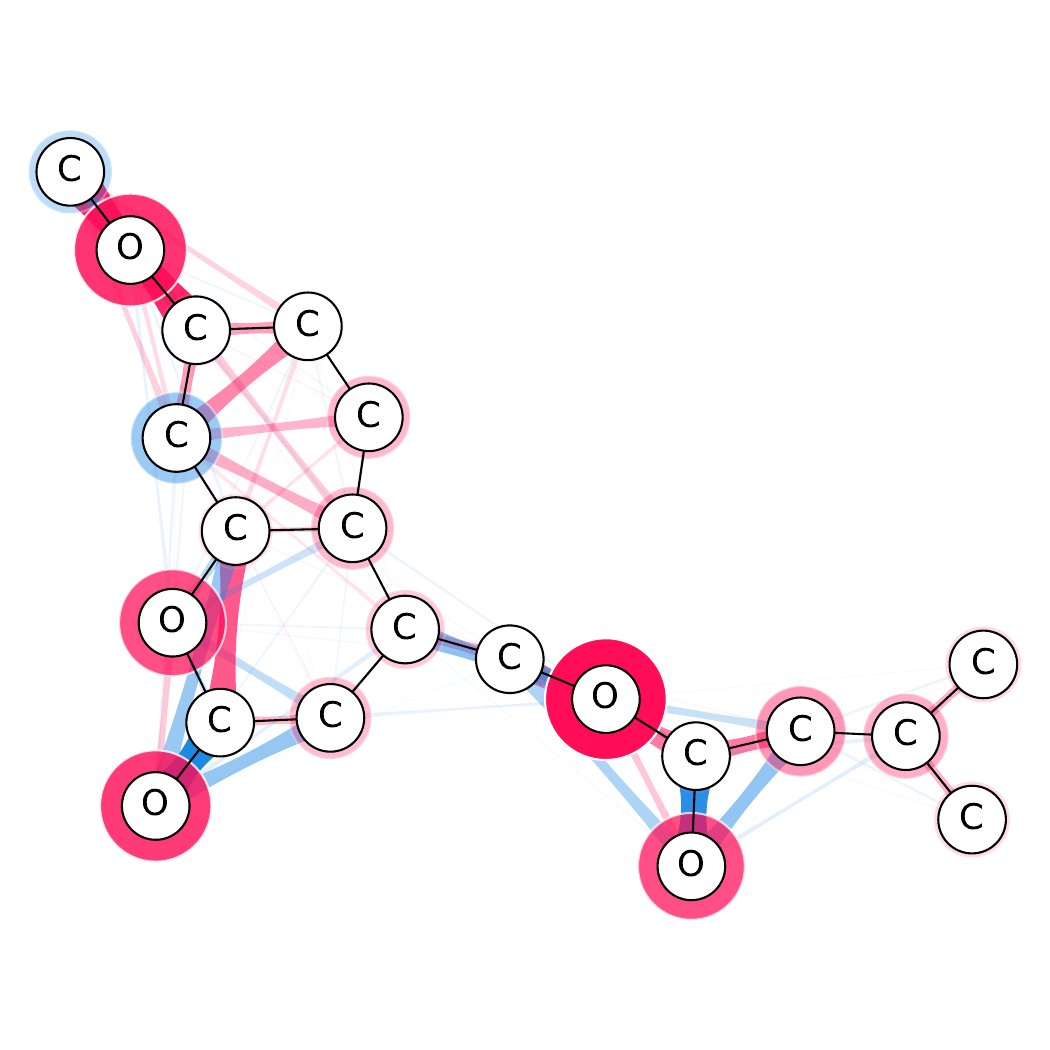}
    \end{minipage}
    \\
    \begin{minipage}[c]{0.2\textwidth}
    \textbf{$6$-SII-Graph:}
    \end{minipage}
    \hfill
    \begin{minipage}[c]{0.38\textwidth}
    \centering
        \includegraphics[width=\textwidth]{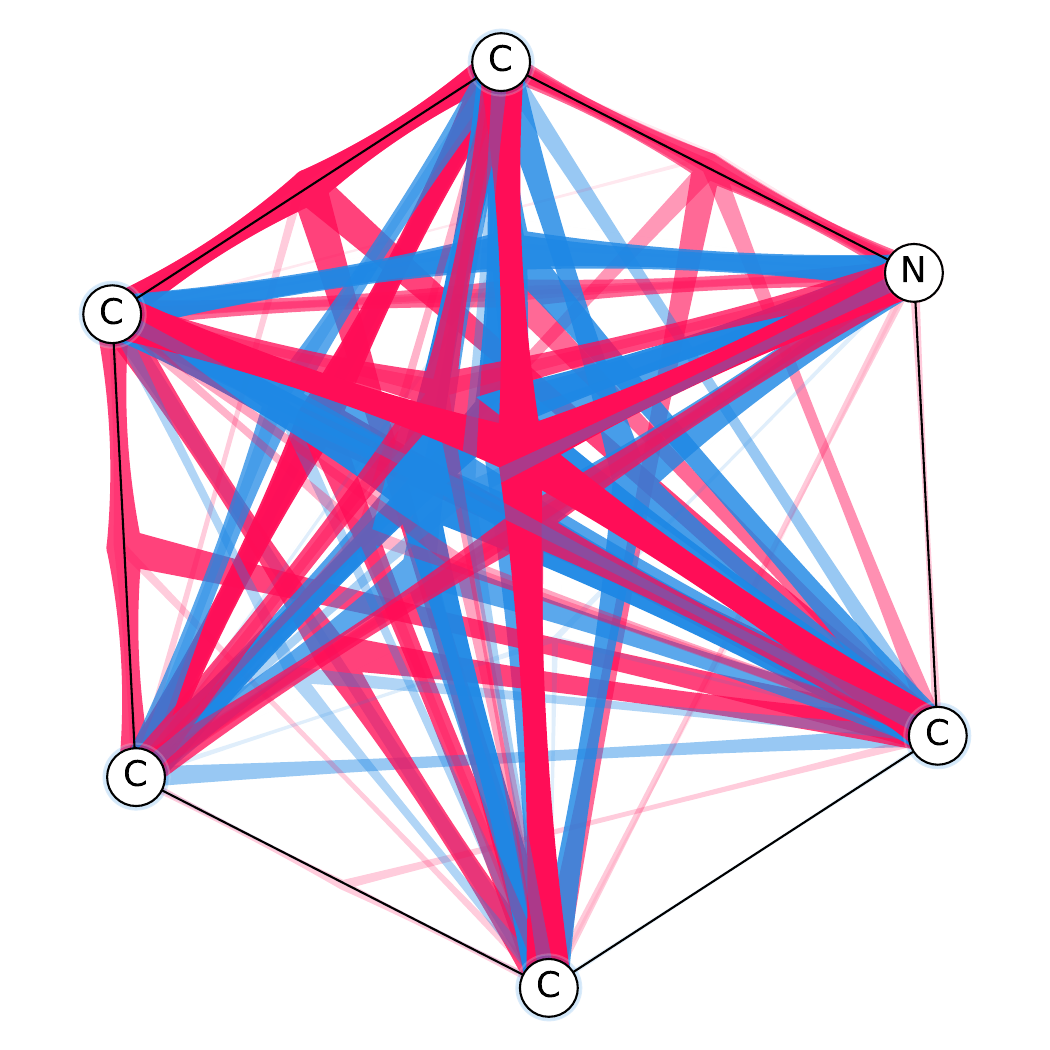}        
    \end{minipage}
    \hfill
    \begin{minipage}[c]{0.38\textwidth}
    \centering
        \includegraphics[width=\textwidth]{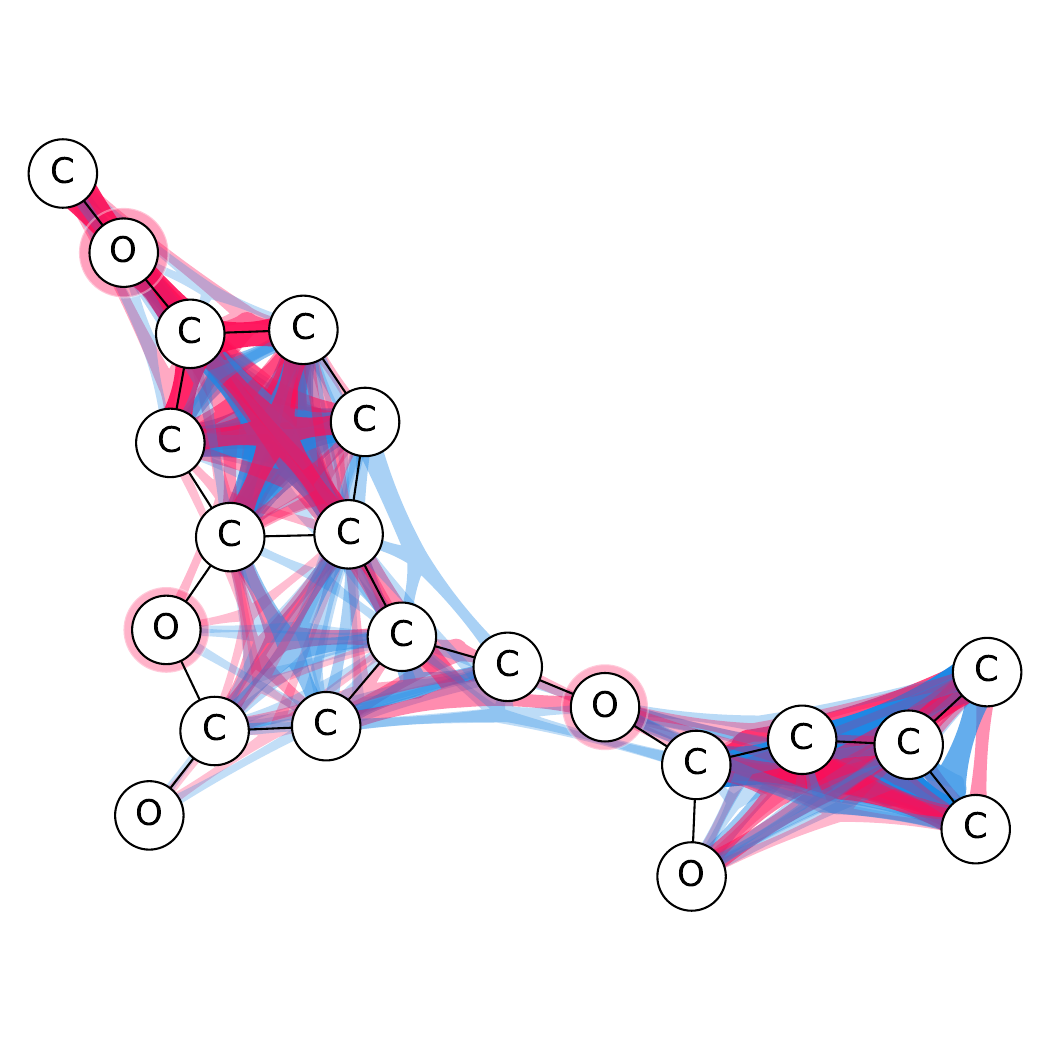}
    \end{minipage}
    \caption{Additional \gls*{SI}-Graphs for molecule structures of the \gls*{BNZ} dataset; molecule \emph{Pyridine} with $6$ atoms (left) and molecule \emph{57} with $21$ atoms (right). The model is a 3-layer GAT.}
    \label{appx_fig_benzene_graphs}
\end{figure}

\begin{figure}
    \centering
    \hspace{0.1\textwidth}
    \hfill
    \begin{minipage}[c]{0.28\textwidth}
    \centering
    \textbf{3-Layer GCN}        
    \end{minipage}
    \hfill
    \begin{minipage}[c]{0.28\textwidth}
    \centering
    \textbf{3-Layer GIN}
    \end{minipage}
    \hfill
    \begin{minipage}[c]{0.28\textwidth}
    \centering
    \textbf{3-Layer GAT}
    \end{minipage}
    \\[1em]
    \begin{minipage}[c]{0.1\textwidth}
    \centering
    \textbf{$2$-SII:}      
    \end{minipage}
    \hfill
    \begin{minipage}[c]{0.28\textwidth}
    \centering
        \includegraphics[width=\textwidth]{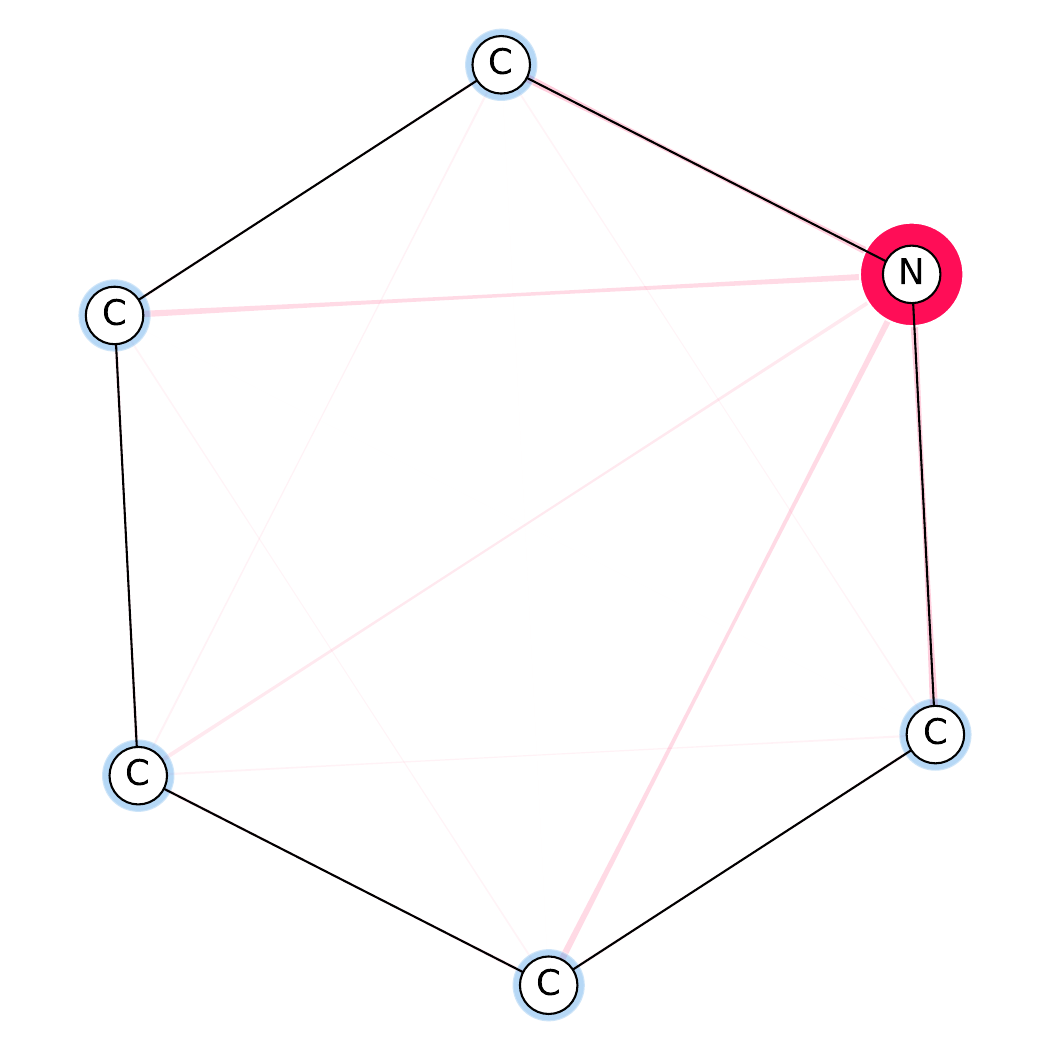}        
    \end{minipage}
    \hfill
    \begin{minipage}[c]{0.28\textwidth}
    \centering
        \includegraphics[width=\textwidth]{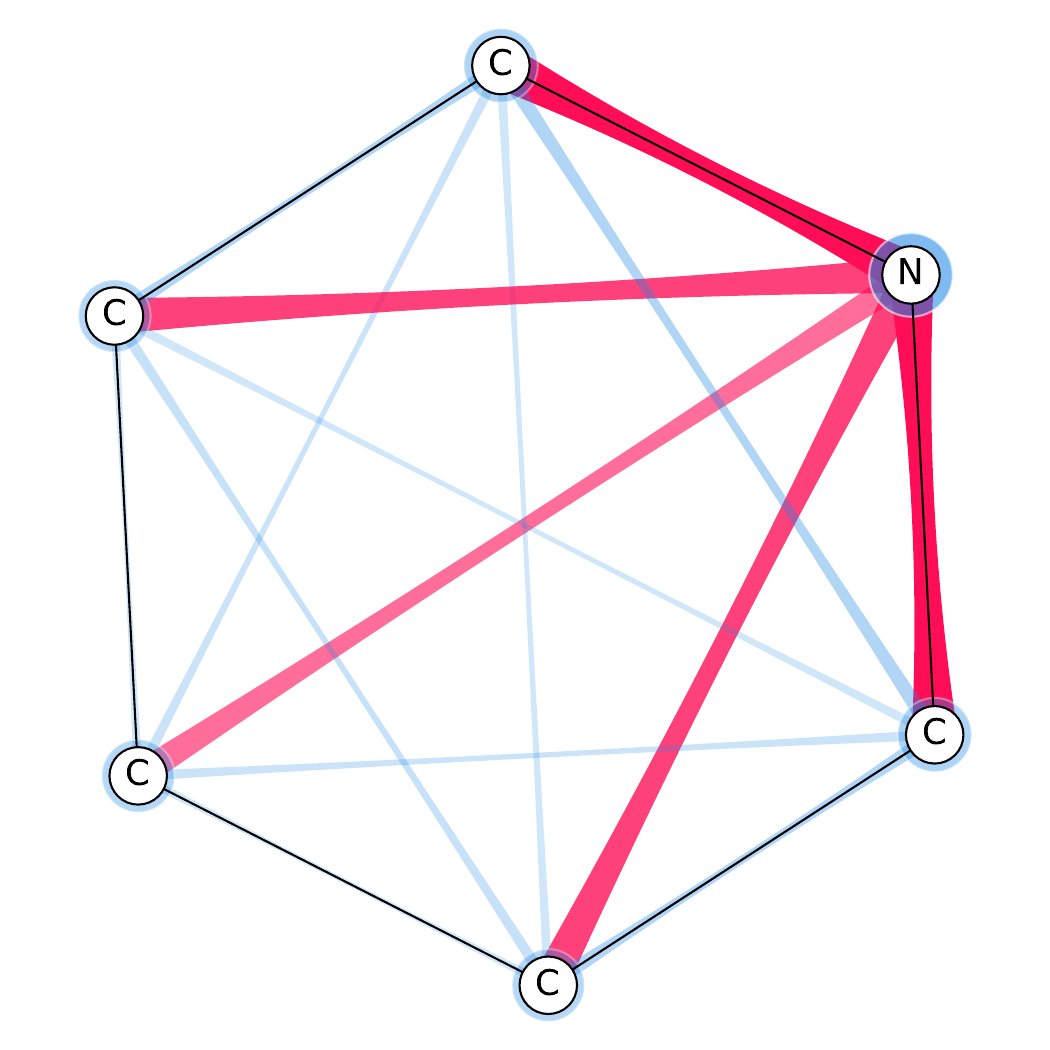}
    \end{minipage}
    \hfill
    \begin{minipage}[c]{0.28\textwidth}
    \centering
        \includegraphics[width=\textwidth]{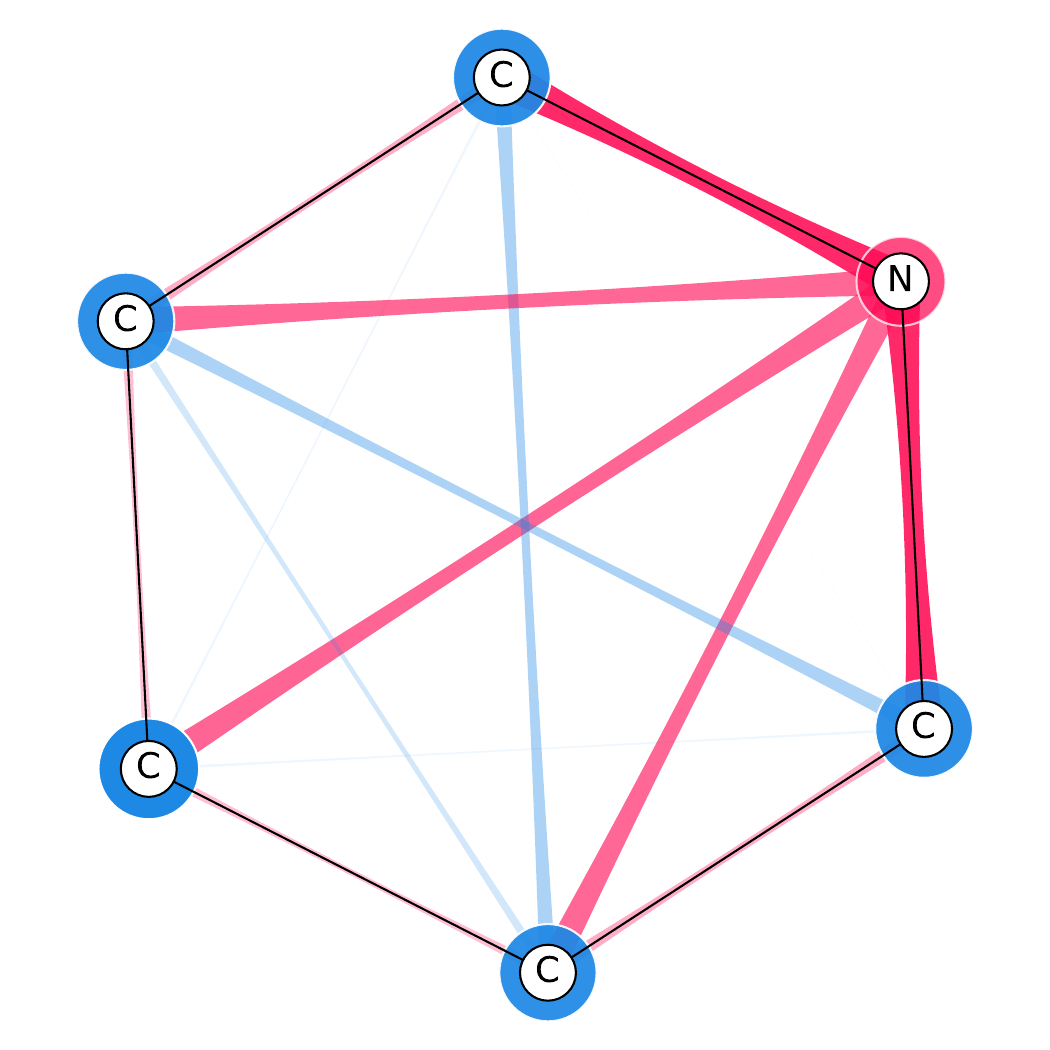}
    \end{minipage}
    \\[1em]
    \begin{minipage}[c]{0.1\textwidth}
    \centering
    \textbf{MI:}      
    \end{minipage}
    \hfill
    \begin{minipage}[c]{0.28\textwidth}
    \centering
        \includegraphics[width=\textwidth]{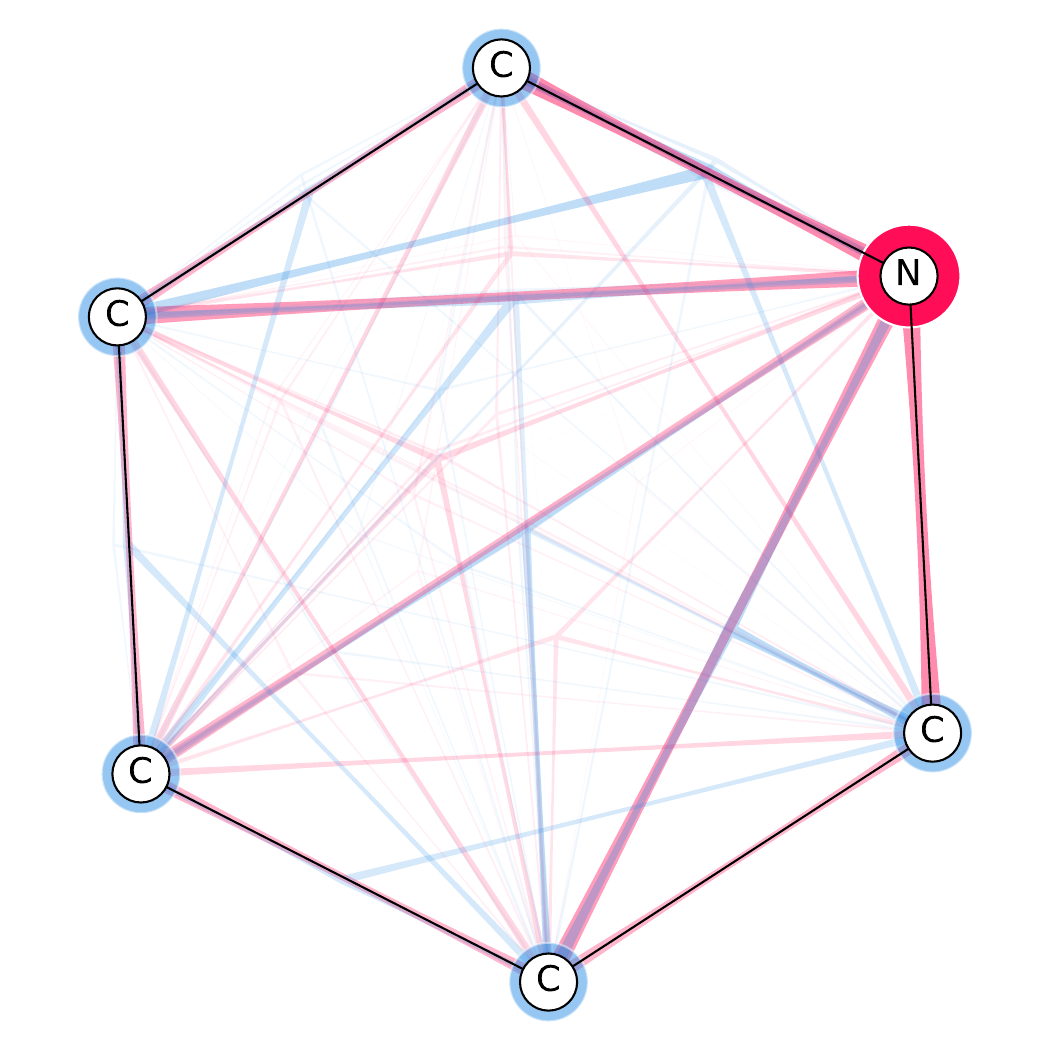}        
    \end{minipage}
    \hfill
    \begin{minipage}[c]{0.28\textwidth}
    \centering
        \includegraphics[width=\textwidth]{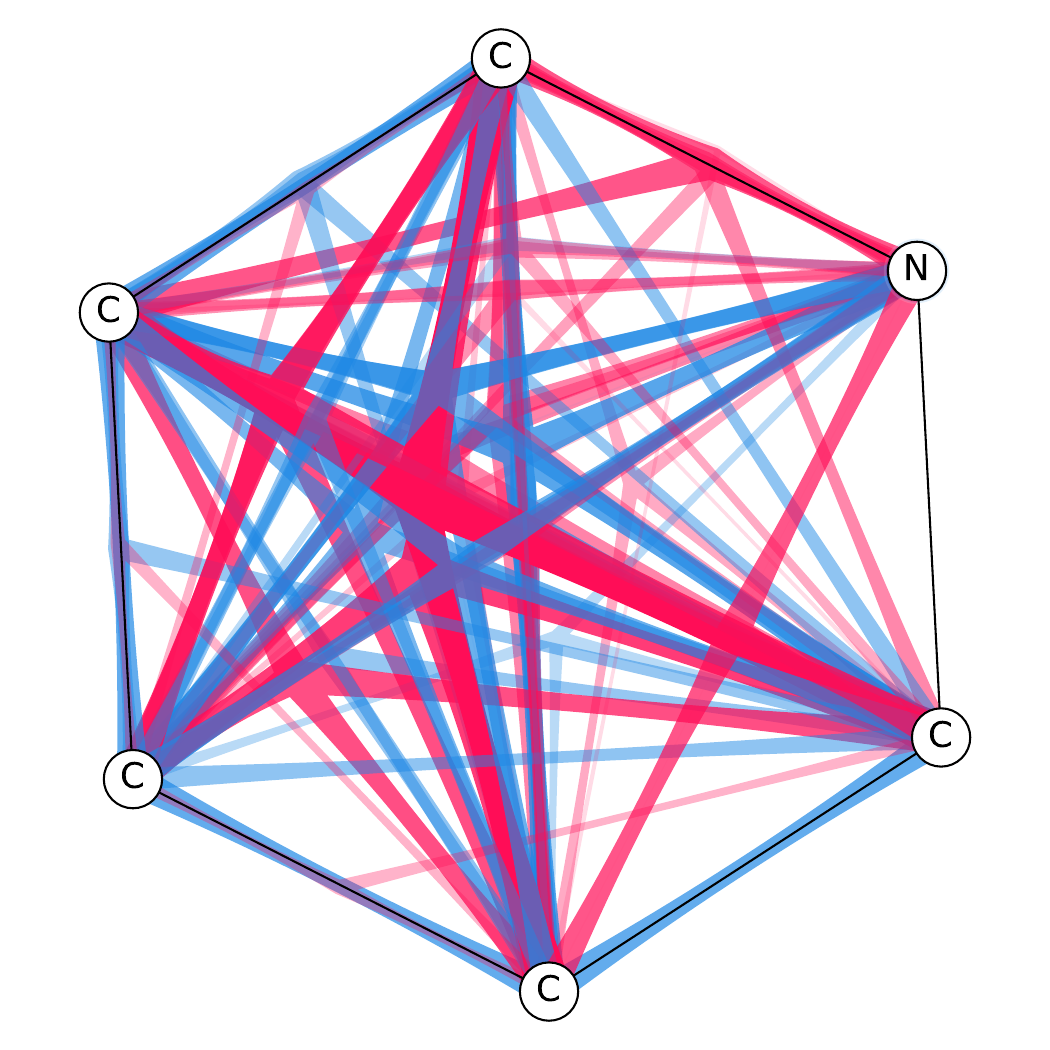}
    \end{minipage}
    \hfill
    \begin{minipage}[c]{0.28\textwidth}
    \centering
        \includegraphics[width=\textwidth]{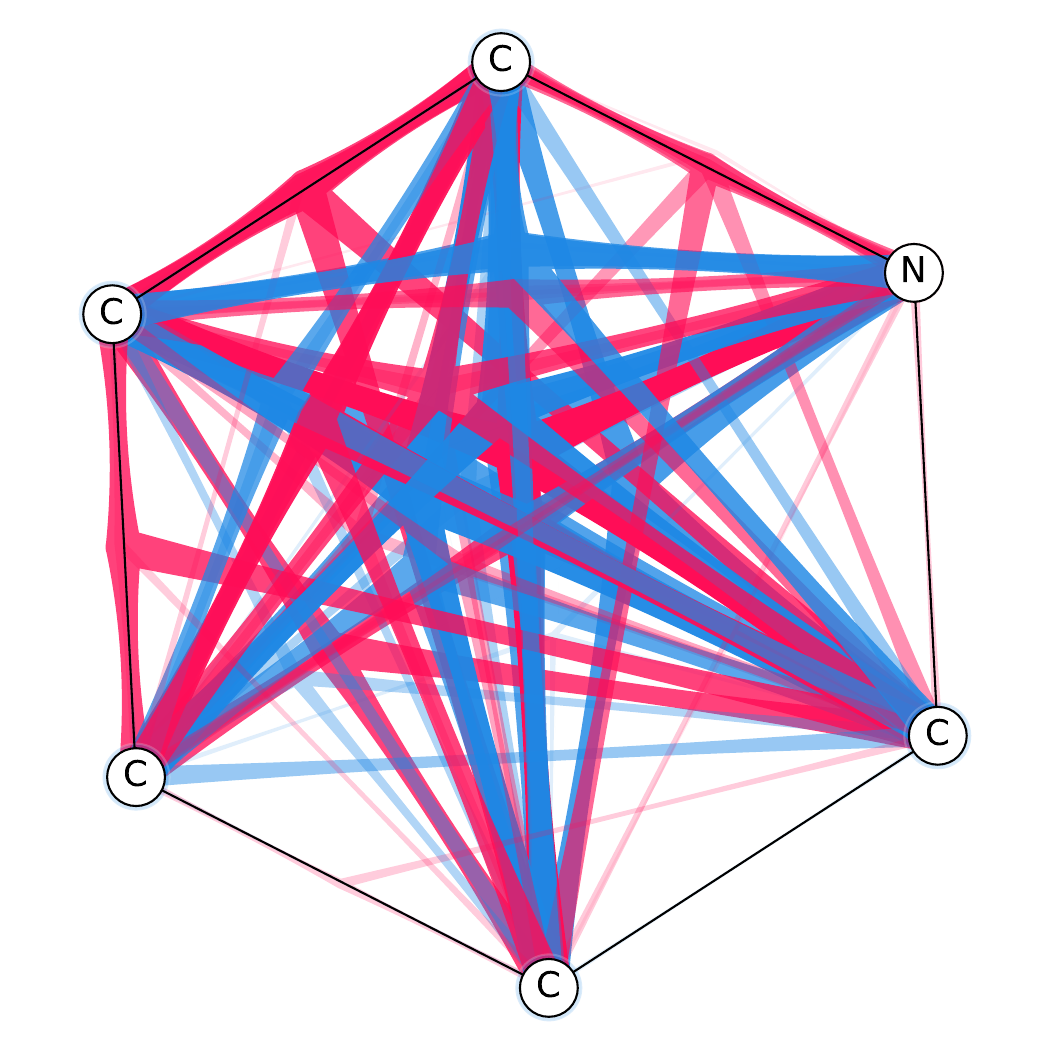}
    \end{minipage}
    \caption{Comparison of the \gls*{SI}-Graphs for the pyridine molecule for the three model architectures fitted on \gls*{BNZ}. All models accurately predict the molecule to be non-benzene.}
    \label{appx_fig_molecuels_gnn_comps}
\end{figure}

\end{document}